\documentclass{article}

\PassOptionsToPackage{numbers, compress}{natbib}

\usepackage[final]{neurips_2021}

\usepackage[utf8]{inputenc} 
\usepackage[T1]{fontenc}    
\usepackage{hyperref}       
\usepackage{url}            
\usepackage{booktabs}       
\usepackage{amsfonts}       
\usepackage{nicefrac}       
\usepackage{microtype}      
\usepackage[noend]{algorithmic} 
\usepackage{algorithm}

\usepackage{amsmath,amssymb,amsthm}
\usepackage{caption}
\usepackage{subcaption}
\usepackage{float}
\usepackage{xcolor}
\usepackage{graphicx}
\hypersetup{colorlinks,linkcolor={blue},citecolor={blue},urlcolor={red}}

\definecolor{DarkRed}{rgb}{0.75,0,0}
\definecolor{DarkGreen}{rgb}{0,0.5,0}
\definecolor{DarkPurple}{rgb}{0.5,0,0.5}
\definecolor{DarkBlue}{rgb}{0,0,0.7}

\hypersetup{colorlinks=true, plainpages = false, citecolor = DarkBlue, urlcolor = DarkBlue, filecolor = black, linkcolor =DarkBlue}

\usepackage{times}
\usepackage{enumitem}

\usepackage{thmtools} 

\usepackage{thm-restate}
\newtheorem{lemma}{Lemma}[section]
\newtheorem{theorem}[lemma]{Theorem}
\newtheorem{proposition}[lemma]{Proposition}
\newtheorem{corollary}[lemma]{Corollary}

\newtheorem{claim}[lemma]{Claim}

\newtheorem{definition}[lemma]{Definition}

\usepackage{prettyref}

\usepackage{upgreek}
\usepackage[mathscr]{euscript}
\usepackage{mathtools}
\usepackage{multirow}

\newcommand{\Acal}{\mathcal{A}}
\newcommand{\Scal}{\mathcal{S}}
\newcommand{\Ecal}{\mathcal{E}}
\newcommand{\Mcal}{\mathcal{M}}
\newcommand{\Zcal}{\mathcal{Z}}
\newcommand{\Bcal}{\mathcal{B}}

\newcommand{\Vcal}{\mathcal{V}}
\newcommand{\Pcal}{\mathcal{P}}

\newcommand{\Ncal}{\mathcal{N}}
\newcommand{\Xcal}{\mathcal{X}}
\newcommand{\Hcal}{\mathcal{H}}

\newcommand{\EE}{\mathbb{E}} 
\newcommand{\VV}{\mathbb{V}} 
\newcommand{\NN}{\mathbb{N}} 
\newcommand{\PP}{\mathbb{P}} 
\newcommand{\RR}{\mathbb{R}} 
\newcommand{\Ocal}{\mathcal{O}}

\newcommand{\minimize}[3]{
& \underset{#1}{\operatorname{minimize}}
& & #2\\
& \operatorname{s.t.}
& & #3
}

\newcommand{\maximize}[3]{
& \underset{#1}{\textrm{maximize}}
& & #2\\
& \textrm{s.t.}
& & #3
}

\DeclareMathOperator*{\argmax}{argmax}

\DeclareMathOperator{\gap}{gap}
\DeclareMathOperator{\clip}{clip}

\DeclarePairedDelimiter{\bracket}{[}{]}
\DeclarePairedDelimiter{\curl}{\{}{\}}

\newcommand{\returngap}{\overline{\gap}}

\newcommand{\ignore}[1]{}

\newcommand{\return}[1]{v^{#1}}
\newcommand{\regret}{\mathfrak{R}}

\newcommand{\Blead}{B^{\mathrm{lead}}}
\newcommand{\Bfut}{B^{\mathrm{fut}}}
\newcommand{\indicator}[1]{\chi\left(#1\right)}

\usepackage{prettyref}
\newcommand{\pref}[1]{\prettyref{#1}}

\newcommand{\savehyperref}[2]{\texorpdfstring{\hyperref[#1]{#2}}{#2}}
\newrefformat{eq}{\savehyperref{#1}{\textup{(\ref*{#1})}}}
\newrefformat{eqn}{\savehyperref{#1}{Equation~\ref*{#1}}}
\newrefformat{con}{\savehyperref{#1}{Conjecture~\ref*{#1}}}
\newrefformat{lem}{\savehyperref{#1}{Lemma~\ref*{#1}}}
\newrefformat{def}{\savehyperref{#1}{Definition~\ref*{#1}}}
\newrefformat{line}{\savehyperref{#1}{line~\ref*{#1}}}
\newrefformat{thm}{\savehyperref{#1}{Theorem~\ref*{#1}}}
\newrefformat{corr}{\savehyperref{#1}{Corollary~\ref*{#1}}}
\newrefformat{sec}{\savehyperref{#1}{Section~\ref*{#1}}}
\newrefformat{app}{\savehyperref{#1}{Appendix~\ref*{#1}}}
\newrefformat{ass}{\savehyperref{#1}{Assumption~\ref*{#1}}}
\newrefformat{ex}{\savehyperref{#1}{Example~\ref*{#1}}}
\newrefformat{fig}{\savehyperref{#1}{Figure~\ref*{#1}}}
\newrefformat{alg}{\savehyperref{#1}{Algorithm~\ref*{#1}}}
\newrefformat{rem}{\savehyperref{#1}{Remark~\ref*{#1}}}
\newrefformat{conj}{\savehyperref{#1}{Conjecture~\ref*{#1}}}
\newrefformat{prop}{\savehyperref{#1}{Proposition~\ref*{#1}}}
\newrefformat{proto}{\savehyperref{#1}{Protocol~\ref*{#1}}}
\newrefformat{prob}{\savehyperref{#1}{Problem~\ref*{#1}}}
\newrefformat{claim}{\savehyperref{#1}{Claim~\ref*{#1}}}

\usepackage{pifont}

\title{Beyond Value-Function Gaps: Improved Instance-Dependent Regret Bounds for Episodic Reinforcement Learning}

\author{%
  Chris Dann \\
  Google Research \\
  \texttt{chrisdann@google.com} \\
  \And
  Teodor V. Marinov\thanks{Author was at Johns Hopkins University during part of this work.} \\
  Google Research \\
  \texttt{tvmarinov@google.com} \\
  \AND
  Mehryar Mohri \\
  Courant Institute and Google Research \\
  \texttt{mohri@google.com} \\
  \And
  Julian Zimmert \\
  Google Research \\
  \texttt{zimmert@google.com} \\
}

\begin{document}

\maketitle

\begin{abstract}
We provide improved gap-dependent regret bounds for reinforcement learning in finite episodic Markov decision processes. Compared to prior work, our bounds depend on alternative definitions of gaps. These definitions are based on the insight that, in order to achieve a favorable regret, an algorithm does not need to learn how to behave optimally in states that are not reached by an optimal policy. We prove tighter upper regret bounds for optimistic algorithms and accompany them with new information-theoretic lower bounds for a large class of MDPs. Our results show that optimistic algorithms can not achieve the information-theoretic lower bounds even in deterministic MDPs unless there is a unique optimal policy.
\end{abstract}

\section{Introduction}

Reinforcement Learning (RL) is a general scenario where agents interact with
the environment to achieve some goal. The environment and an agent's interactions
are typically modeled as a Markov decision process (MDP) \citep{puterman1994markov},
which can represent a rich variety of tasks. But, for which MDPs can an agent or an RL algorithm succeed? This requires a theoretical analysis of the complexity of an MDP.
This paper studies this question in the tabular episodic setting, where an agent interacts with the environment in episodes of fixed length $H$ and where the size of the state and action space is finite ($S$ and $A$ respectively).

While the performance of RL algorithms in tabular Markov decision processes has been 
the subject of many studies in the past \citep[e.g.][]{fiechter1994efficient,kakade2003sample,osband2013more,dann2017unifying,azar2017minimax,jin2018q,zanette2019tighter, dann2019strategic}, the vast majority of existing analyses focuses on worst-case problem-independent regret bounds, which only take into account the size of the MDP, the horizon $H$ and the number of episodes $K$. 

Recently, however, some significant progress has been achieved towards deriving
more optimistic (problem-dependent) guarantees. This includes 
more refined regret bounds for the tabular episodic setting that depend on 
structural properties of the specific MDP considered \citep{simchowitz2019non, lykouris2019corruption, jin2020simultaneously,foster2020instance,he2020logarithmic}. Motivated by instance-dependent analyses in multi-armed bandits \citep{lai1985asymptotically}, these analyses derive 
gap-dependent regret-bounds of the form 
$O\left(\sum_{(s,a)\in\Scal\times\Acal} \frac{H\log(K)}{\gap(s,a)}\right)$,
where the sum is over state-actions pairs $(s, a)$ and
where the gap notion is defined as the difference of the optimal value function $V^{*}$ of the Bellman optimal policy $\pi^*$ and the $Q$-function of $\pi^*$ at a sub-optimal action:
$\gap(s,a) = V^{*}(s) - Q^{*}(s,a)$. We will refer to this gap definition as \emph{value-function gap} in the following. We note that a similar notion of gap has been used in the infinite horizon setting to achieve instance-dependent bounds \citep{auer2007logarithmic,tewari2008optimistic,auer2009near,filippi2010optimism,ok2018exploration}, however, a strong assumption about irreducibility of the MDP is required.

While regret bounds based on these value function gaps generalize the bounds available in the multi-armed bandit setting, we argue that they have a major limitation. The bound at each state-action pair depends only on the gap at the pair and treats all state-action pairs equally, ignoring their topological ordering in the MDP. This can have a major impact on the derived bound.
In this paper, we address this issue and formalize the following key observation about the difficulty of RL in an episodic MDP through improved instance-dependent regret bounds:
\begin{center}
\begin{minipage}{0.9\linewidth}
\emph{Learning a policy with optimal return does not require an RL agent to distinguish between actions with similar outcomes (small value-function gap) in states that can only be reached by taking highly suboptimal actions (large value-function gap).}
\end{minipage}
\end{center}

To illustrate this insight, consider autonomous driving, where each episode corresponds to driving from a start to a destination. If the RL agent decides to run a red light on a crowded intersection, then a car crash is inevitable. Even though the agent could slightly affect the severity of the car crash by steering, this effect is small and, hence, a good RL agent does not need to learn how to best steer after running a red light. Instead, it would only need a few samples to learn to obey the traffic light in the first place as the action of disregarding a red light has a very large value-function gap.

\begin{figure}[t]
\begin{minipage}[c]{0.3\textwidth}
\includegraphics[width=\textwidth, trim=0 0cm 0cm 0, clip]{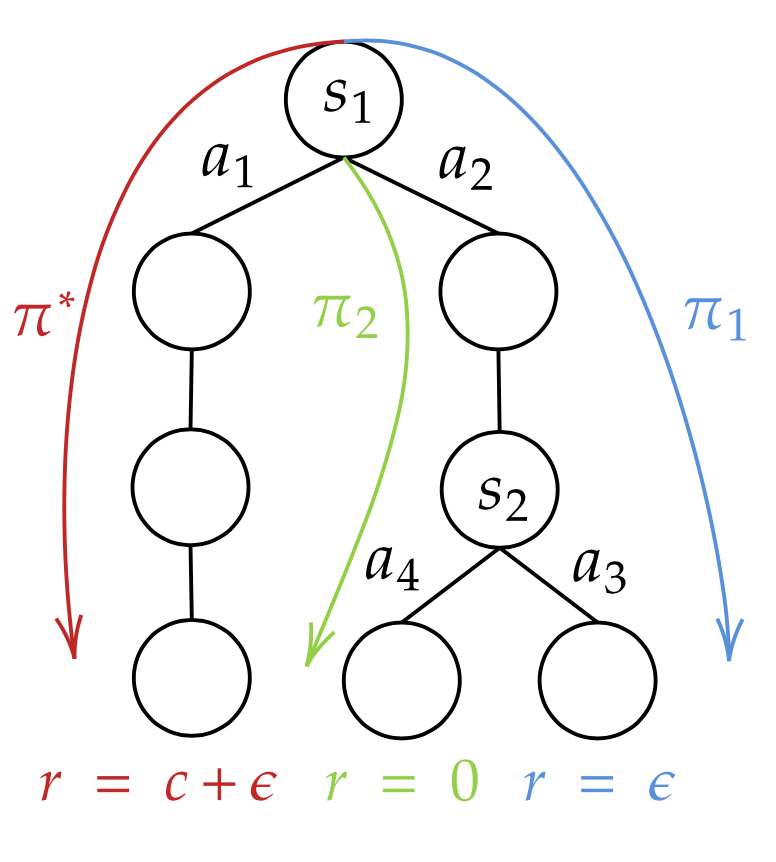}
\end{minipage}
\scalebox{0.9}{
\begin{minipage}[r]{0.7\textwidth}
    \bgroup
    \vspace*{-5mm}
\def\arraystretch{2.1}
\begin{tabular}{r|c|c|}
 \multicolumn{1}{c}{} & \multicolumn{1}{c}{Value-function gap (prior)} & 
 \multicolumn{1}{c}{Return gap (ours)}
 \\ \cline{2-3}
 \multirow{2}{*}{
\begin{minipage}{1.4cm}\raggedleft General Regret bounds\end{minipage}
}
& $\displaystyle O\Big(\sum_{s,a} \frac{H\log(K)}{\gap(s,a)}\Big)$       
& $\displaystyle O\Big(\sum_{s,a} \frac{\log(K)}{\returngap(s,a)}\Big)$ \\
& $\displaystyle \Omega\Big(\sum_{\substack{s,a \colon s \in \pi^*}} \frac{\log(K)}{\gap(s,a)}\Big)$       
& $\displaystyle \Omega\Big(\sum_{s,a} \frac{\log(K)}{H\returngap(s,a)}\Big)$ \\ \cline{2-3}
 \multirow{2}{*}{
\begin{minipage}{1.6cm}\raggedleft Example on the left\end{minipage}
}
& 
\begin{minipage}{2.5cm}
\begin{center}
$\gap(s_1, a_2) = c$\\
$\gap(s_2, a_4) = \epsilon$
\end{center}
\end{minipage}
& 
\begin{minipage}{4cm}
\begin{center}
$\returngap(s_1, a_2) = c$\\
$\returngap(s_2, a_4) = \frac{c + \epsilon}{H} \approx c$
\end{center}
\end{minipage}
\\
&
$\displaystyle O\Big( \frac{SH\log(K)}{\epsilon}\Big)$ 
& 
$\displaystyle O\Big( \frac{SH\log(K)}{c}\Big)$ 
\\ \cline{2-3}
\end{tabular}\egroup
\end{minipage}}
\caption{Comparison of our contributions in MDPs with deterministic transitions. Bounds only include the main terms and all sums over $(s,a)$ are understood to only include terms where the respective gap is nonzero. $\returngap$ is our alternative \emph{return gap} definition introduced later (\pref{def:return_gap}). 
}
\label{fig:summary}
\end{figure}

To understand how this observation translates into regret bounds,  consider the toy example in \prettyref{fig:summary}. This MDP has deterministic transitions and only terminal rewards with $c \gg \epsilon > 0$. There are two decision points, $s_1$ and $s_2$, with two actions each, and all other states have a single action. There are three policies which govern the regret bounds: $\pi^*$ (red path) which takes action $a_1$ in state $s_1$; 
$\pi_1$ which takes action $a_2$ at $s_1$ and $a_3$ at $s_2$ (blue path); and $\pi_2$ which takes action $a_2$ at $s_1$ and $a_4$ at $s_2$ (green path). 
Since $\pi^*$ follows the red path, it never reaches $s_2$ and achieves optimal return $c+\epsilon$, while $\pi_1$ and $\pi_2$ are both suboptimal with return $\epsilon$ and $0$ respectively.
Existing value-function gaps evaluate to $\gap(s_1,a_2) = c$ and $\gap(s_2,a_4) = \epsilon$ which yields a regret bound of order $H\log(K)(1/c + 1/\epsilon)$. The idea behind these bounds is to capture the necessary number of episodes to distinguish the value of the optimal policy $\pi^*$ from the value of any other sub-optimal policy \emph{on all states}. 
However, since $\pi^*$ will never reach $s_2$ it is not necessary to distinguish it from any other policy at $s_2$.
A good algorithm only needs to determine that $a_2$ is sub-optimal in $s_1$, which eliminates both $\pi_1$ and $\pi_2$ as optimal policies after only $\log(K)/c^2$ episodes. This suggests a regret of order $O(\log(K)/c)$. 
The bounds presented in this paper achieve this rate up to factors of $H$ by replacing the gaps at every state-action pair with the average of all gaps along certain paths containing the state action pair. We call these averaged gaps \emph{return gaps}. The return gap at $(s,a)$ is denoted as $\returngap(s,a)$. Our new bounds replace $\gap(s_2,a_4) = \epsilon$ by $\returngap(s_2,a_4) \approx \frac{1}{2}\gap(s_1,a_2) + \frac{1}{2}\gap(s_2,a_4) = \Omega(c)$.
Notice that $\epsilon$ and $c$ can be selected arbitrarily in this example. In particular, if we take $c=0.5$ and $\epsilon = 1/\sqrt{K}$ our bounds remain logarithmic $O(\log(K))$, while prior regret bounds scale as $\sqrt{K}$.

This work is motivated by the insight just discussed. First, we show that improved regret bounds are indeed possible by proving a tighter regret bound for \textsc{StrongEuler}, an existing algorithm based on the optimism-in-the-face-of-uncertainty (OFU) principle \citep{simchowitz2019non}. 
Our regret bound is stated in terms of our new return gaps that capture the problem difficulty more accurately and avoid explicit dependencies on the smallest value function gap $\gap_{\min}$. Our technique applies to optimistic algorithms in general and as a by-product improves the dependency on episode length $H$ of prior results. 
%
Second, we investigate the difficulty of RL in episodic MDPs from an information-theoretic perspective by deriving regret lower-bounds. We show that existing value-function gaps are indeed sufficient to capture difficulty of problems but only when each state is visited by an optimal policy with some probability. 
Finally, we prove a new lower bound when the transitions of the MDP are deterministic that depends only on the difference in return of the optimal policy and suboptimal policies, which is closely related to our notion of return gap.

\section{Problem setting and notation}

We consider reinforcement learning in episodic tabular MDPs with a fixed horizon. An MDP can be described as a tuple $(\Scal, \Acal, P, R, H)$, where $\Scal$ and $\Acal$ are state- and action-space of size $S$ and $A$ respectively, $P$ is the state transition distribution with $P(\cdot|s,a) \in \Delta^{S-1}$ the next state probability distribution, given that action $a$ was taken in the current state $s$. $R$ is the reward distribution defined over $\Scal \times \Acal$  
and $r(s,a) = \mathbb{E}[R(s,a)] \in [0,1]$. Episodes
admit a fixed length or \emph{horizon} $H$. 

We consider \emph{layered} MDPs: each state $s \in \Scal$ belongs to a layer $\kappa(s) \in [H]$ and the only non-zero transitions are between states $s, s'$ in consecutive layers, with $\kappa(s') = \kappa(s)+1$. This common assumption \citep[see e.g.][]{krishnamurthy2016pac} corresponds to MDPs with time-dependent transitions, as in  \citep{jin2018q, dann2017unifying}, but allows us to omit an explicit time-index in value-functions and policies.
For ease of presentation, we assume there is a unique start state $s_1$ with $\kappa(s_1) = 1$ but our results can be generalized to multiple (possibly adversarial) start states. 
Similarly, for convenience, we assume that all states are reachable by some policy with non-zero probability, but not necessarily all policies or the same policy.

We denote by $K$ the number of episodes during which the MDP is visited. Before each episode $k \in [K]$, the agent selects a deterministic policy 
$\pi_k \colon \Scal \rightarrow \Acal$ out of a set of all policies $\Pi$ and $\pi_k$ is then executed for all $H$ time steps in episode $k$.
For each policy $\pi$, we denote by
$w^\pi(s,a) = \PP( S_{\kappa(s)} = s, A_{\kappa(s)} = a \mid A_h = \pi(S_h) \,\,\forall h \in [H])$
and $w^\pi(s) = \sum_a w^\pi(s,a)$
 probability of reaching state-action pair $(s,a)$ and state $s$ respectively when executing $\pi$.
For convenience, $supp(\pi) = \{ s \in \Scal \colon w^\pi(s) > 0\}$ is the set of states visited by $\pi$ with non-zero probability. The Q- and value function of a policy $\pi$ are
\begin{align*}
        Q^\pi(s, a) &= \EE_\pi \Bigg[ \sum_{h=\kappa(s)}^H r(S_h, A_h) ~\Bigg|~ S_{\kappa(s)} = s, A_{\kappa(s)} = a\Bigg],
        & \textrm{and} \quad
            V^\pi(s) &= Q^\pi(s, \pi(s))
\end{align*}
and the regret incurred by the agent is the sum of its regret over $K$ episodes
\begin{equation}
    \regret(K) = \sum_{k=1}^K \return{*} - \return{\pi_k} = \sum_{k=1}^K V^{*}(s_1) - V^{\pi_k}(s_1),
\end{equation}
where  $\return{\pi} = V^{\pi}(s_1)$ is the expected total sum of rewards or \emph{return} of $\pi$ and $V^*$ is the optimal value function $V^*(s) = \max_{\pi \in \Pi} V^\pi(s)$.
Finally, the set of optimal policies is denoted as $\Pi^* = \{\pi \in \Pi : V^{\pi} = V^*\}$. Note that we only call a policy optimal if it satisfies the Bellman equation in every state, as is common in literature, but there may be policies outside of $\Pi^*$ that also achieve maximum return because they only take suboptimal actions outside of their support. The variance of the $Q$ function at a state-action pair $(s,a)$ of the optimal policy is $\Vcal^*(s,a) = \VV[R(s,a)] + \VV_{s'\sim P(\cdot|s,a)}[V^{*}(s')]$, where $\VV[X]$ denotes the variance of the r.v.\ $X$. The maximum variance over all state-action pairs is $\Vcal^* = \max_{(s,a)}\Vcal^*(s,a)$.
Finally, our proofs will make use of the following clipping operator $\clip[a|b] = \chi(a\geq b)a$ that sets $a$ to zero if it is smaller than $b$, where $\chi$ is the indicator function.

\section{Novel upper bounds for optimistic algorithms}
\label{sec:upper_bounds}
In this section, we present tighter regret upper-bounds for optimistic algorithms through a novel analysis technique.
Our technique can be generally applied to model-based optimistic algorithms such as \textsc{StrongEuler} \citep{simchowitz2019non}, \textsc{Ucbvi} \citep{azar2012sample}, \textsc{ORLC} \citep{dann2019policy} or \textsc{Euler} \citep{zanette2019tighter}.
In the following, we will first give a brief overview of this class of algorithms (see \pref{app:opt_algs} for more details) and then state our main results for the \textsc{StrongEuler} algorithm \cite{simchowitz2019non}. We focus on this algorithm for concreteness and ease of comparison.

Optimistic algorithms maintain estimators of the $Q$-functions at every state-action pair such that there exists at least one policy $\pi$ for which the estimator, $\bar Q^{\pi}$, overestimates the $Q$-function of the optimal policy, that is $\bar Q^{\pi}(s,a) \geq Q^*(s,a),\forall (s,a)\in\Scal\times\Acal$. During episode $k\in[K]$, the optimistic algorithm selects the policy $\pi_k$ with highest optimistic value function $\bar V_k$. By definition, it holds that $\bar V_k(s) \geq V^*(s)$. The optimistic value and $Q$-functions are constructed through finite-sample estimators of the true rewards $r(s,a)$ and the transition kernel $P(\cdot|s,a)$ plus bias terms, similar to estimators for the UCB-I multi-armed bandit algorithm. Careful construction of these bias terms is crucial for deriving min-max optimal regret bounds in $S,A$ and $H$ \citep{azar2017minimax}. Bias terms which yield the tightest known bounds come from concentration of martingales results such as Freedman's inequality~\citep{Freedman1975} and empirical Bernstein's inequality for martingales~\citep{maurer2009empirical}.

The \textsc{StrongEuler} algorithm not only satisfies optimism, i.e.,  $\bar V_k \geq V^*$, but also a stronger version called \emph{strong optimism}. To define strong optimism we need the notion of \emph{surplus} which roughly measures the optimism at a fixed state-action pair. Formally the surplus at $(s,a)$ during episode $k$ is defined as
\begin{align}
\label{eq:surpl_def}
    E_k(s,a) = \bar Q_k(s,a) - r(s,a) - \langle P(\cdot|s,a), \bar V_k\rangle~.
\end{align}
We say that an algorithm is strongly optimistic if $E_k(s,a) \geq 0,\forall (s,a) \in \Scal\times\Acal, k\in[K]$. Surpluses are also central to our new regret bounds and we will carefully discuss their use in Appendix~\ref{app:upper_bounds}.

As hinted to in the introduction, the way prior regret bounds treat value-function gaps independently at each state-action pair can lead to excessively loose guarantees.
Bounds  that use value-function gaps \citep{simchowitz2019non,lykouris2019corruption,jin2020simultaneously}
scale at least as
$$\sum_{s,a \colon \gap(s,a) >0 } \frac{H\log(K)}{\gap(s,a)} + 
\sum_{s,a \colon \gap(s,a) =0 }
\frac{H\log(K)}{\gap_{\min}} ,
$$ 
where state-action pairs with zero gap appear, with  $\gap_{\min} = \min_{s, a \colon \gap(s,a) > 0} \gap(s,a)$, the smallest positive gap.
To illustrate where these bounds are loose, let us revisit the example in \pref{fig:summary}. Here, these bounds evaluate to $  \frac{H\log(K)}{c}+\frac{H\log(K)}{\epsilon}+\frac{SH\log(K)}{\epsilon}$, where the first two terms come from state-action pairs with positive value-function gaps and the last term comes from all the state-action pairs with zero gaps. There are several opportunities for improvement:
\begin{enumerate}[label = \textbf{O.\arabic*}]
    \item\label{enum_prob_1} \textbf{State-action pairs that can only be visited by taking optimal actions:} We should not pay the $1/\gap_{\min}$ factor for such $(s,a)$ as there are no other suboptimal policies $\pi$ to distinguish from $\pi^*$ in such states.
    \item\label{enum_prob_3} 
    \textbf{State-action pairs that can only be visited by taking at least one suboptimal action:}
    We should not pay the $1 / \gap(s_2, a_3)$ factor for state-action pair $(s_2, a_3)$ and the $1 / \gap_{\min}$ factor for $(s_2, a_4)$ because no optimal policy visits $s_2$. Such state-action pairs should only be accounted for with the price to learn that $a_2$ is not optimal in state $s_1$. After all,  learning to distinguish between $\pi_1$ and $\pi_2$ is unnecessary for optimal return. 
\end{enumerate}
Both opportunities suggest that the price $\frac{1}{\gap(s,a)}$ or $\frac{1}{\gap_{\min}}$ that each state-action pair $(s,a)$ contributes to the regret bound can be reduced by taking into account the regret incurred by the time $(s,a)$ is reached. Opportunity~\ref{enum_prob_1} postulates that if no regret can be incurred up to (and including) the time step $(s,a)$ is reached, then this state-action pair should not appear in the regret bound. Similarly, if this regret is necessarily large, then the agent can learn this with few observations and stop reaching $(s,a)$ earlier than $\gap(s,a)$ may suggest. Thus, as claimed in \ref{enum_prob_3}, the contribution of $(s,a)$ to the regret should be more limited in this case.

Since the total regret incurred during one episode by a policy $\pi$ is simply the expected sum of value-function gaps visited (\pref{lem:gap_decomp_pi} in the appendix),
\begin{align}
v^* - v^\pi = \EE_{\pi}\left[ \sum_{h=1}^H \gap(S_h, A_h) \right],
\label{eq:reg_decomp}
\end{align}
we can measure the regret incurred up to reaching $(S_{t}, A_{t})$ by the sum of value function gaps $\sum_{h=1}^t \gap(S_h, A_h)$ up to this point $t$. We are interested in the regret incurred up to visiting a certain state-action pair $(s,a)$ which $\pi$ may visit only with some probability. We therefore need to take the expectation of such gaps conditioned on the event that $(s,a)$ is actually visited. We further condition on the event that this regret is nonzero, which is exactly the case when the agent encounters a positive value-function gap within the first $\kappa(s)$ time steps. We arrive at
\begin{align*}
\EE_{\pi}\left[ \sum_{h=1}^{\kappa(s)} \gap(S_h, A_h) ~ \bigg| ~S_{\kappa(s)} = s, A_{\kappa(s)} = a, B \leq \kappa(s) \right],
\end{align*}
where $B = \min \{ h \in [H+1] \colon \gap(S_h, A_h) > 0\}$ is the first time a non-zero gap is visited. This quantity measures the regret incurred up to visiting $(s,a)$ through suboptimal actions. If this quantity is large for all policies $\pi$, then a learner will stop visiting this state-action pair after few observations because it can rule out all actions that lead to $(s,a)$ quickly. Conversely, if the event that we condition on has zero probability under any policy, then $(s,a)$ can only be reached through optimal action choices (including $a$ in $s$) and incurs no regret. This motivates our new definition of gaps that combines value function gaps with the regret incurred up to visiting the state-action pair:
\begin{definition}[Return gap]
\label{def:return_gap}
For any state-action pair $(s, a)\in\Scal\times\Acal$ define $\Bcal(s,a)\equiv \{B \leq \kappa(s), S_{\kappa(s)} = s, A_{\kappa(s)} = a\}$, where $B$ is the first time a non-zero gap is encountered. $\Bcal(s,a)$ denotes the event that state-action pair $(s,a)$ is visited 
and that a suboptimal action was played at any time up to visiting $(s,a)$.
We define the return gap as
\begin{align*}
    \returngap(s,a) \equiv \gap(s,a)\lor\min_{\substack{\pi \in \Pi \colon \\\PP_{\pi}(\Bcal(s,a)) > 0}}&~
    \frac{1}{H} \,
    \EE_{\pi} \bracket*{ \sum_{h = 1}^{\kappa(s)} \gap(S_h, A_h) ~ \bigg| ~\Bcal(s,a) }
\end{align*}
if there is a policy $\pi \in \Pi$ with $\PP_\pi(\Bcal(s,a)) > 0$ and $\returngap(s,a) \equiv 0$ otherwise.
\end{definition}

The additional $1/H$ factor in the second term is a required normalization suggesting that it is the average gap rather than their sum that matters.
We emphasize that Definition~\ref{def:return_gap} is independent of the choice of RL algorithm and in particular does not depend on the algorithm being optimistic. Thus, we expect our main ideas and techniques to be useful beyond the analysis of optimistic algorithms.
Equipped with this definition, we are ready to state our main upper bound which pertains to the \textsc{StrongEuler} algorithm proposed by \citet{simchowitz2019non}.
\begin{theorem}[Main Result (Informal)]
\label{thm:reg_bound_gen_informal}
The regret $\regret(K)$ of  \textsc{StrongEuler} is bounded with high probability for all number of episodes $K$ as
\begin{align*}
    \regret(K) &\lessapprox \sum_{\substack{(s,a)\in\Scal\times\Acal \colon \\\returngap(s,a) > 0}} \frac{\Vcal^*(s,a)}{\returngap(s,a)}
    \log K
    .
\end{align*}
\end{theorem}
In the above, 
we have restricted the bound to only those terms that have inverse polynomial dependence on the gaps.

\paragraph{Comparison with existing gap-dependent bounds.} We now compare our bound to the existing gap-dependent bound for \textsc{StrongEuler} by \citet[Corollary B.1]{simchowitz2019non}
\begin{align}
\label{eq:cor_1b_simchowitz}
    \regret(K) \lessapprox \sum_{\substack{(s,a)\in\Scal\times\Acal \colon \\\gap(s,a) > 0}} \frac{ H\Vcal^*(s,a)}{\gap(s,a)}\log K + \sum_{\substack{(s,a)\in\Scal\times\Acal \colon \\\gap(s,a) = 0}} \frac{H\Vcal^*}{\gap_{\min}}\log K .
\end{align}
We here focus only on terms that admit a dependency on $K$ and an inverse-polynomial dependency on gaps as all other terms are comparable. 
Most notable is the absence of the second term of \pref{eq:cor_1b_simchowitz} in our bound in \pref{thm:reg_bound_gen_informal}. Thus, while state-action pairs with $\returngap(s, a) = 0$ do not contribute to our regret bound, they appear with a $1/\gap_{\min}$ factor in existing bounds. Therefore, our bound addresses \ref{enum_prob_1} because it does not pay for state-action pairs that can only be visited through optimal actions.
Further, state-action pairs that do contribute to our bound satisfy $\frac{1}{\returngap(s,a)} \leq \frac{1}{\gap(s,a)} \wedge \frac{H}{\gap_{\min}}$ and thus never contribute more than in the existing bound in \pref{eq:cor_1b_simchowitz}. Therefore, our regret bound is never worse. 
In fact, it is significantly tighter when there are states that are only reachable by taking severely suboptimal actions, i.e., when the average value-function gaps are much larger than $\gap(s,a)$ or $\gap_{\min}$. By our definition of return gaps, we only pay the inverse of these larger gaps instead of $\gap_{\min}$.
Thus, our bound also addresses \ref{enum_prob_3} and achieves the desired $\log(K)/c$ regret bound in the motivating example of \pref{fig:summary} as opposed to the $\log(K)/\epsilon$ bound of prior work. 

One of the limitations of optimistic algorithms is their $\nicefrac{S}{\gap_{\min}}$ dependence even when there is only one state with a gap of $\gap_{\min}$ \citep{simchowitz2019non}. We note that even though our bound in \pref{thm:reg_bound_gen_informal} improves on prior work, our result does not aim to address this limitation.
Very recent concurrent work \citep{xu2021fine} proposed an action-elimination based algorithm that avoids the $\nicefrac{S}{\gap_{\min}}$ issue of optimistic algorithm but their regret bounds still suffer the issues illustrated in \pref{fig:summary} (e.g.  \ref{enum_prob_3}).
We therefore view our contributions as complementary. In fact, we believe our analysis techniques can be applied to their algorithm as well and result similar improvements as for the example in \pref{fig:summary}.

\paragraph{Regret bound when transitions are deterministic.}
We now interpret \pref{def:return_gap} for MDPs with deterministic transitions and derive an alternative form of our bound in this case.  
Let $\Pi_{s,a}$ be the set of all policies that visit $(s,a)$ and have taken a suboptimal action up to that visit, that is,
$$
\Pi_{s,a} \equiv \curl*{\pi \in \Pi ~\colon s^\pi_{\kappa(s)} = s ,a^\pi_{\kappa(s)} = a, \exists~ h \leq \kappa(s), \gap(s^\pi_{h},a^\pi_{h})>0}.
$$
where $(s^\pi_1, a^{\pi}_1, s^\pi_2, \dots, s^\pi_H, a^{\pi}_H)$ are the state-action pairs visited (deterministically) by $\pi$.
Further,  let $v^{*}_{s,a} = \max_{\pi \in \Pi_{s,a}} v^\pi$ be the best return of such policies.
\pref{def:return_gap} now evaluates to $\returngap(s,a) = \gap(s,a) \vee \frac{1}{H}(v^* - v^{*}_{s,a})$ and the bound in  \pref{thm:reg_bound_gen_informal} can be written as
\begin{align}
\label{eq:det_trans_reg}
    \regret(K) &\lessapprox \sum_{s,a \colon \Pi_{s,a} \neq \varnothing} \frac{H\log(K)}{v^* - v^*_{s,a}}~.
\end{align}
We show in \pref{app:unique_opt_pol}, that it is possible to further improve this bound when the optimal policy is unique by only summing over state-action pairs which are not visited by the optimal policy.

\subsection{Regret analysis with improved clipping: from minimum gap to average gap}
\label{sec:gen_clipping}
In this section, we present the main technical innovations of our tighter regret analysis.
Our framework applies to \emph{optimistic} algorithms that maintain a $Q$-function estimate, $\bar Q_k(s,a)$, which overestimates the optimal $Q$-function $Q^*(s,a)$ with high probability in all states $s$, actions $a$ and episodes $k$.
We first give an overview of gap-dependent analyses and then describe our approach.

\paragraph{Overview of gap-dependent analyses. }
A central quantity in regret analyses of optimistic algorithms are the surpluses $E_k(s,a)$, defined in \pref{eq:surpl_def}, which, roughly speaking, quantify the local amount of optimism.
Worst-case regret analyses bound the regret in episode $k$ as
$\sum_{(s,a) \in \Scal\times\Acal}w_{\pi_k}(s,a)E_{k}(s,a)$, the expected surpluses under the optimistic policy $\pi_k$ executed in that episode. Instead, gap-dependent analyses rely on a tighter version and bound the instantaneous regret by the \emph{clipped surpluses} \citep[e.g. Proposition 3.1][]{simchowitz2019non}
\begin{align}
    V^*(s_1) - V^{\pi_k}(s_1) \leq 2e \sum_{s,a} w^{\pi_k}(s,a)\clip\left[E_{k}(s,a) ~\bigg\vert~ \frac{1}{4H}\gap(s,a) \lor \frac{\gap_{\min}}{2H}\right].
    \label{eqn:old_surplusclipping}
\end{align}

\paragraph{Sharper clipping with general thresholds.}
Our main technical contribution for achieving a regret bound in terms of return gaps $\returngap(s,a)$ is the following improved surplus clipping bound:

\begin{restatable}[Improved surplus clipping bound]{proposition}{surplusclippingbound}
\label{prop:surplus_clipping_bound}
Let the surpluses $E_k(s,a)$ be generated by an optimistic algorithm.
Then the instantaneous regret of $\pi_k$ is bounded as follows:
\vspace{-1mm}
\begin{align*}
    V^*(s_1) - V^{\pi_k}(s_1) 
    \leq 4 \sum_{s,a} w^{\pi_k}(s,a) \clip\left[ E_k(s,a) ~\bigg| ~ \frac 1 4 \gap(s,a) \vee \epsilon_{k}(s, a) \right]~, 
\end{align*}
\vspace{-1mm}
where $\epsilon_k \colon \Scal \times \Acal \rightarrow \RR^+_0$ 
is any clipping threshold function that satisfies
\vspace{-1mm}
\begin{align*}
    \EE_{\pi_k}\left[\sum_{h=B}^H \epsilon_k(S_h, A_h) \right]
    \leq \frac{1}{2} \EE_{\pi_k} \left[\sum_{h=1}^H \gap(S_h, A_h) \right].
\end{align*}

\end{restatable}
Compared to previous surplus clipping bounds in \eqref{eqn:old_surplusclipping}, there are several notable differences. First, instead of $\gap_{\min}/2H$, we can now pair $\gap(s,a)$ with more general clipping thresholds $\epsilon_k(s,a)$, as long as their expected sum over time steps after the first non-zero gap was encountered is at most half the expected sum of gaps. We will provide some intuition for this condition below. Note that $\epsilon_{k}(s,a) \equiv \frac{\gap_{\min}}{2H}$ satisfies the condition because the LHS is bounded between $\frac{\gap_{\min}}{2H} \PP_{\pi_k}( B \leq H)$ and $\gap_{\min} \PP_{\pi_k}( B \leq H)$, and there must be at least one positive gap in the sum $\sum_{h=1}^H\gap(S_h,A_h)$ on the RHS in event $\{B \leq H\}$. Thus our bound recovers existing results.
In addition, the first term in our clipping thresholds is $\frac{1}{4}\gap(s,a)$ instead of $\frac{1}{4H}\gap(s,a)$. \citet{simchowitz2019non} are able to remove this spurious $H$ factor only if the problem instance happens to be a bandit instance and the algorithm satisfies a condition called \emph{strong optimism} where surpluses have to be non-negative. Our analysis does not require such conditions and therefore generalizes these existing results.\footnote{Our layered state space assumption changes $H$ factors in lower-order terms of our final regret compared to \citet{simchowitz2019non}. However, \pref{prop:surplus_clipping_bound} directly applies to their setting with no penalty in $H$.}

\paragraph{Choice of clipping thresholds for return gaps.}
The condition in \pref{prop:surplus_clipping_bound} suggests that one can set $\epsilon_{k}(S_h,A_h)$ to be proportional to the average expected gap under policy $\pi_k$: 
\begin{align}
    \epsilon_{k}(s,a) = 
    \frac{1}{2H}
    \EE_{\pi_k}\left[ \sum_{h=1}^H \gap(S_h, A_h) ~ \bigg| ~\Bcal(s,a)\right].
    \label{eqn:epsilon_choice}
\end{align}
if $\PP_{\pi_k}(\Bcal(s,a)) > 0$ and  $\epsilon_{k}(s,a) = \infty$ otherwise.  
\pref{lem:clipping_gaps_rel} in \pref{app:upper_bounds} shows that this choice indeed satisfies the condition in \pref{prop:surplus_clipping_bound}. 
If we now take the minimum over all policies for $\pi_k$, then we can proceed with the standard analysis and derive our main result in \pref{thm:reg_bound_gen_informal}. However, by avoiding the minimum over policies, we can derive a stronger policy-dependent regret bound which we discuss in the appendix.

\section{Instance-dependent lower bounds}
\label{sec:lower_bounds_main}
We here shed light on what properties on an episodic MDP determine the statistical difficulty of RL by deriving information-theoretic lower bounds on the asymptotic expected regret of any (good) algorithm. To that end, we first derive a general result that expresses a lower bound as the optimal value of a certain optimization problem and then derive closed-form lower-bounds from this optimization problem that depend on certain notions of gaps for two special cases of episodic MDPs.

Specifically, in those special cases, we assume that the rewards follow a Gaussian distribution with variance $1/2$. We further assume that the optimal value function is bounded in the same range as individual rewards, e.g.\ as $0 \leq V^*(s) < 1$ for all $s \in \Scal$. This assumption is common in the literature \citep[e.g.][]{krishnamurthy2016pac, jiang2017contextual, dann2018oracle} and can be considered harder than a normalization of $V^*(s) \in [0, H]$ \citep{jiang2018open}.

\subsection{General instance-dependent lower bound as an optimization problem}

The idea behind deriving instance-dependent lower bounds for the stochastic MAB problem~\citep{lai1985asymptotically,combes2017minimal,garivier2019explore} and infinite horizon MDPs~\citep{graves1997asymptotically,ok2018exploration} are based on first assuming that the algorithm studied is \emph{uniformly good}, that is, on any instance of the problem and for any $\alpha >0$, the algorithm incurs regret at most $o(T^\alpha)$, and then argue that, to achieve that guarantee, the algorithm must select a certain policy or action at least some number of times as it would otherwise not be able to distinguish the current MDP from another MDP that requires a different optimal strategy. 

Since comparison between different MDPs is central to lower-bound constructions, it is convenient to make the problem-instance explicit in the notation. To that end, let $\Theta$ be the problem class of possible MDPs and we use subscripts $\theta$ and $\lambda$ for value functions, return, MDP parameters etc., to denote specific problem instances $\theta, \lambda \in \Theta$ of those quantities. Further, for a policy $\pi$ and MDP $\theta$, $\PP_\theta^{\pi}$ denotes the law of one episode, i.e., the distribution of $(S_1, A_1, R_1, S_2, A_2, R_2, \dots, S_{H+1})$. To state the general regret lower-bound we need to introduce the set of \emph{confusing} MDPs. This set consists of all MDPs $\lambda$ in which there is at least one optimal policy $\pi$ such that $\pi \not\in \Pi^*_\theta$, i.e., $\pi$ is not optimal for the original MDP and no policy in $\Pi^*_\theta$ has been changed.
\begin{definition}
\label{def:conf_MDP_set}
For any problem instance $\theta\in\Theta$ we define the set of confusing MDPs $\Lambda(\theta)$ as
\begin{align*}
    \Lambda(\theta): = \{\lambda \in \Theta \colon \Pi^*_\lambda \setminus  \Pi^*_\theta \neq \varnothing \textrm{ and } KL(\PP_\theta^{\pi}, \PP_\lambda^{\pi}) = 0 \,\,\forall \pi \in \Pi^*_\theta\}.
\end{align*}
\end{definition}
We are now ready to state our general regret lower-bound for episodic MDPs:

\begin{restatable}[General instance-dependent lower bound for episodic MDPs]{theorem}{generallb}
\label{thm:lower_bound_gen}
Let $\psi$ be a uniformly good RL algorithm for $\Theta$, that is, for all problem instances $\theta \in \Theta$ and exponents $\alpha > 0$, the regret of $\psi$ is bounded as $\EE[\regret_\theta(K)] \leq o(K^{\alpha})$, and assume that $\return{*}_\theta < H$. Then, for any $\theta \in \Theta$, the regret of $\psi$ satisfies
\begin{align*}
    \liminf_{K \to \infty} \frac{\EE[\regret_\theta(K)]}{\log{K}} \geq C(\theta),
\end{align*}
where $C(\theta)$ is the optimal value of the following optimization problem
\begin{equation}
\label{eq:opt_prob}
\begin{aligned}
    \minimize{\eta(\pi)\geq 0}{\sum_{\pi \in \Pi} \eta(\pi)\left(\return{*}_{\theta} - \return{\pi}_{\theta}\right)}
    {
    \sum_{\pi \in \Pi} \eta(\pi) KL(\PP_\theta^\pi,\PP_\lambda^\pi) \geq 1 \qquad \textrm{for all } \,\,\lambda \in \Lambda(\theta)
    }.
\end{aligned}
\end{equation}
\end{restatable}

 The optimization problem in \pref{thm:lower_bound_gen} can be interpreted as follows. The variables $\eta(\pi)$ are the (expected) number of times the algorithm chooses to play policy $\pi$ which makes the objective the total expected regret incurred by the algorithm.
 The constraints encode that any uniformly good algorithm needs to be able to distinguish the true instance $\theta$ from all confusing instances $\lambda \in \Lambda(\theta)$, because otherwise it would incur linear regret. To do so, a uniformly good algorithm needs to play policies $\pi$ that induce different behavior in $\lambda$ and $\theta$ which is precisely captured by the constraints $\sum_{\pi\in\Pi}\eta(\pi) KL(\PP_\theta^{\pi}, \PP_\lambda^{\pi}) \geq 1$.

Although \pref{thm:lower_bound_gen} has the flavor of results in the bandit and RL literature, there are a few notable differences.
Compared to lower-bounds in the infinite-horizon MDP setting \citep{graves1997asymptotically,tewari2008optimistic,ok2018exploration}, we for example do not assume that the Markov chain induced by an optimal policy $\pi^*$ is irreducible. That irreducibility plays a key role in converting the semi-infinite linear program \pref{eq:opt_prob}, which typically has uncountably many constraints, into a linear program with only $O(SA)$ constraints. While for infinite horizon MDPs, irreducibility is somewhat necessary to facilitate exploration, this is not the case for the finite horizon setting and in general we cannot obtain a convenient reduction of the set of constraints $\Lambda(\theta)$ (see also \pref{app:lower_bounds_full_supp}).

\subsection{Gap-dependent lower bound when optimal policies visit all states}
To derive closed-form gap-dependent bounds from the general optimization problem \pref{eq:opt_prob}, we need to identify a finite subset of confusing MDPs $\Lambda(\theta)$ that each require the RL agent to play a distinct set of policies that do not help to distinguish the other confusing MDPs. To do so, we restrict our attention to the special case of MDPs where every state is visited with non-zero probability by some optimal policy, similar to the irreducibility assumptions in the infinite-horizon setting \citep{tewari2008optimistic, ok2018exploration}. In this case, it is sufficient to raise the expected immediate reward of a suboptimal $(s,a)$ by $\gap_\theta(s,a)$ in order to create a confusing MDP, as shown in \pref{lem:non-empty_change_env}:

\begin{restatable}[]{lemma}{gaplemmafullsupp}
\label{lem:non-empty_change_env}
Let $\Theta$ be the set of all episodic MDPs with Gaussian immediate rewards and optimal value function uniformly bounded by 1 and let $\theta \in \Theta$ be an MDP in this class. Then for any suboptimal state-action pair $(s,a)$ with $\gap_\theta(s,a) > 0$ such that $s$ is visited by some optimal policy with non-zero probability, there exists a confusing MDP $\lambda \in \Lambda(\theta)$ with
\begin{itemize}[itemsep=1mm, topsep=1mm]
    \item $\lambda$ and $\theta$ only differ in the immediate reward at $(s,a)$
    \item $KL(\PP_\theta^{\pi}, \PP_\lambda^{\pi}) \leq  \gap_\theta(s,a)^2$ for all $\pi \in \Pi$.
\end{itemize}
\end{restatable}

By relaxing the problem in \pref{eq:opt_prob} to only consider constraints from the confusing MDPs in \pref{lem:non-empty_change_env} with $KL(\PP_\theta^{\pi},\PP_\lambda^{\pi}) \leq \gap_\theta(s,a)^2$, for every $(s,a)$, we can derive the following closed-form bound:

\begin{restatable}[Gap-dependent lower bound when optimal policies visit all states]{theorem}{fullsupportlb}
\label{thm:lower_bound_all_states_supp}
Let $\Theta$ be the set of all episodic MDPs with Gaussian immediate rewards and optimal value function uniformly bounded by 1. Let $\theta \in \Theta$ be an instance where every state is visited by some optimal policy with non-zero probability. Then any uniformly good algorithm on $\Theta$ has expected regret on $\theta$ that satisfies
\begin{align*}
    \liminf_{K\rightarrow \infty}\frac{\EE[\regret_\theta(K)]}{\log{K}} \geq \sum_{s,a \colon \gap_\theta(s,a) > 0} \frac{1}{\gap_\theta(s,a)}.
\end{align*}
\end{restatable}
\pref{thm:lower_bound_all_states_supp} can be viewed as a generalization of Proposition 2.2 in \citet{simchowitz2019non}, which gives a lower bound of order  $\sum_{s,a \colon \gap_\theta(s,a) > 0} \frac{H}{\gap_\theta(s,a)}$ for a certain set of MDPs.\footnote{We translated their results to our setting where $V^* \leq 1$ which reduces the bound by a factor of $H$.} While our lower bound is a factor of $H$ worse, it is significantly more general and holds in any MDP where optimal policies visit all states and with appropriate normalization of the value function.
\pref{thm:lower_bound_all_states_supp} indicates that value-function gaps characterize the instance-optimal regret when optimal policies cover the entire state space.

\subsection{Gap-dependent lower bound for deterministic-transition MDPs} 
\label{sec:lower_bound_def}
We expect that optimal policies do not visit all states in most MDPs of practical interest (e.g. because certain parts of the state space can only be reached by making an egregious error). We therefore now consider the general case where $\bigcup_{\pi \in \Pi^*_\theta} supp(\pi) \subsetneq \Scal$ but restrict our attention to MDPs with deterministic transitions where we are able to give an intuitive closed-form lower bound.
Note that deterministic transitions imply $\forall \pi,s,a:\,w^\pi(s,a)\in\{0,1\}$. 
Here, a confusing MDP can be created by simply raising the reward of any $(s,a)$ by
\begin{align}
\label{eq:return_gap}
    \return{*}_\theta - \max_{\pi \colon w^\pi_\theta(s,a) > 0}\return{\pi}_\theta~, 
\end{align}
the regret of the best policy that visits $(s,a)$, as long as it is positive and $(s,a)$ is not visited by any optimal policy. \pref{eq:return_gap} is positive when no optimal policy visits $(s,a)$ in which case suboptimal actions have to be taken to reach $(s,a)$ and $\returngap_\theta(s,a) > 0$. Let $\pi^*_{(s,a)}$ be any maximizer in \pref{eq:return_gap}, which has to act optimally after visiting $(s,a)$. From the regret decomposition in \pref{eq:reg_decomp} and the fact that $\pi^*_{(s,a)}$ visits $(s,a)$ with probability $1$, it follows that $v_{\theta}^* - v_{\theta}^{\pi^*_{(s,a)}} \geq \gap_{\theta}(s,a)$. We further have $v_{\theta}^* - v_{\theta}^{\pi^*_{(s,a)}} \leq H\returngap_\theta(s,a)$.
Equipped with the subset of confusing MDPs $\lambda$ that each raise the reward of a single $(s,a)$ as 
$r_\lambda(s,a) = r_\theta(s,a) + v^*_{\theta} - v^{\pi^*_{(s,a)}}_{\theta}$, we can derive the following gap-dependent lower bound:

\begin{restatable}{theorem}{lowerbounddeterministic}
\label{thm:lower_bound_deterministic}
Let $\Theta$ be the set of all episodic MDPs with Gaussian immediate rewards and optimal value function uniformly bounded by 1. Let $\theta \in \Theta$ be an instance with deterministic transitions. Then any uniformly good algorithm on $\Theta$ has expected regret on $\theta$ that satisfies
\begin{align*}
    \liminf_{K \rightarrow \infty} \frac{\EE[\regret_\theta(K)]}{\log K} 
    \geq \sum_{s, a \in \Zcal_\theta \colon \returngap_\theta(s,a) > 0 }
    \frac{1}{H \cdot (v^*_{\theta}-v^{\pi^*_{(s,a)}}_{\theta})} \geq \sum_{s, a \in \Zcal_\theta \colon \returngap_\theta(s,a) > 0 }
    \frac{1}{H^2 \cdot \returngap_\theta(s,a)},
\end{align*}
where $\Zcal_\theta = \{ (s,a) \in \Scal \times \Acal \colon \forall \pi^* \in \Pi^*_\theta ~~~ w^{\pi^*}_\theta(s,a) = 0\}$ is the set of state-action pairs that no optimal policy in $\theta$ visits.
\end{restatable}

We now compare the above lower bound to the upper bound guaranteed by \textsc{StrongEuler} in \pref{eq:det_trans_reg}. The comparison is only with respect to number of episodes and gaps{\footnote{We carry out the comparison in expectation, since our lower bounds do not apply with high probability.}}
\begin{align*}
    \sum_{s,a \in \Zcal_\theta \colon \returngap_\theta(s,a) > 0} \frac{\log(K)}{H^2\returngap_\theta(s,a)} \leq \mathbb{E}_\theta[\regret(K)] \leq \sum_{s,a \colon \returngap_\theta(s,a) > 0}\frac{\log(K)}{\returngap_\theta(s,a)}.
\end{align*}
The difference between the two bounds, besides the extra $H^2$ factor, is the fact that $(s,a)$ pairs that are visited by any optimal policy ($s,a \neq \Zcal_\theta$) do not appear in the lower-bound while the upper-bound pays for such pairs if they can also be visited after playing a suboptimal action. This could result in cases where the number of terms in the lower bound is $O(1)$ but the number of terms in the upper bound is $\Omega(SA)$ leading to a large discrepancy. In \pref{thm:det_lower_bound} in the appendix we show that there exists an MDP instance on which it is information-theoretically possible to achieve $O(\log(K)/\epsilon)$ regret, however, any optimistic algorithm with confidence parameter $\delta$ will incur expected regret of at least $\Omega(S\log(1/\delta)/\epsilon)$. \pref{thm:det_lower_bound} has two implications for optimistic algorithms in MDPs with deterministic transitions. Specifically, optimistic algorithms
\begin{itemize}[topsep=1mm, itemsep=0mm]
    \item cannot be asymptotically optimal if confidence parameter $\delta$ is tuned to the time horizon $K$;
    \item cannot have an anytime bound that matches the information-theoretic lower bound.
\end{itemize}


\section{Conclusion}
In this work, we prove that optimistic algorithms such as \textsc{StrongEuler}, can suffer substantially less regret compared to what prior work had shown. We do this by introducing a new notion of gap, while greatly simplifying and generalizing existing analysis techniques.
We further investigated the information-theoretic limits of learning episodic layered MDPs. 
We provide two new closed-form lower bounds in the special case where the MDP has either 
deterministic transitions or the optimal policy is supported on all states.
These lower bounds suggest that our notion of gap better captures the difficulty of an episodic MDP for RL.

\bibliographystyle{plainnat}
\bibliography{mybib,cdann_thesis_lib}
\clearpage
\appendix
\renewcommand{\contentsname}{Contents of main article and appendix}
\tableofcontents
\addtocontents{toc}{\protect\setcounter{tocdepth}{3}}
\clearpage

\section*{Checklist}


\begin{enumerate}

\item For all authors...
\begin{enumerate}
  \item Do the main claims made in the abstract and introduction accurately reflect the paper's contributions and scope?
    \answerYes{See \pref{sec:upper_bounds}, \pref{sec:lower_bounds_main} and corresponding sections in the appendix.}
  \item Did you describe the limitations of your work?
    \answerYes{See lower bounds, discussion after Equation~\ref{eq:det_trans_reg}, \pref{app:lower_bounds_issues}}
  \item Did you discuss any potential negative societal impacts of your work?
    \answerNA{Our work is theoretical and we do not see any potential negative societal impacts.}
  \item Have you read the ethics review guidelines and ensured that your paper conforms to them?
    \answerYes{}
\end{enumerate}

\item If you are including theoretical results...
\begin{enumerate}
  \item Did you state the full set of assumptions of all theoretical results?
    \answerYes{}
	\item Did you include complete proofs of all theoretical results?
    \answerYes{See \pref{app:lower_bounds} for lower bounds and \pref{app:upper_bounds} for upper bounds.}
\end{enumerate}

\item If you ran experiments...
\begin{enumerate}
  \item Did you include the code, data, and instructions needed to reproduce the main experimental results (either in the supplemental material or as a URL)?
    \answerNA{}
  \item Did you specify all the training details (e.g., data splits, hyperparameters, how they were chosen)?
    \answerNA{}
	\item Did you report error bars (e.g., with respect to the random seed after running experiments multiple times)?
    \answerNA{}
	\item Did you include the total amount of compute and the type of resources used (e.g., type of GPUs, internal cluster, or cloud provider)?
    \answerNA{}
\end{enumerate}

\item If you are using existing assets (e.g., code, data, models) or curating/releasing new assets...
\begin{enumerate}
  \item If your work uses existing assets, did you cite the creators?
    \answerNA{}
  \item Did you mention the license of the assets?
    \answerNA{}
  \item Did you include any new assets either in the supplemental material or as a URL?
    \answerNA{}
  \item Did you discuss whether and how consent was obtained from people whose data you're using/curating?
    \answerNA{}
  \item Did you discuss whether the data you are using/curating contains personally identifiable information or offensive content?
    \answerNA{}
\end{enumerate}

\item If you used crowdsourcing or conducted research with human subjects...
\begin{enumerate}
  \item Did you include the full text of instructions given to participants and screenshots, if applicable?
    \answerNA{}
  \item Did you describe any potential participant risks, with links to Institutional Review Board (IRB) approvals, if applicable?
    \answerNA{}
  \item Did you include the estimated hourly wage paid to participants and the total amount spent on participant compensation?
    \answerNA{}
\end{enumerate}

\end{enumerate}

\section{Related work}
We now discuss related work carefully. Instance dependent regret lower bounds for the MAB were first introduced in \citet{lai1985asymptotically}. Later \citet{graves1997asymptotically} extend such instance dependent lower bounds to the setting of controlled Markov chains, while assuming infinite horizon and certain properties of the stationary distribution of each policy. Building on their work, more recently \citet{combes2017minimal} establish instance dependent lower bounds for the Structured Stochastic Bandit problem. Very recently, in the stochastic MAB, \citet{garivier2019explore} generalize and simplify the techniques of \citet{lai1985asymptotically} to completely characterize the behavior of uniformly good algorithms. The work of \citet{ok2018exploration} builds on these ideas to provide an instance dependent lower bound for infinite horizon MDPs, again under assumptions of how the stationary distributions of each policy will behave and irreducibility of the Markov chain. The idea behind deriving the above bounds is to use the uniform goodness of the studied algorithm to argue that the algorithm must select a certain policy or action at least a fixed number of times. This number is governed by a change of environment under which said policy/action is now the best overall. The reasoning now is that unless the algorithm is able to distinguish between these two environments it will have to incur linear regret asymptotically. Since the algorithm is uniformly good this can not happen.

For infinite horizon MDPs with additional assumptions the works of \citet{auer2007logarithmic,tewari2008optimistic,auer2009near,filippi2010optimism,ok2018exploration} establish logarithmic in horizon regret bounds of the form $O(D^2S^2A\log(T)/\delta)$, where $\delta$ is a gap-like quantity and $D$ is a diameter measure. We now discuss the works of \citep{tewari2008optimistic,ok2018exploration}, which should give more intuition about how the infinite horizon setting differs from our setting. Both works consider the non-episodic problem and therefore make some assumptions about the MDP $\Mcal$. The main assumption, which allows for computationally tractable algorithms is that of irreducibility. Formally both works require that under any policy the induced Markov chain is irreducible. Intuitively, the notion of irreducibility allows for coming up with exploration strategies, which are close to min-max optimal and are easy to compute. In \citep{ok2018exploration} this is done by considering the same semi-infinite LP~\ref{eq:opt_prob} as in our work. Unlike our work, however, assuming that the Markov chain induced by the optimal policy $\pi^*$ is irreducible allows for a nice characterization of the set $\Lambda(\theta)$ of "confusing" environments. In particular the authors manage to show that at every state $s$ it is enough to consider the change of environment which makes the reward of any action $a : (s,a) \not\in \pi^*$ equal to the reward of $a' : (s,a') \in \pi^*$. Because of the irreducability assumption we know that the support of $P(\cdot|s,a)$ is the same as the support of $P(\cdot |s,a')$ and this implies that the above change of environment makes the policy $\pi$ which plays $(s,a)$ and then coincides with $\pi^*$ optimal. Some more work shows that considering only such changes of environment is sufficient for an equivalent formulation to the LP\ref{eq:opt_prob}. Since this is an LP with at most $S\times A$ constraints it is solvable in polynomial time and hence a version of the algorithm in \citep{combes2017minimal} results in asymptotic min-max rates for the problem. The exploration in \citep{tewari2008optimistic} is also based on a similar LP, however, slightly more sophisticated.

Very recently there has been a renewed interest in proposing instance dependent regret bounds for finite horizon tabular MDPs~\citep{simchowitz2019non,lykouris2019corruption,jin2020simultaneously}. The works of \citep{simchowitz2019non,lykouris2019corruption} are based on the OFU principle and the proposed regret bounds scale as $O(\sum_{(s,a) \not\in \pi^*} H\log(T)/\gap(s,a) + SH\log(T)/\gap_{\min})$, disregarding variance terms and terms depending only poli-logarithmically on the gaps. The setting in \citep{lykouris2019corruption} also considers adversarial corruptions to the MDP, unknown to the algorithm, and their bound scales with the amount of corruption. \citet{jin2020simultaneously} derive similar upper bounds, however, the authors assume a known transition kernel and take the approach of modelling the problem as an instance of Online Linear Optimization, through using occupancy measures~\citep{zimin2013online}. For the problem of $Q$-learning, \citet{yang2020q,du2020agnostic}, also propose algorithms with regret scaling as $O(SAH^6\log(T)/\gap_{\min})$.
All of these bounds scale at least as $\Omega(SH\log(T)/\gap_{\min})$. 
\citet{simchowitz2019non} show an MDP instance on which no optimistic algorithm can hope to do better.

\section{Model-based optimistic algorithms for tabular RL}
\label{app:opt_algs}

This section is a general discussion of optimistic algorithms for the tabular setting. Our regret upper bounds can be extended to other model based optimistic algorithms or in general any optimistic algorithm for which we can show a meaningful bound on the surpluses in terms of the number of times a state-action pair has been visited throughout the $K$ episodes.
\begin{algorithm}[t]
\caption{Generic Model-Based Optimistic Algorithm for Tabular RL}
\label{alg:model_based_alg}
\begin{algorithmic}[1]
\REQUIRE{Number of episodes $K$, horizon $H$, number of states $S$, number of actions $A$, probability of failure $\delta$.}
\ENSURE{A sequence of policies $(\pi_k)_{k=1}^K$ with low regret.}
\STATE Initialize empirical transition kernel $\hat P \in [0,1]^{S\times A\times S}$, empirical reward kernel $\hat r \in [0,1]^{S\times A}$, bonuses $b \in [0,1]^{S\times A}$.
\FOR{$k\in [K]$}
\STATE $h=H$, $Q_k(s_{H+1},a_{H+1}) = 0,\forall (s,a) \in \Scal\times\Acal$.
\WHILE{$h>0$} 
\STATE $Q_k(s,a) = \hat r(s,a) + \langle \hat P(\cdot|s,a), V_k \rangle + b(s,a)$.
\STATE $\pi_k(s) := \argmax_{a} Q_k(s,a)$.
\STATE $h-=1$
\ENDWHILE
\STATE Play $\pi_k$, collect observations from transition kernel $P$ and reward kernel $r$ and update $\hat P$, $\hat r$, $b$.
\ENDFOR
\end{algorithmic}
\end{algorithm}

Pseudo-code for a generic algorithm can be found in Algorithm~\ref{alg:model_based_alg}. The algorithm begins by initializing an empirical transition kernel $\hat P \in [0,1]^{S\times A\times S}$, empirical reward kernel $\hat r \in [0,1]^{S\times A}$, and bonuses $b \in [0,1]^{S\times A}$. If we let $n_k(s,a)$ be the number of times we have observed state-action pair $(s,a)$ up to episode $k$ and $n_k(s',s,a)$ the number of times we have observed state $s'$ after visiting $(s,a)$ then one standard way to define the empirical kernels at episode $k$ are as follows:
\begin{equation}
    \hat r(s,a) = \frac{1}{n_k(s,a)}\sum_{j=1}^k R_j(s,a), \qquad
    \hat P(s'|s,a) = 
    \begin{cases*}
        \frac{n_k(s',s,a)}{n_k(s,a)} &\text{if} $n_k(s,a)>0$ \\
        0 &otherwise
    \end{cases*}\\
\end{equation}
where $R_j(s,a)$ is a sample from $r(s,a)$ at episode $j$ if $(s,a)$ was visited and $0$ otherwise.
At every episode the generic algorithm constructs an policy $\pi_k$ using the empirical model together with bonus terms $b(s,a),\forall (s,a)\in\Scal\times\Acal.$ Bonuses are constructed by using concentration of measure results relating $\hat r(s,a)$ to $r(s,a)$ and $\hat P(\cdot|s,a)$ to $P(\cdot|s,a)$. These bonuses usually scale inversely with the empirical visitations $n_k(s,a), \forall (s,a) \in \Scal\times \Acal$, as $O(1/\sqrt{n_k(s,a)})$. Further, depending on the type of concentration of measure result, the bonuses could either have a direct dependence on $K,H,S,A,\delta$ (following from Azuma-Hoeffding style concentration bounds) or replace $H$ with the empirical estimator (following Freedman style concentration bounds). 
The bonus terms ensure that optimism is satisfied for $\pi_k$, that is $Q_k(s,a) \geq Q^{\pi_k}(s,a)$ for all $(s,a)\in \Scal\times\Acal$ and all episodes $k \in [K]$ with probability at least $1-\delta$. Algorithms such as \textsc{UCBVI}~\citep{azar2017minimax}, \textsc{Euler}~\citep{zanette2019tighter} and \textsc{StrongEuler}~\citep{simchowitz2019non} are all versions of Algorithm~\ref{alg:model_based_alg} with different instantiations of the bonus terms.

The greedy choice of $\pi_k$ together with optimism also ensures that $V_k(s) \geq V^*(s)$. This has been key in prior work as it is what allows to bound the instantaneous regret by the sum of surpluses and ultimately relate the regret upper bound back to the bonus terms and the number of visits of each state-action pair respectively.
Our regret upper bounds are also based on this decomposition and as such are not really tied to the \textsc{StrongEuler} algorithm but would work with any model-based optimistic algorithm for the tabular setting. The main novelty in this work is a way to control the surpluses by clipping them to a gap-like quantity which better captures the sub-optimality of $\pi_k$ compared to $\pi^*$. We remark that our analysis can be extended to any algorithm which follows Algorithm~\ref{alg:model_based_alg} so as long as we can control the bonus terms sufficiently well.
\section{Experimental results}
\label{app:experiments}

In this section we present experiments based on the following deterministic LP which can be found in \pref{fig:mdp_experiments}. 
\begin{figure}[ht]
    \centering
    \includegraphics[width=.6\textwidth]{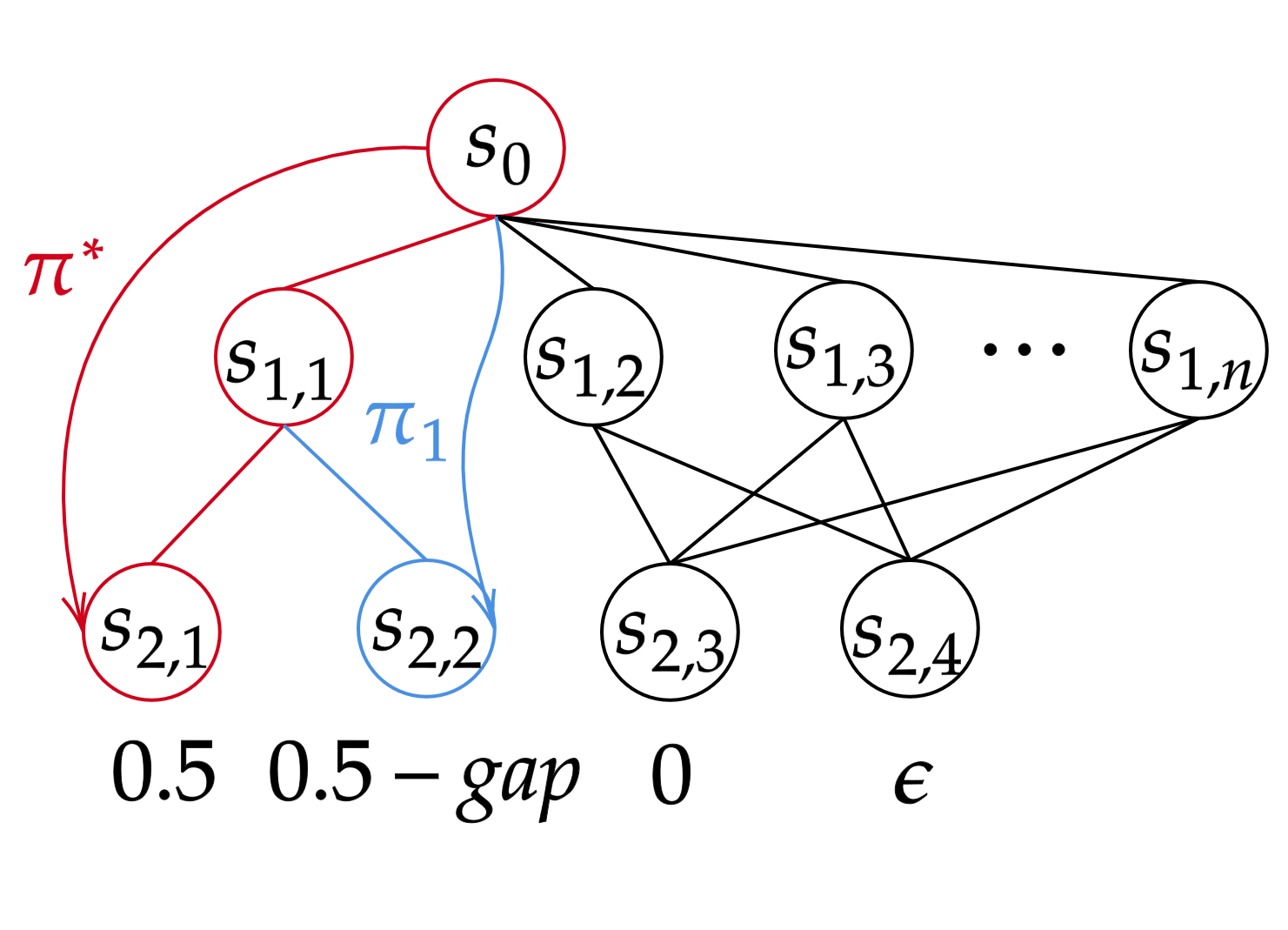}
    \caption{Deterministic MDP used in experiments}
    \label{fig:mdp_experiments}
\end{figure}
In short the MDP has only deterministic transitions and 3 layers. The starting state is denoted by $s_0$ and the $j$-th state at layer $i$ by $s_{i,j}$. There are $n+1$ possible actions at $s_0$, two possible actions at $s_{1,j},\forall j\in [n+1]$, and a single possible action at $s_{2,j}, \forall j\in[4]$. The only non-negative rewards are at state-action pairs in the final layer. The unique optimal policy reaches state $s_{2,1}$ and has return equal to $0.5$. We distinguish between two types of sub-optimal policies given by $\pi_1$ which visists $s_{1,1}$ and all other sub-optimal policies which visit $s_{1,j},j\geq 2$. The return of policy $\pi_1$ determines the $gap$ parameter in our experiments and the reward at state $s_{2,4}$ determines the $\epsilon$ parameter.

We run two sets of experiments using the \textsc{UCBVI} algorithm~\citep{azar2017minimax}. We have chosen this algorithm over Strong-\textsc{Euler} since \textsc{UCBVI} is slightly easier to implement and their differences are orthogonal to the issues studied here.  The rewards in both experiments are Bernoulli with the respective mean provided below the state in \pref{fig:mdp_experiments}. In the first set of experiments we let the gap parameter to be equal to $0.5$ and in the second set of experiments we let the gap parameter to be $\sqrt{\frac{S}{K}}$. We let $\epsilon = \frac{4^{\epsilon_{pow}}}{\sqrt{K}}$, where $\epsilon_{pow}$ takes integer values between $0$ and $\lfloor0.5*\log_4(K)\rfloor$. We have two settings for $n$ (respectively $S$) which are $n=1$ and $n=250$. In all experiments we have set $K=500000$ and the topology of the MDP implies $H=3$. Each experiment is repeated $5$ times and we report the average regret of the algorithm, together with standard deviation of the regret. We note that in the first set of experiments we should observe regret which is close to $\Theta(\frac{SA\log(T)}{gap})$, this is because with our parameter choices the return gap is $gap/2$ for all settings of $\epsilon$. In the second set of experiments we should observe regret which is close to $\Theta(\sqrt{SAK})$ as the min-max regret bounds dominate.
\begin{figure}[ht]
    \centering
    \begin{subfigure}[b]{0.45\textwidth}
        \centering
        \includegraphics[width=1.\textwidth]{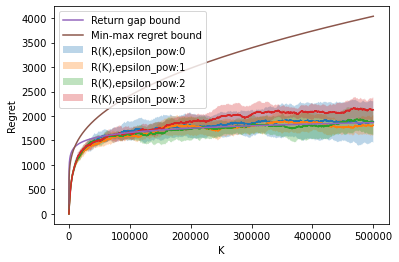}
        \caption{$n=1$}
    \end{subfigure}
    \begin{subfigure}[b]{0.45\textwidth}
        \centering
        \includegraphics[width=1.\textwidth]{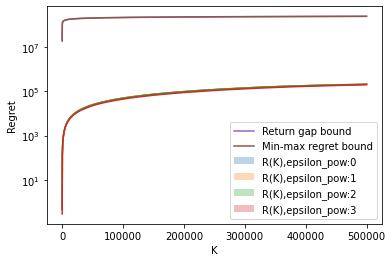}
        \caption{$n=250$}
    \end{subfigure}
    \caption{Large gap experiments}
    \label{fig:large_gap}
\end{figure}

The first set of experiments can be found in Figure~\ref{fig:large_gap}. We plot $S^2A + \frac{SA\log(T)}{gap}$ in purple and $S^2A + \sqrt{SAK}$ in brown for reference. We include the additive term of $S^2A$ as this is what the theoretical regret bounds suggest. We see that for $n=1$ our experiments almost perfectly match theory, including the observations made regarding Opportunity~\ref{enum_prob_1} and Opportunity~\ref{enum_prob_3}. In particular there is no obvious dependence on $1/\gap_{\min} = 1/\epsilon$, especially when $\epsilon = O(1/\sqrt{K})$, which in the plot is reflected by $\epsilon_{pow} = 0$.
In the case for $n=250$ the algorithm performs better than what our theory suggests. We expect that our bounds do not accurately capture the dependence on $S$ and $A$, at least for deterministic transition MDPs.
\begin{figure}[ht]
    \centering
    \begin{subfigure}[b]{0.45\textwidth}
        \centering
        \includegraphics[width=1.\textwidth]{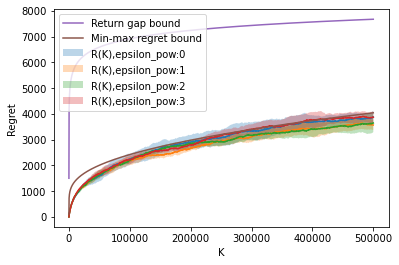}
        \caption{$n=1$}
    \end{subfigure}
    \begin{subfigure}[b]{0.45\textwidth}
        \centering
        \includegraphics[width=1.\textwidth]{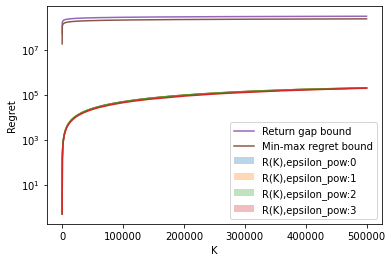}
        \caption{$n=250$}
    \end{subfigure}
    \caption{Small gap experiments}
    \label{fig:small_gap}
\end{figure}
The second set of experiments can be found in Figure~\ref{fig:small_gap}. Similar observations hold as in the large gap experiment.

\section{Additional Notation}
We use the shorthand $(s, a) \in \pi$ to indicate that $\pi$ admits a non-zero probability of visiting the state-action pair $(s,a)$ and abusively use $\pi$ as the set of such state-action pairs, when convenient.

\section{Proofs and extended discussion for regret lower-bounds}
\label{app:lower_bounds}

Let $N_{\psi,\pi}(k)$ be the random variable denoting the number of times policy $\pi$ has been chosen by the strategy $\psi$. Let $N_{\psi, (s,a)}(k)$ be the number of times the state-action pair has been visited up to time $k$ by the strategy $\psi$.

\subsection{Lower bound as an optimization problem}
\label{sec:opt_lower_bound}

We begin by formulating an LP characterizing the minimum regret incurred by any uniformly good algorithm $\psi.$
\begin{theorem}
\label{thm:lower_bound_gen1}
Let $\psi$ be a uniformly good RL algorithm for $\Theta$, that is, for all problem instances $\theta \in \Theta$ and exponents $\alpha > 0$, the regret of $\psi$ is bounded as $\EE[\regret_\theta(K)] \leq o(K^{\alpha})$. Then, for any $\theta \in \Theta$, the regret of $\psi$ satisfies
\begin{align*}
    \liminf_{K \to \infty} \frac{\EE[\regret_\theta(K)]}{\log{K}} \geq C(\theta),
\end{align*}
where $C(\theta)$ is the optimal value of the following optimization problem
\begin{equation}
\label{eq:opt_prob1}
\begin{aligned}
    \minimize{\eta(\pi)\geq 0}{\sum_{\pi \in \Pi} \eta(\pi)\left(\return{*}_{\theta} - \return{\pi}_{\theta}\right)}
    {
    \sum_{\pi \in \Pi} \eta(\pi) KL(\PP_\theta^\pi,\PP_\lambda^\pi) \geq 1 \qquad \textrm{for all } \,\,\lambda \in \Lambda(\theta)
    },
\end{aligned}
\end{equation}
where $\Lambda'(\theta) = \{ \lambda \in \Theta \colon \Pi^{*}_{\lambda} \cap \Pi^{*}_{\theta} = \varnothing, KL(\PP_\theta^{\pi^*_\theta}, \PP_\lambda^{\pi^*_\theta}) = 0\}$ are all environments that share no optimal policy with $\theta$ and do not change the rewards or transition kernel on $\pi^*$.
\end{theorem}
\begin{proof}
 We can write the expected regret as $\EE[\regret_\theta(K)] = \sum_{\pi \in \Pi} \mathbb{E}_\theta[N_{\psi,\pi}(K)](\return{*}_\theta - \return{\pi}_\theta)$. We will show that $\eta(\pi) = \mathbb{E}_\theta[N_{\psi,\pi}(K)]/\log{K}$ is feasible for the optimization problem in \pref{eq:opt_prob}. This is sufficient to prove the theorem. To do so we follow the techniques of~\cite{garivier2019explore}.
 With slight abuse of notation, let $\mathbb{P}_\theta^{I_k}$ be the law of all trajectories up to episode $k$, where $I_k$ is the history up to and including time $k$. Let $Y_k$ be the random variable which is the value function of the policy, $\psi(I_k)$, selected at episode $k$. We have
\begin{equation}
\label{eq:kl-expand}
    \begin{aligned}
        KL(\mathbb{P}_\theta^{I_{k+1}},\mathbb{P}_\lambda^{I_{k+1}}) &= KL(\mathbb{P}_\theta^{Y_{k+1},I_{k}},\mathbb{P}_\lambda^{Y_{k+1},I_{k}})\\
        &=
        KL(\mathbb{P}_\theta^{I_{k}},\mathbb{P}_\lambda^{I_{k}})
        +\mathbb{E}\left[ \mathbb{E}_{\PP^{\psi(I_k)}_\theta}\left[\log{\frac{\PP^{\psi(I_k)}_\theta(Y_{k+1})}{\PP^{\psi(I_k)}_\lambda(Y_{k+1})}} ~\bigg|~I_k\right]\right]\\
        &=KL(\mathbb{P}_\theta^{I_{k}},\mathbb{P}_\lambda^{I_{k}})
        + \mathbb{E}\left[\sum_{\pi \in \Pi}\chi(\psi(I_k) = \pi)KL(\PP_\theta^{\pi},\PP_\lambda^\pi)\right].
    \end{aligned}
\end{equation}
Iterating the argument we arrive at $\sum_{\pi\in\Pi}\mathbb{E}_\theta[N_{\psi,\pi}(K)]
KL(\PP_\theta^\pi,\PP_\lambda^\pi) = KL(\PP_{\theta}^{I_K},\PP_{\lambda}^{I_K})$ where $\EE_\theta$ denotes expectation in problem instance $\theta$.
Next one shows that for any measurable $Z \in [0,1]$, with respect to the natural sigma-algebra induced by $I_K$, it holds that $KL(\PP_{\theta}^{I_K},\PP_{\lambda}^{I_K}) \geq kl(\mathbb{E}_\theta[Z],\mathbb{E}_\lambda[Z])$ where $kl(p,q) = p\log{p/q} + (1-p)\log{(1-p)/(1-q)}$ denotes the KL-divergence between two Bernoulli random variables $p$ and $q$. This follows directly from Lemma~1 by \citet{garivier2019explore}.
Finally we choose $Z = N_{\psi,\Pi^*_\lambda}(K) / K$ as the fraction of episodes where an optimal policy for $\lambda$ was played (here we use the short-hand notation $N_{\psi,\Pi^*_\lambda}(K) = \sum_{\pi \in \Pi^*_\lambda} N_{\psi,\pi}(K)$). Evaluating the $kl$-term we have
\begin{align*}
    &kl\left(\frac{\mathbb{E}_\theta[  N_{\psi,\Pi^*_\lambda}(K)]}{K},\frac{\mathbb{E}_\lambda[ N_{\psi,\Pi^*_\lambda}(K)]}{K}\right) \geq \left(1 - \frac{\mathbb{E}_\theta[ N_{\psi,\Pi^*_\lambda}(K)]}{K}\right)\log{\frac{K}{K - \mathbb{E}_\lambda[N_{\psi,\Pi^*_\lambda}(K)]}}
    -\log{2}.
\end{align*}
Since $\psi$ is a uniformly good algorithm it follows that for any $\alpha>0$, $K - \mathbb{E}_\lambda[N_{\psi,\Pi^*_\lambda}(K)] = o(K^\alpha)$. By assuming that $\Pi^*_\theta \cap \Pi^*_\lambda = \varnothing$, we get $\mathbb{E}_\theta[N_{\psi,\Pi^*_\lambda}(K)] = o(K)$. This implies that for $K$ sufficiently large and all $1\geq\alpha>0$
\begin{align*}
    kl\left(\frac{\mathbb{E}_\theta[N_{\psi,\Pi^*_\lambda}(K)]}{K},\frac{\mathbb{E}_\lambda[N_{\psi,\Pi^*_\lambda}(K)]}{K}\right) \geq \log{K} - \log{K^\alpha} = (1-\alpha)\log{K} \xrightarrow{\alpha \rightarrow 0} \log{K}.
\end{align*}
\end{proof}

The set $\Lambda'(\theta)$ is uncountably infinite for any reasonable $\Theta$ we consider. What is worse the constraints of LP~\ref{eq:opt_prob} will not form a closed set and thus the value of the optimization problem will actually be obtained on the boundary of the constraints. To deal with this issue it is possible to show the following.
\generallb*
\begin{proof}
For the rest of this proof we identify $\Lambda'(\theta) = \{\lambda \in \Theta : \Pi^*_\lambda \cap \Pi^*_\theta = \emptyset, KL(\PP^{\pi^*_\theta}_\theta,\PP^{\pi^*_\theta}_\lambda) = 0,\forall \pi^*_\theta \in \Pi^*_\theta\}$ as the set from Theorem~\ref{thm:lower_bound_gen1} and $\tilde \Lambda(\theta) = \{\lambda \in \Theta : \return{\pi^*_\lambda}_\lambda \geq \return{\pi^*_\theta}_\theta, \pi^*_\lambda \not\in \Pi^*_\theta, KL(\PP^{\pi^*_\theta}_\theta,\PP^{\pi^*_\theta}_\lambda) = 0\}$. From the proof of Theorem~\ref{thm:lower_bound_gen1} it is clear that we can rewrite $\Lambda'(\theta)$ as the union $\bigcup_{\pi\in\Pi} \Lambda_{\pi}(\theta)$, where $\Lambda_{\pi}(\theta) = \{\lambda \in \Theta: KL(\PP_\theta^{\pi^*_\theta},\PP_\lambda^{\pi^*_\theta}) = 0, \return{\pi^*_\lambda} > \return{\pi^*_\theta}_\theta, \pi^*_\lambda = \pi\}$ is the set of all environments which make $\pi$ the optimal policy. This implies that we can equivalently write LP~\ref{eq:opt_prob} as 
\begin{equation}
\label{eq:opt_prob_equiv}
\begin{aligned}
    \minimize{\eta(\pi)\geq 0}{\sum_{\pi \in \Pi} \eta(\pi)\left(\return{*}_{\theta} - \return{\pi}_{\theta}\right)}
    {
    \inf_{\lambda \in \Lambda_{\pi'}(\theta)}\sum_{\pi \in \Pi} \eta(\pi) KL(\PP_\theta^\pi,\PP_\lambda^\pi) \geq 1 \qquad \textrm{for all } \,\,\pi' \in \Pi
    }.
\end{aligned}
\end{equation}
The above formulation now minimizes a linear function over a finite intersection of sets, however, these sets are still slightly inconvenient to work with. We are now going to try to make these sets more amenable to the proof techniques we would like to use for deriving specific lower bounds. We begin by noting that $\Lambda_{\pi}(\theta)$ is bounded in the following sense. We identify each $\lambda$ with a vector in $[0,1]^{S^2A}\times[0,1]^{SA}$ where the first $S^2A$ coordinates are transition probabilities and the last $SA$ coordinates are the expected rewards. From now on we work with the natural topology on $[0,1]^{S^2A}\times[0,1]^{SA}$, induced by the $\ell_1$ norm. Further, we claim that we can assume that $KL(\PP_\theta^\pi,\PP_\lambda^\pi)$ is a continuous function over $\Lambda_{\pi'}(\theta)$. The only points of discontinuity are at $\lambda$ for which the support of the transition kernel induced by $\lambda$ does not match the support of the transition kernel induced by $\theta$. At such points the $KL(\PP_\theta^\pi,\PP_\lambda^\pi) = \infty$. This implies that such $\lambda$ does not achieve the infimum in the set of constraints so we can just restrict $\Lambda_{\pi'}(\theta)$ to contain only $\lambda$ for which $KL(\PP_\theta^\pi,\PP_\lambda^\pi) < \infty$. With this restriction in hand the KL-divergence is continuous in $\lambda$.

Fix a $\pi'$ and consider the set $\{ \eta:\inf_{\lambda \in \Lambda_{\pi'}(\theta)}\sum_{\pi \in \Pi} \eta(\pi) KL(\PP_\theta^\pi,\PP_\lambda^\pi) \geq 1\}$ corresponding to one of the constraints in LP~\ref{eq:opt_prob_equiv}. Denote $\tilde \Lambda_{\pi'}(\theta) = \{\lambda \in \Theta: KL(\PP_\theta^{\pi^*_\theta},\PP_\lambda^{\pi^*_\theta}) = 0, \return{\pi^*_\lambda}_\lambda \geq \return{\pi^*_\theta}_\theta, \pi^*_\lambda \not\in \Pi^*_\theta, \pi^*_\lambda = \pi'\}$. $\tilde \Lambda_{\pi'}(\theta)$ is closed as $KL(\PP_\theta^{\pi^*_\theta},\PP_\lambda^{\pi^*_\theta})$ and $\return{\pi^*_\lambda}_\lambda - \return{\pi^*_\theta}_\theta$ are both continuous in $\lambda$. To see the statement for $\return{\pi^*_\lambda}_\lambda$, notice that this is the maximum over the continuous functions $\return{\pi}_\lambda$ over $\pi\in\Pi$. Take any $\eta \in \Lambda_{\pi'}(\theta)$ and let $\{\lambda_{j}\}_{j=1}^\infty, \lambda_j \in \Lambda_{\pi'}(\theta)$ be a sequence of environments such that $\sum_{\pi\in \Pi}\eta(\pi)KL(\PP_\theta^{\pi}, \PP_{\lambda_j}^{\pi}) \geq 1+ 2^{-j}$. If there is no convergent subsequence of $\{\lambda_{j}\}_{j=1}^\infty$ in $\Lambda_{\pi'}(\theta)$ we claim it is because of the constraint $\return{\pi^*_\lambda}_\lambda > \return{\pi^*_\theta}_\theta$. Take the limit $\lambda$ of any convergent subsequence of $\{\lambda_{j}\}_{j=1}^\infty$ in the closure of $\Lambda_{\pi'}(\theta)$. Then by continuity of the divergence we have $0=\lim_{j\rightarrow\infty}KL(\PP_\theta^{\pi^*_\theta},\PP_{\lambda_j}^{\pi^*_\theta}) = KL(\PP_\theta^{\pi^*_\theta},\PP_\lambda^{\pi^*_\theta})$, thus it must be the case that $\return{\pi^*_\lambda}_\lambda \leq \return{\pi^*_\theta}_\theta$. This shows that $\tilde \Lambda_{\pi'}(\theta)$ is a subset of the closure of $\Lambda_{\pi'}(\theta)$ which implies it is the closure of $\Lambda_{\pi'}(\theta)$, i.e., $\bar \Lambda_{\pi'}(\theta) = \tilde \Lambda_{\pi'}(\theta)$.

Next, take $\eta \in \{ \eta:\min_{\lambda \in \bar\Lambda_{\pi'}(\theta)}\sum_{\pi \in \Pi} \eta(\pi) KL(\PP_\theta^\pi,\PP_\lambda^\pi) \geq 1\}$ and let $\lambda_{\pi',\eta}$ be the environment on which the minimum is achieved. Such $\lambda_{\pi',\eta}$ exists because we just showed that $\bar \Lambda_{\pi'}(\theta)$ is closed and bounded and hence compact and the sum consists of a finite number of continuous functions. If $\lambda_{\pi',\eta} \in \Lambda_{\pi'}(\theta)$ then $\eta \in \{ \eta:\inf_{\lambda \in \Lambda_{\pi'}(\theta)}\sum_{\pi \in \Pi} \eta(\pi) KL(\PP_\theta^\pi,\PP_\lambda^\pi) \geq 1\}$. If $\lambda_{\pi',\eta} \not\in \Lambda_{\pi'}(\theta)$ then $\lambda_{\pi',\eta}$ must be a limit point of $\Lambda_{\pi'}(\theta)$. By definition we can construct a convergent sequence of $\{\lambda_j\}_{j=1}^\infty,\lambda_j \in \Lambda_{\pi'}(\theta)$ to $\lambda_{\pi',\eta}$ such that $\sum_{\pi\in\Pi} \eta(\pi) KL(\PP_\theta^{\pi}, \PP_{\lambda_{j}}^\pi) \geq 1$.  This implies
$\sum_{\pi\in\Pi} \eta(\pi) KL(\PP_\theta^{\pi}, \PP_{\lambda_{j}}^\pi) \geq \inf_{\lambda \in \Lambda_{\pi'}(\theta)}\sum_{\pi \in \Pi} \eta(\pi) KL(\PP_\theta^\pi,\PP_\lambda^\pi)$. Using the continuity of the KL term and taking limits, the above implies that the minimum upper bounds the infimum.
Since we argued that $\Lambda_{\pi'}(\theta)$ is bounded and $\sum_{\pi \in \Pi} \eta(\pi) KL(\PP_\theta^{\pi},\PP_{\lambda_j}^{\pi})$ is also bounded from below this implies $\bar \Lambda_{\pi'}(\theta)$ contains the infimum $\inf_{\lambda \in \Lambda_{\pi'}(\theta)}\sum_{\pi \in \Pi} \eta(\pi) KL(\PP_\theta^\pi,\PP_\lambda^\pi)$. This implies $\inf_{\lambda \in \Lambda_{\pi'}(\theta)}\sum_{\pi \in \Pi} \eta(\pi) KL(\PP_\theta^\pi,\PP_\lambda^\pi) \geq \min_{\lambda \in \bar \Lambda_{\pi'}(\theta)}\sum_{\pi \in \Pi} \eta(\pi) KL(\PP_\theta^\pi,\PP_\lambda^\pi)$
, and so the infimum over $\Lambda_{\pi}(\theta)$ equals the minimum over $\bar\Lambda_{\pi}(\theta)$. Which finally implies that $\eta \in \{ \eta:\inf_{\lambda \in \Lambda_{\pi'}(\theta)}\sum_{\pi \in \Pi} \eta(\pi) KL(\PP_\theta^\pi,\PP_\lambda^\pi) \geq 1\}$. This shows that LP~\ref{eq:opt_prob_equiv} is equivalent to 
\begin{align*}
    \minimize{\eta(\pi)\geq 0}{\sum_{\pi \in \Pi} \eta(\pi)\left(\return{*}_{\theta} - \return{\pi}_{\theta}\right)}
    {
    \min_{\lambda \in \bar\Lambda_{\pi'}(\theta)}\sum_{\pi \in \Pi} \eta(\pi) KL(\PP_\theta^\pi,\PP_\lambda^\pi) \geq 1 \qquad \textrm{for all } \,\,\pi' \in \Pi
    },
\end{align*}
or equivalently that we can consider the closure of $\Lambda(\theta)$ in LP~\ref{eq:opt_prob}, $\bar\Lambda(\theta) = \{ \lambda \in \Theta \colon \return{\pi^*_\lambda}_\lambda \geq \return{\pi^*_\theta}_\theta,\pi^*_\lambda \not\in \Pi^*_\theta, KL(\PP_\theta^{\pi^*_\theta}, \PP_\lambda^{\pi^*_\theta}) = 0\}$ i.e. the set of environments which makes any $\pi$ optimal without changing the environment on state-action pairs in $\pi^*_\theta$.
\end{proof}


\subsection{Lower bounds for full support optimal policy}
\label{app:lower_bounds_full_supp}

\gaplemmafullsupp*

\begin{proof}Let 
 $\lambda$ be the environment that is identical to $\theta$ except for the immediate reward for state-action pair for $(s,a)$. Specifically, let $R_{\lambda}(s,a)$ so that $r_{\lambda}(s,a) = r_\theta(s,a) + \Delta$ with $\Delta = \gap_\theta(s,a)$ . Since we assume that rewards are Gaussian, it follows that
 \begin{align*}
     KL(\PP_\theta^{\pi}, \PP_\lambda^{\pi}) 
     &= w^{\pi}_\lambda(s, a) KL(R_\theta(s,a), R_\lambda(s,a))
     \leq KL(R_\theta(s,a), R_\lambda(s,a))\\
     &\leq \gap_\theta(s,a)^2
 \end{align*}
 for any policy $\pi \in \Pi$.
We now show that the optimal value function (and thus return) of $\lambda$ is uniformly upper-bounded by the optimal value function of $\theta$. To that end, consider their difference in any state $s'$, which we will upper-bound by their difference in $s$ as
\begin{align*}
    V_\lambda^*(s') - V_\theta^*(s')
    &\leq
    \chi(\kappa(s) > \kappa(s')) \PP_\theta^{\pi^*_\lambda}(s_{\kappa(s)} = s | s_{\kappa(s')} = s')[V_\lambda^*(s) - V_\theta^*(s)] \\
    &\leq V_\lambda^*(s) - V_\theta^*(s).
\end{align*}
Further, the difference in $s$ is exactly
\begin{align*}
    V_\lambda^*(s) - V_\theta^*(s)
    &= r_{\lambda}(s,a) + \langle P_\theta(\cdot | s,a), V_\theta^*\rangle - V_\theta^*(s)\\
    &= r_{\theta}(s,a) + \langle P_\theta(\cdot | s,a), V_\theta^*\rangle +\gap_\theta(s,a)  - V_\theta^*(s)
    = 0.
\end{align*}
Hence, $V^*_\lambda = V^*_\theta \leq 1$ and thus $\lambda \in \Theta$.
 We will now show that there is a policy that is optimal in $\lambda$ but not in $\theta$.
 Let $\pi^* \in \Pi^*_\theta$ be any optimal policy for $\theta$ that has non-zero probability of visiting $s$ and consider the policy 
 \begin{align*}
     \tilde \pi(\tilde s) = 
     \begin{cases}
        \pi^*(\tilde s) & \textrm{if } s \neq \tilde s
        \\
        a & \textrm{if } s = \tilde s
     \end{cases}
 \end{align*}
 that matches $\pi^*$ on all states except $s$. We will now show that $\tilde \pi$ achieves the same return as $\pi^*$ in $\lambda$. Consider their difference
 \begin{align*}
     \return{\tilde \pi}_\lambda - 
     \return{\pi^*}_\lambda
     \overset{(i)}{=} &~
     w_\lambda^{\tilde \pi}(s, \tilde \pi(s)) [r_\lambda(s, \tilde \pi(s)) + \langle P_\lambda(\cdot | s, \tilde \pi(s)), V_\lambda^{\tilde \pi} \rangle ]\\
     &~- 
          w_\lambda^{\pi^*}(s, \pi^*(s)) [r_\lambda(s, \pi^*(s)) + \langle P_\lambda(\cdot | s,  \pi^*(s)), V_\lambda^{\pi^*} \rangle ]\\
\overset{(ii)}{=} &~
     w_\lambda^{\pi^*}(s, \pi^*(s)) [r_\lambda(s, \tilde \pi(s)) - r_\lambda(s, \pi^*(s)) + \langle P_\lambda(\cdot | s, \tilde \pi(s)) - P_\lambda(\cdot | s,  \pi^*(s)), V_\lambda^{\pi^*} \rangle ]\\
\overset{(iii)}{=} &~
     w_\theta^{\pi^*}(s, \pi^*(s)) [\Delta + r_\theta(s, \tilde \pi(s)) - r_\theta(s, \pi^*(s)) + \langle P_\theta(\cdot | s, \tilde \pi(s)) - P_\theta(\cdot | s,  \pi^*(s)), V_\theta^* \rangle ]\\
\overset{(iv)}{=} &~
          w_\theta^{\pi^*}(s, \pi^*(s)) [\Delta - \gap_\theta(s, \tilde \pi(s)) ] 
 \end{align*}
 where $(i)$ and $(ii)$ follow from the fact that $\tilde \pi$ and $\pi^*$ only differ on $s$ and hence, their probability at arriving at $s$ and their value for any successor state of $s$ is identical.
 Step $(iii)$ follows from the fact that $\lambda$ and $\theta$ only differ on $(s,a)$ which is not visited by $\pi^*$. Finally, step $(iv)$ applies the definition of optimal value functions and value-function gaps.
 Since $\Delta = \gap_\theta(s, \tilde \pi(s))$, it follows that $\return{\tilde \pi}_\lambda = \return{\pi^*}_\lambda = \return{\pi^*}_\theta = \return{*}_\theta$. As we have seen above, the optimal value function (and return) is identical in $\theta$ and $\lambda$ and, hence, $\tilde \pi$ is optimal in $\lambda$. 
 
 Note that the we can apply the chain of equalities above in the same manner to $\return{\tilde \pi}_\theta - 
     \return{\pi^*}_\theta$ if we consider $\Delta = 0$. This yields
     \begin{align*}
         \return{\tilde \pi}_\theta - 
     \return{\pi^*}_\theta = - w_\theta^{\pi^*}(s, \pi^*(s)) \gap_\theta(s, a) < 0 
     \end{align*}
     because $w_\theta^{\pi^*}(s, \pi^*(s)) > 0$ and $\gap_\theta(s, a) < 0$ by assumption. Hence $\tilde \pi$ is not optimal in $\theta$, which completes the proof.
\end{proof}

\begin{lemma}[Optimization problem over $\Scal \times \Acal$ instead of $\Pi$]
\label{lem:LP_(s,a)_relaxed}
Let optimal value $C(\theta)$ of the optimization problem \pref{eq:opt_prob} in \pref{thm:lower_bound_gen} is lower-bound by the optimal value of the problem 
\begin{equation}
\label{eq:LP_(s,a)_relaxed}
    \begin{aligned}
        \underset{\eta(s,a) \geq 0}{\operatorname{minimize}} \quad
        & 
        \sum_{s,a} \eta(s,a)\gap_{\theta}(s,a)
        \\
        \textrm{s.t.} \quad &
        \sum_{s,a} 
        \eta(s,a)
        KL(R_\theta(s,a), R_\lambda(s,a)) \\
        & + \sum_{s,a} 
        \eta(s,a) KL(P_\theta(\cdot|s,a), P_\lambda(\cdot|s,a)) 
        \geq 1\qquad \textrm{for all } \,\,\lambda \in \Lambda(\theta)
    \end{aligned}
\end{equation}
\end{lemma}
\begin{proof}
First, we rewrite the objective of \pref{eq:opt_prob} as
\begin{align*}
    \sum_{\pi \in \Pi} \eta(\pi) (\return{*}_\theta - \return{\pi}_\theta)
    \overset{(i)}{=} \sum_{\pi \in \Pi} \eta(\pi) \sum_{s,a} w^\pi_\theta(s,a) \gap_\theta(s,a)
    = \sum_{s,a} \left(\sum_{\pi \in \Pi} \eta(\pi) w^\pi_\theta(s,a)\right) \gap_\theta(s,a)
\end{align*} where step $(i)$ applies \pref{lem:gap_decomp_pi} proved in Appendix~\ref{app:upper_bounds}. Here, $w^\pi_\theta(s,a)$ is the probability of reaching $s$ and taking $a$ when playing policy $\pi$ in MDP $\theta$.
Similarly, the LHS of the constraints of \pref{eq:opt_prob} can be decomposed as
\begin{align*}
    &\sum_{\pi \in \Pi} \eta(\pi) KL(\PP_\theta^\pi, \PP_\lambda^\pi)\\
    &= \sum_{\pi \in \Pi} \eta(\pi) \sum_{s,a} w_\theta^\pi(s,a) 
    \left( KL(R_\theta(s,a), R_\lambda(s,a)) + KL(P_\theta(\cdot|s,a), P_\lambda(\cdot|s,a))
    \right) \\
        &= \sum_{s,a} \left[\sum_{\pi \in \Pi} \eta(\pi)  w_\theta^\pi(s,a) \right]
    \left( KL(R_\theta(s,a), R_\lambda(s,a)) + KL(P_\theta(\cdot|s,a), P_\lambda(\cdot|s,a))
    \right)
\end{align*}
where the first equality follows from writing out the definition of the KL divergence.
Let now $\eta(\pi)$ be a feasible solution to the original problem \pref{eq:opt_prob}. Then the two equalities we just proved show that $\eta(s,a) = \sum_{\pi \in \Pi} \eta(\pi)  w_\theta^\pi(s,a)$ is a feasible solution for the problem in \pref{eq:LP_(s,a)_relaxed} with the same value. Hence, since \pref{eq:LP_(s,a)_relaxed} is a minimization problem, its optimal value cannot be larger than $C(\theta)$, the optimal value of \pref{eq:opt_prob}.
\end{proof}

\fullsupportlb*
\begin{proof}
Let $\bar \Lambda(\theta)$ be a set of all confusing MDPs from \pref{lem:non-empty_change_env}, that is, for every suboptimal $(s,a)$, $\bar \Lambda(\theta)$ contains exactly one confusing MDP that differs with $\theta$ only in the immediate reward at $(s,a)$. 
Consider now the relaxation of \pref{thm:lower_bound_gen} from \pref{lem:LP_(s,a)_relaxed} and further relax it by reducing the set of constraints induced by $\Lambda(\theta)$ to only the set of constraints induced by $\bar \Lambda(\theta)$:
\begin{equation*}
    \begin{aligned}
        \underset{\eta(s,a) \geq 0}{\operatorname{minimize}} \quad
        & 
        \sum_{s,a} \eta(s,a)\gap_{\theta}(s,a)
        \\
        \textrm{s.t.} \quad &
        \sum_{s,a} 
        \eta(s,a)
        KL(R_\theta(s,a), R_\lambda(s,a)) 
        \geq 1\qquad \textrm{for all } \,\,\lambda \in \bar \Lambda(\theta)
    \end{aligned}
\end{equation*}
Since all confusing MDPs only differ in rewards, we dropped the KL-term for the transition probabilities. We can simplify the constraints by noting that for each $\lambda$, only one KL-term is non-zero and it has value $\gap_{\theta}(s,a)^2$. Hence, we can write the problem above equivalently as
\begin{equation*}
    \begin{aligned}
        \underset{\eta(s,a) \geq 0}{\operatorname{minimize}} \quad
        & 
        \sum_{s,a} \eta(s,a)\gap_{\theta}(s,a)
        \\
        \textrm{s.t.} \quad &
        \eta(s,a)\gap_{\theta}(s,a)^2 \geq 1 \qquad \textrm{for all } \,\,(s,a) \in \Scal \times \Acal \textrm{ with }\,\, \gap_{\theta}(s,a) > 0 
    \end{aligned}
\end{equation*}
Rearranging the constraint as $\eta(s,a) \geq 1 / \gap_{\theta}(s,a)^2$, we see that the value is lower-bounded by \begin{align*}
   \sum_{s,a} \eta(s,a)\gap_{\theta}(s,a) \geq
   \sum_{s,a \colon \gap_\theta(s,a) > 0} \eta(s,a)\gap_{\theta}(s,a) \geq
   \sum_{s,a \colon \gap_\theta(s,a) > 0} \frac{1}{\gap_\theta(s,a)},
\end{align*}
which completes the proof.
\end{proof}

We note that because the relaxation in \pref{lem:LP_(s,a)_relaxed} essentially allows the algorithm to choose which state-action pairs to play instead of just policies, the final lower bound in \pref{thm:lower_bound_all_states_supp} may be loose, especially in factors of $H$. However, it is unlikely that the $\gap_{\min}$ term arising in the upper bound of \citet{simchowitz2019non} can be recovered. We conjecture that such a term can be avoided by algorithms, which do not construct optimistic estimators for the $Q$-function at each state-action pair but rather just work with a class of policies and construct only optimistic estimators of the return.

\subsection{Lower bounds for deterministic MDPs}
\label{app:lower_bounds_det}
We will show that we can derive lower bounds in two cases:
\begin{enumerate}
    \item We show that if the graph induced by the MDP is a tree, then we can formulate a finite LP which has value at most a polynomial factor of $H$ away from the value of LP~\ref{eq:opt_prob}.
    \item We show that if we assume that the value function for any policy is at most $1$ and the rewards of each state-action pair are at most $1$, then we can derive a closed form lower bound. This lower bound is also at most a polynomial factor of $H$ away from the solution to LP~\ref{eq:opt_prob}.
\end{enumerate} 

We begin by stating a helpful lemma, which upper and lower bounds the $KL$-divergence between two environments on any policy $\pi$. Since we consider Gaussian rewards with $\sigma =1/\sqrt{2}$ it holds that $KL(R_\theta(s,a),R_{\lambda}(s,a)) = (r_\theta(s,a) - r_\lambda(s,a))^2$. Further for any $\pi$ and $\lambda$ it holds that $KL(\theta(\pi),\lambda(\pi)) = \sum_{(s,a)\in\pi} KL(R_\theta(s,a),R_\lambda(s,a)) = \sum_{(s,a)\in\pi}(r_\theta(s,a) - r_\lambda(s,a))^2$. We can now show the following lower bound on $KL(\theta(\pi),\lambda(\pi))$.
\begin{lemma}
\label{lem:kl_lower_bound}
Fix $\pi$ and suppose $\lambda$ is such that $\pi^*_\lambda = \pi$. Then $(\return{*}-\return{\pi})^2 \geq KL(\theta(\pi),\lambda(\pi)) \geq \frac{(\return{*} - \return{\pi})^2}{H}$. 
\end{lemma}
\begin{proof}
The second inequality follows from the fact that the optimization problem
\begin{align*}
    \minimize{\theta,\lambda \in \Lambda(\theta) : \pi^*_\lambda = \pi}{\sum_{(s,a)\in\pi}(r_\theta(s,a) - r_\lambda(s,a))^2}{\sum_{(s,a)\in\pi} r_{\lambda}(s,a) - r_\theta(s,a) \geq \return{*} - \return{\pi}},
\end{align*}
admits a solution at $\theta,\lambda$ for which $r_\lambda(s,a) - r_\theta(s,a) = \frac{\return{*} - \return{\pi}}{H}, \forall (s,a) \in \pi$. The first inequality follows from considering the optimization problem
\begin{align*}
    \maximize{\theta,\lambda \in \Lambda(\theta) : \pi^*_\lambda = \pi}{\sum_{(s,a)\in\pi}(r_\theta(s,a) - r_\lambda(s,a))^2}{\sum_{(s,a)\in\pi} r_{\lambda}(s,a) - r_\theta(s,a) \geq \return{*} - \return{\pi}},
\end{align*}
and the fact that it admits a solution at $\theta,\lambda$ for which there exists a single state-action pair $(s,a) \in \pi$ such that $r_\theta(s,a)-r_\lambda(s,a) = \return{*} - \return{\pi}$ and for all other $(s,a)$ it holds that $r_\lambda(s,a) = r_\theta(s,a)$.
\end{proof}

Using the above Lemma~\ref{lem:kl_lower_bound} we now show that we can restrict our attention only to environments $\lambda \in \Lambda(\theta)$ which make one of $\pi^*_{(s,a)}$ optimal and derive an upper bound on $C(\theta)$ which we will try to match, up to factors of $H$, later. Define the set $\tilde\Lambda(\theta) = \{\lambda \in \Lambda(\theta) : \exists (s,a)\in \Scal\times\Acal, \pi^*_{\lambda} = \pi^*_{(s,a)}\}$ and $\Pi^* = \{\pi \in \Pi,\pi\neq\pi^*_\theta: \exists (s,a) \in \Scal\times\Acal, \pi = \pi^*_{(s,a)}\}$. We have
\begin{lemma}
\label{lem:primal_opt_upper_bound}
Let $\tilde C(\theta)$ be the value of the optimization problem
\begin{equation}
\label{eq:opt_primal_relaxed}
    \begin{aligned}
        \minimize{\eta(\pi)\geq 0}{\sum_{\pi\in \Pi^*} \eta(\pi)(\return{*} - \return{\pi})}{ \sum_{\pi \in \Pi^*} \eta(\pi)KL(\theta(\pi),\lambda(\pi)) \geq 1,\forall \lambda \in \tilde\Lambda(\theta)}.
    \end{aligned}
\end{equation}
Then $\sum_{\pi \in \Pi^*} \frac{H}{\return{*} - \return{\pi}} \geq C(\theta) \geq \frac{\tilde C(\theta)}{H}$.
\end{lemma}
\begin{proof}
We begin by showing $C(\theta) \geq \frac{\tilde C(\theta)}{H}$ holds. Fix a $\pi \not\in\Pi^*$ s.t. the solution of LP~\ref{eq:opt_prob} implies $\eta(\pi)>0$. Let $\lambda \in \tilde\Lambda(\theta)$ be a change of environment for which $KL(\theta(\pi),\lambda(\pi))>0$. We can now shift all of the weight of $\eta(\pi)$ to $\eta(\pi^*_\lambda)$ while still preserving the validity of the constraint. Further doing so to all $\pi^*_{(s,a)}$ for which $\pi^*_{(s,a)} \cap \pi \neq \emptyset$ will not increase the objective by more than a factor of $H$ as $\return{*} - \return{\pi} \geq \frac{1}{H}\sum_{(s,a) \in \pi} \return{*} - \return{\pi^*_{(s,a)}}$. Thus, we have converted the solution to LP~\ref{eq:opt_prob} to a feasible solution to LP~\ref{eq:opt_primal_relaxed} which is only a factor of $H$ larger.

Next we show that $\sum_{\pi \in \Pi^*}\frac{H}{\return{*} - \return{\pi}} \geq C(\theta)$. Set $\eta(\pi) = 0,\forall \pi \in \Pi\setminus \Pi^*$ and set $\eta(\pi) = \frac{H}{(\return{*} - \return{\pi})^2}, \forall \pi \in \Pi^*$.
If $\pi$ is s.t. $\eta(\pi) > 0$ then for any $\lambda$ which makes $\pi$ optimal it holds that 
\begin{align*}
    1 &\leq \frac{H}{(\return{*} - \return{\pi^*_\lambda})^2} \times \frac{(\return{*} - \return{\pi^*_\lambda})^2}{H} \leq \frac{H}{(\return{*} - \return{\pi^*_\lambda})^2} KL(\theta(\pi^*_\lambda),\lambda(\pi^*_\lambda))\\
    &= \eta(\pi^*_\lambda) KL(\theta(\pi^*_\lambda),\lambda(\pi^*_\lambda)) \leq \sum_{\pi' \in \Pi} \eta(\pi')KL(\theta(\pi'),\lambda(\pi')),
\end{align*}
where the second inequality follows from Lemma~\ref{lem:kl_lower_bound}.
Next, if $\pi$ is s.t. $\eta(\pi) = 0$ then for any $\lambda$ which makes $\pi$ optimal it holds that
\begin{align*}
    \sum_{\pi' \in \Pi} \eta(\pi')KL(\theta(\pi'),\lambda(\pi')) &\geq \sum_{(s,a) \in \pi^*_\lambda} \eta(\pi^*_{(s,a)}) KL(\theta(\pi^*_{(s,a)}),\lambda(\pi^*_{(s,a)}))\\
    &=\sum_{(s,a) \in \pi^*_\lambda} \frac{H}{(\return{*} - \return{\pi^*_{(s,a)}})^2}KL(\theta(\pi^*_{(s,a)}),\lambda(\pi^*_{(s,a)}))\\
    &\geq \frac{H}{(\return{*} - \return{\pi^*_{\lambda}})^2}\sum_{(s,a) \in \pi^*_\lambda} KL(\theta(\pi^*_{(s,a)}),\lambda(\pi^*_{(s,a)}))\\
    &\geq \frac{H}{(\return{*} - \return{\pi^*_{\lambda}})^2}\sum_{(s,a) \in \pi^*_\lambda} KL(R_\theta(s,a),R_\lambda(s,a))\\
    &= \frac{H}{(\return{*} - \return{\pi^*_{\lambda}})^2} KL(\theta(\pi^*_{\lambda}),\lambda(\pi^*_\lambda))\geq 1,
\end{align*}
where the second inequality follows from the fact that $\return{\pi^*_\lambda} \leq \return{\pi^*_{(s,a)}},\forall (s,a) \in \pi^*_{\lambda}$.
\end{proof}

\subsubsection{Lower bound for Markov decision processes with bounded value function}

\begin{lemma}
\label{lem:confusing_mdps_det}
Let $\Theta$ be the set of all episodic MDPs with Gaussian immediate rewards and optimal value function uniformly bounded by $1$. Consider an MDP $\theta \in \Theta$ with deterministic transitions. Then, for any reachable state-action pair $(s,a)$ that is not visited by any optimal policy, there exists a confusing MDP $\lambda \in \Lambda(\theta)$ with
\begin{itemize}
    \item $\lambda$ and $\theta$ only differ in the immediate reward at $(s,a)$
    \item $KL(\PP_\theta^\pi, \PP_\lambda^\pi) = w_\theta^\pi(s,a) (\return{*}_\theta - \return{\pi_{(s,a)}^*}_\theta)^2$ for all $\pi \in \Pi$ where $\return{\pi_{(s,a)}^*}_\theta = \max_{\pi \colon w^\pi(s,a) > 0} \return{\pi}_\theta$.
\end{itemize}
\end{lemma}
\begin{proof}
Let $(s,a) \in \Scal \times \Acal$ be any state-action pair that is not visited by any optimal policy. Then $\return{\pi_{(s,a)}^*}_\theta = \max_{\pi \colon w^\pi(s,a) > 0} \return{\pi}_\theta \leq \return{*}_\theta$ is strictly suboptimal in $\theta$. Let $\tilde \pi$ be any policy that visits $(s,a)$ and achieves the highest return $\return{\pi_{(s,a)}^*}_\theta$ in $\theta$ possible among such policies. 

Define $\lambda$ to be the MDP that matches $\theta$ except in the immediate reward at $(s,a)$, which we set as $R_\lambda(s,a) = \Ncal(r_\theta(s,a) + \Delta, 1/2)$ with $\Delta = \return{*}_\theta - \return{\pi_{(s,a)}^*}_\theta$. That is, the expected reward of $\lambda$ in $(s,a)$ is raised by $\Delta$.
For any policy $\pi$, it then holds
\begin{align*}
    KL(\PP_\theta^\pi, \PP_\lambda^\pi) &= w_\theta^\pi(s,a) KL(R_\theta(s,a), R_\lambda(s,a))\\
    \return{\pi}_\lambda &=  w_\theta^\pi(s,a) \Delta + \return{\pi}_\theta
\end{align*}
due to the deterministic transitions. Hence, while $\return{*}_\lambda = \return{*}_\theta$ and all optimal policies of $\theta$ are still optimal in $\lambda$, now policy $\tilde \pi$, which is not optimal in $\theta$ is optimal in $\lambda$.

By the choice of Gaussian rewards with variance $1/2$, we have $KL(R_\theta(s,a), R_\lambda(s,a)) = (\return{*}_\theta - \return{\pi_{(s,a)}^*}_\theta)^2$ and thus $KL(\PP_\theta^\pi, \PP_\lambda^\pi) = w_\theta^\pi(s,a) (\return{*}_\theta - \return{\pi_{(s,a)}^*}_\theta)^2$ for all $\pi \in \Pi$.

It only remains to show that $\lambda \in \Theta$, i.e., that all immediate rewards and optimal value function is bounded by $1$. For rewards, we have
\begin{align*}
    r_\lambda(s,a) = r_\theta(s,a) + \Delta = r_\theta(s,a) + \return{*}_\theta - \return{\pi_{(s,a)}^*}_\theta
    = \return{*}_\theta - \underset{\geq 0}{\underbrace{(\return{\pi_{(s,a)}^*}_\theta - r_\theta(s,a))}} \leq \return{*}_\theta \leq 1
\end{align*}
for $(s,a)$ and for all other $(s', a')$, $r_\lambda(s', a') = r_\theta(s', a') \leq 1$. Finally, the value function at any reachable state is bounded by the optimal return $\return{*}_\lambda = \return{*}_\theta \leq 1$ and for any unreachable state, the optimal value function of $\lambda$ is identical to the optimal value function of $\theta$. Hence, $\lambda \in \Theta$.
\end{proof}

\lowerbounddeterministic*
\begin{proof}
The proof 
works by first relaxing the general LP~\ref{eq:opt_prob} and then considering its dual. We now define the set $\check\Lambda(\theta)$ which consists of all changes of environment which make $\pi^*_{(s,a)}$ optimal by only changing the distribution of the reward at $(s,a)$ by making it $\return{*}_\theta - \return{\pi^*_{(s,a)}}_\theta$ larger.
Formally, the set is defined as 
\begin{align*}
    \check\Lambda(\theta) = \big\{\lambda_{(s,a)}\colon \lambda \in \Lambda(\theta), KL(R_\theta(s,a), R_{\lambda}(s,a)) = (\return{*}_\lambda - \return{\pi^*_{(s,a)}})^2,\\ KL(R_\theta(s',a'), R_{\lambda}(s',a')) = 0,KL(P_\theta(s',a'), P_{\lambda}(s',a')) = 0,\forall (s',a')\neq (s,a)\big\}.
\end{align*}
This set is guaranteed to be non-empty (for any reasonable MDP) by Lemma~\ref{lem:confusing_mdps_det}.
The relaxed LP is now give by
\begin{equation}
    \begin{aligned}
        \minimize
        {\eta(\pi)\geq 0}
        {
        \sum_{\pi \in \Pi} \eta(\pi)(\return{*}_\theta - \return{\pi}_\lambda)
        }
        {
        \sum_{\pi \in \Pi} \eta(\pi)KL(\PP^\pi_\theta,\PP^\pi_\lambda) \geq 1
        \qquad \textrm{for all }  \lambda \in \check\Lambda(\theta)
        }.
    \end{aligned}
\end{equation}
The dual of the above LP is given by
\begin{equation}
\label{eq:opt_dual_gen}
    \begin{aligned}
        \maximize
        {\mu(\lambda)\geq 0}
        {
        \sum_{\lambda \in \check\Lambda(\theta)} \mu(\lambda)
        }
        {
        \sum_{\lambda \in \check \Lambda(\theta)} \mu(\lambda)KL(\PP^\pi_\theta,\PP^\pi_\lambda) \leq \return{*}_\theta - \return{\pi}_\theta
        \qquad \textrm{for all }  \pi \in \Pi
        }.
    \end{aligned}
\end{equation}
By weak duality, the value of any feasible solution to \pref{eq:opt_dual_gen} produces a lower bound on $C(\theta)$ in \pref{thm:lower_bound_gen}.
Let 
\begin{align*}
    \Xcal = \{ (s,a) \in \Scal \times \Acal \colon  w^\pi_\theta(s,a) = 0 \textrm{ for all }\pi \in \Pi^*_\theta \textrm{ and }  w^\pi_\theta(s,a) > 0 \textrm{ for some } \pi \in \Pi \setminus \Pi^*_\theta \}
\end{align*}
be the set of state-action pairs that are reachable in $\theta$ but no optimal policy visits.
Then consider a dual solution $\mu$ that puts $0$ on all confusing MDPs except on the $|\Xcal|$ many MDPs from \pref{lem:confusing_mdps_det}. Since each such confusing MDP is associated with an $(s,a) \in \Xcal$, we can rewrite $\mu$ as a mapping from $\Xcal$ to $\RR$ sending $(s,a) \rightarrow \lambda_{(s,a)}$. Specifically, we set
\begin{align*}
    \mu(s,a) &= \frac{1}{H} \left(\return{*}_\theta - \return{\pi_{(s,a)}^*}_\theta\right)^{-1} & \textrm{for all } & (s,a) \in \Xcal.
\end{align*}
To show that this $\mu$ is feasible, consider the LHS of the constraints in \pref{eq:opt_dual_gen}
\begin{align*}
    \sum_{\lambda \in \check \Lambda(\theta)} \mu(\lambda)KL(\PP^\pi_\theta,\PP^\pi_\lambda)
    &= \sum_{(s,a) \in \Xcal} \frac{1}{H} \left(\return{*}_\theta - \return{\pi_{(s,a)}^*}_\theta\right)^{-1} KL(\PP^\pi_\theta,\PP^\pi_{(s,a)})\\
      &= \sum_{(s,a) \in \Xcal} \frac{1}{H} \left(\return{*}_\theta - \return{\pi_{(s,a)}^*}_\theta\right)^{-1} w_\theta^\pi(s,a) (\return{*}_\theta - \return{\pi_{(s,a)}^*}_\theta)^2\\
    &= \sum_{(s,a) \in \Xcal} \frac{1}{H} w_\theta^\pi(s,a) (\return{*}_\theta - \return{\pi_{(s,a)}^*}_\theta)
\end{align*}
where the first equality applies our definition of $\mu$ and the second uses the expression for the KL-divergence from \pref{lem:confusing_mdps_det}. By definition of $\return{\pi_{(s,a)}^*}_\theta$, we have $\return{\pi_{(s,a)}^*}_\theta \geq \return{\pi}_\theta$ for all policies $\pi$ with $w_\theta^\pi(s,a) > 0$. Thus, 
\begin{align*}
\sum_{(s,a) \in \Xcal} \frac{1}{H} w_\theta^\pi(s,a) (\return{*}_\theta - \return{\pi_{(s,a)}^*}_\theta)
&\leq 
\sum_{(s,a) \in \Xcal} \frac{1}{H} w_\theta^\pi(s,a) (\return{*}_\theta - \return{\pi}_\theta)
\\
& \leq \return{*}_\theta - \return{\pi}_\theta
\end{align*}
where the second inequality holds because each policy visits at most $H$ states. Thus proves that $\mu$ defined above is indeed feasible. Hence, its objective value
\begin{align*}
    \sum_{\lambda \in \Lambda(\theta)} \mu(\lambda)
    = \sum_{(s,a) \in \Xcal} \frac{1}{H} \left(\return{*}_\theta - \return{\pi_{(s,a)}^*}_\theta\right)
\end{align*}
is a lower-bound for $C(\theta)$ from \pref{thm:lower_bound_gen} which finishes the proof.
\end{proof}

\subsubsection{Tree-structured MDPs}
Even though Lemma~\ref{lem:primal_opt_upper_bound} restricts the set of confusing environments from $\Lambda(\theta)$ to $\tilde\Lambda(\theta)$, this set could still have exponential or even infinite cardinality. In this section we show that for a type of special MDPs we can restrict ourselves to a finite subset of $\tilde\Lambda(\theta)$ of size at most $SA$. 

Arrange $\pi^*_{(s,a)},(s,a)\in\Scal\times\Acal$ according to the value functions $\return{\pi^*_{(s,a)}}$. Under this arrangement let $\pi_1 \succeq \pi_2 \succeq,\ldots,\succeq \pi_m$. Let $\pi_0 = \pi_\theta^*$. We will now construct $m$ environments $\lambda_1,\ldots,\lambda_m$, which will constitute the finite subset. We begin by constructing $\lambda_1$ as follows. Let $\mathcal{B}_1$ be the set of all $(s_h,a_h) \in \pi_1$ and $(s_h,a_h)\not\in \pi_0$. Arrange the elements in $\mathcal{B}_1$ in inverse dependence on horizon $(s_{h_1},a_{h_1}) \preceq (s_{h_2},a_{h_2}) \preceq \ldots \preceq (s_{h_{H_1}},a_{h_{H_1}})$, where $H_1 = |\mathcal{B}_1|$, so that $h_1 > h_2 >,\ldots, h_{H_1}$. Let $\lambda_1$ be the environment which sets 
\begin{align*}
    R_{\lambda_1}(s_{h_1},a_{h_1}) &= \min(1,\return{\pi_0} - \return{\pi_1})\\
    R_{\lambda_1}(s_{h_2},a_{h_2}) &= \min(1,\max(R_{\theta}(s_{h_2},a_{h_2}),R_{\theta}(s_{h_2},a_{h_2})+\return{\pi_0} - (\return{\pi_1} - R_{\theta}(s_{h_1},a_{h_1})) - 1)))\\
    &\vdots\\
    R_{\lambda_1}(s_{h_i},a_{h_i}) &= \min(1,\max(R_{\theta}(s_{h_i},a_{h_i}),R_{\theta}(s_{h_i},a_{h_i})+\return{\pi_0} - (\return{\pi_1} - \sum_{\ell=1}^i R_{\theta}(s_{h_\ell},a_{h_\ell})) - i))\\
    &\vdots
\end{align*}
Clearly $\lambda_1$ makes $\pi_1$ optimal and also does not change the value of any state-action pair which belongs to $\pi_0$ so it agrees with $\theta$ on $\pi_0$. Further $\pi_2,\pi_3,\ldots,\pi_m$ are still suboptimal policies under $\lambda_1$. This follows from the fact that for any $i>1$, $\return{\pi_1} > \return{\pi_i}$ and there exists $(s,a)$ such that $(s,a)\in \pi_i$ but $(s,a)\not \in \pi_1$ so $R_{\lambda_1}(s,a) = R_{\theta}(s,a)$. Further $\lambda_1$ only increases the rewards for state-action pairs in $\pi_1$ and hence $\return{\pi_1}_{\lambda_1} > \return{\pi_i}_{\lambda_1}$. Notice that there exists an index $\tilde H_1$ at which $R_{\lambda_1}(s_{h_{\tilde H_1}}, a_{h_{\tilde H_1}}) = \return{\pi_0} - (\return{\pi_1} - \sum_{\ell=1}^{\tilde H_1} R_{\theta}(s_{h_\ell},a_{h_\ell})) - \tilde H_1) \geq R_{\theta}(a_{\tilde H_1},s_{\tilde H_1})$. For this index it holds that for $h < \tilde H_1$, $R_{\lambda_1}(s_h,a_h) = 1$ and for $h> \tilde H_1$, $R_{\lambda_1}(s_h,a_h) = R_{\theta}(s_h,a_h)$.

Let 
\begin{align*}
    \Bcal_i &= \{(s,a) \in \pi_i : (s,a)\not\in \bigcup_{\ell < i} \pi_\ell\}\\
    \tilde\Bcal_i &= \{(s,a) \in \pi_i : (s,a)\in \bigcup_{\ell < i} \pi_\ell\}.
\end{align*} 
We first define an environment $\tilde\lambda_i$ on $(s,a)\in \tilde\Bcal_i$ as follows. $R_{\lambda_{i}}(s,a) = R_{\lambda_{\ell}}(s,a)$, where $\ell < i$ is such that $(s,a) \in \mathcal{B}_\ell$.
Let $\return{\pi_i}_{\tilde\lambda_i}$ be the value function of $\pi_i$ with respect to $\tilde\lambda_i$.
\begin{lemma}
\label{lem:val_func_comp}
It holds that $\return{\pi_i}_{\tilde\lambda_i} \leq \return{\pi_0}$.
\end{lemma}
\begin{proof}
Let $\tilde H_i$ be the index for which it holds that for $\ell \leq \tilde H_i$, $(s_{h_\ell},a_{h_\ell}) \in \pi_i \iff (s_{h_\ell},a_{h_\ell}) \in \Bcal_i$. Such a $\tilde H_i$ exists as there is a unique sub-tree $\Mcal_i$, of maximal depth, for which it holds that if $\pi_j \bigcap \Mcal_i \neq \emptyset \iff \pi_i \succeq \pi_j$. The root of this subtree is exactly at depth $H - h_{\tilde H_i}$. Let $\pi_j$ be any policy such that $\pi_j \succeq \pi_i$ and $\exists (s_{h_{\tilde H_i}},a_{h_{\tilde H_i}}) \in \pi_j$. By the maximality of $\Mcal_i$ such a $\pi_j$ exists. Because of the tree structure it holds that for any $h' > h_{\tilde H_i}$ if $(s_{h'},a_{h'}) \in \pi_{i} \implies (s_{h'},a_{h'}) \in \pi_{j}$ and hence $\tilde\lambda_i = \lambda_j$ up to depth $h_{\tilde H_i}$. Since $\pi_i$ and $\pi_j$ match up to depth $H - h_{\tilde H_i}$ and $\pi_j \succeq \pi_i$ it also holds that 
\begin{align*}
    \sum_{\ell \leq \tilde H_i} R_{\lambda_j}(s_{h_\ell}^{\pi_j},a_{h_\ell}^{\pi_j}) \geq \sum_{\ell \leq \tilde H_i} R_{\theta}(s_{h_\ell}^{\pi_j},a_{h_\ell}^{\pi_j}) \geq \sum_{\ell \leq \tilde H_i} R_{\theta}(s_{h_\ell}^{\pi_i},a_{h_\ell}^{\pi_i}) = \sum_{\ell \leq \tilde H_i} R_{\tilde\lambda_i}(s_{h_\ell}^{\pi_i},a_{h_\ell}^{\pi_i}).
\end{align*}
Since $\pi_j$ is optimal under $\lambda_j$ the claim holds.
\end{proof}
For all $(s_{h_j},a_{h_j}) \in \Bcal_i$ we now set
\begin{align}
\label{eq:ith_env_constr}
    R_{\lambda_i}(s_{h_j},a_{h_j}) = \min(1,\max(R_{\theta}(s_{h_j},a_{h_j}), R_{\theta}(s_{h_j},a_{h_j}) + \return{\pi_0} - (\return{\pi_i}_{\tilde\lambda_i} - \sum_{\ell=1}^j R_{\tilde\lambda_i}(s_{h_\ell},a_{h_\ell})) -j )),
\end{align}
and for all $(s_h,a_h) \in \tilde\Bcal_i$ we set $R_{\lambda_i}(s_h,a_h) = R_{\tilde\lambda_i}(s_h,a_h)$. From the definition of $\tilde\Bcal_i$ it follows that $\lambda_i$ agrees with all $\lambda_j$ for $j\leq i$ on state-action pairs in $\pi_i$. Finally we need to show that the construction in Equation~\ref{eq:ith_env_constr} yields an environment $\lambda_i$ for which $\pi_i$ is optimal.
\begin{lemma}
\label{lem:opt_of_lambdai}
Under $\lambda_i$ it holds that $\pi_i$ is optimal.
\end{lemma}
\begin{proof}
Let $\tilde H_i$ and $\pi_j$ be as in the proof of Lemma~\ref{lem:val_func_comp}. We now show that $\sum_{\ell \leq \tilde H_i} R_{\lambda_j}(s_{h_\ell}^{\pi_j},a_{h_\ell}^{\pi_j}) \leq \sum_{\ell \leq \tilde H_i} R_{\lambda_i}(s_{h_\ell}^{\pi_i},a_{h_\ell}^{\pi_i})$.
We only need to show that $\sum_{\ell \leq \tilde H_i} R_{\lambda_i}(s_{h_\ell}^{\pi_i},a_{h_\ell}^{\pi_i}) \geq \return{\pi_0} - \return{\pi_i}_{\tilde\lambda_i}$. From Equation~\ref{eq:ith_env_constr} we have $R_{\lambda_i}(s_{h_1},a_{h_1}) = \min(1,\return{\pi_0} - \return{\pi_i}_{\tilde\lambda_i})$. If $R_{\lambda_i}(s_{h_1},a_{h_1}) = \return{\pi_0} - \return{\pi_i}_{\tilde\lambda_i}$ then the claim is complete. Suppose $R_{\lambda_i}(s_{h_1},a_{h_1}) = 1$. This implies $\return{\pi_0} - \return{\pi_i}_{\tilde\lambda_i} \geq 1 - R_{\theta}(s_{h_1},a_{h_1})$. Next the construction adds the remaining gap of $\return{\pi_0} - \return{\pi_i}_{\tilde\lambda_i} + R_{\theta}(s_{h_1},a_{h_1}) - 1$ to $R_\theta(s_{h_2},a_{h_2})$ and clips $R_{\lambda_i}(s_{h_2},a_{h_2})$ to $1$ if necessary. Continuing in this way we see that if ever $R_{\lambda_i}(s_{h_j},a_{h_j}) = R_{\theta}(s_{h_j},a_{h_j}) + \return{\pi_0} - (\return{\pi_i}_{\tilde\lambda_i} - \sum_{\ell=1}^j R_{\tilde\lambda_i}(s_{h_\ell},a_{h_\ell})) -j$ then $\return{\pi_0} - V_{\tilde\lambda_i}^{\pi_i} \leq \sum_{\ell \leq \tilde H_i} R_{\lambda_i}(s_{h_\ell}^{\pi_i},a_{h_\ell}^{\pi_i})$. On the other hand if this never occurs, we must have $R_{\lambda_i}(s_{h_\ell}^{\pi_i},a_{h_\ell}^{\pi_i}) = 1 \geq R_{\lambda_j}(s_{h_\ell}^{\pi_j},a_{h_\ell}^{\pi_j})$ which concludes the claim.
\end{proof}

Let $\hat\Lambda(\theta) = \{\lambda_1,\ldots,\lambda_m\}$ be the set of the environments constructed above. We now show that the value of the optimization problem is not too much smaller than the value of Problem~\ref{eq:opt_prob}.

\begin{theorem}
\label{thm:tree_mdp_bound}
The value $\hat C(\theta)$ of the LP
\begin{align*}
    \minimize{\eta(\pi)\geq 0}{\sum_{\pi\in \Pi^*} \eta(\pi)(\return{*} - \return{\pi})}{ \sum_{\pi \in \Pi^*} \eta(\pi)KL(\theta(\pi),\lambda(\pi)) \geq 1,\forall \lambda \in \hat\Lambda(\theta)},
\end{align*}
satisfies $\hat C(\theta) \geq \frac{C(\theta)}{H^2}$ and $C(\theta) \geq \frac{\hat C(\theta)}{H}$.
\end{theorem}
\begin{proof}
The inequality $C(\theta) \geq \frac{\hat C(\theta)}{H}$ follows from Lemma~\ref{lem:primal_opt_upper_bound} and the fact that the above optimization problem is a relaxation to LP~\ref{eq:opt_primal_relaxed}.

To show the first inequality we consider the following relaxed LP
\begin{align*}
    \minimize{\eta(\pi)\geq 0}{\sum_{\pi\in \Pi} \eta(\pi)(\return* - \return{\pi})}{ \sum_{\pi \in \Pi} \eta(\pi)KL(\theta(\pi),\lambda(\pi)) \geq 1,\forall \lambda \in \hat\Lambda(\theta)}.
\end{align*}
Any solution to the LP in the statement of the theorem is feasible for the above LP and thus the value of the above LP is no larger. We now show that the value of the above LP is greater than or equal to $\frac{C(\theta)}{H^2}$. Fix $\lambda \in \hat \Lambda(\theta)$. We show that for any $\lambda' \in \Lambda(\theta)$ such that $\pi^*_{\lambda} = \pi^*_{\lambda'}$ it holds that $KL(\theta(\pi),\lambda(\pi)) \leq H^2 KL(\theta(\pi),\lambda'(\pi)),\forall \pi \in \Pi$. This would imply that if $\eta$ is a solution to the above LP, then $H^2 \eta$ is feasible for LP~\ref{eq:opt_prob} and therefore $\hat C(\theta) \geq \frac{C(\theta)}{H^2}$.

Arrange $\pi \in \Pi : KL(\theta(\pi),\lambda(\pi)) > 0$ according to $KL(\theta(\pi),\lambda(\pi))$ so that 
\begin{align*}
    \pi_i \preceq \pi_j \iff KL(\theta(\pi_i),\lambda(\pi_i)) \geq KL(\theta(\pi_j),\lambda(\pi_j)).
\end{align*} 
Consider the optimization problem
\begin{align*}
    \minimize{\lambda' \in \Lambda(\theta)}{KL(\theta(\pi_i),\lambda'(\pi_i))}{\pi^*_{\lambda'} = \pi^*_\lambda}.
\end{align*}
If we let $\Delta_{\lambda'}(s_h,a_h),(s_h,a_h) \in \pi^*_\lambda$ denote the change of reward for $(s_h,a_h)$ under environment $\lambda'$, then the above optimization problem can be equivalently written as
\begin{align*}
    \minimize{\lambda' \in \Lambda(\theta)}{\sum_{h=1}^{h_{\tilde H_i}} \Delta_{\lambda'}(s_h,a_h)^2}{\sum_{h=1}^H r(s_h,a_h) + \Delta_{\lambda'}(s_h,a_h) \geq \return*}.
\end{align*}
It is easy to see that the solution to the above optimization problem is to set $r(s_h,a_h) + \Delta_{\lambda'}(s_h,a_h) = 1$ for all $h \in [h_{\tilde H_i}+1,H]$ and spread the remaining mass of $\return* - \tilde H_i - (\return{\pi^*_\lambda} - \sum_{\ell=1}^{\tilde H_i}) R_\theta(s_{h_\ell},a_{h_\ell})$ as uniformly as possible on $\Delta_{\lambda'}(s_h,a_h)$, $h \in [1, h_{\tilde H_i}]$. Notice that under this construction the solution to the above optimization problem and $\lambda$ match for $h \in [h_{\tilde H_i}+1,H]$. Since the remaining mass is now the same it now holds that for any $\lambda'$, $\sum_{h=1}^{h_{\tilde H_i}} \Delta_{\lambda'}(s_h,a_h)^2 \geq \frac{1}{h_{\tilde H_i}^2} \sum_{h=1}^{h_{\tilde H_i}} \Delta_{\lambda}(s_h,a_h)^2$. This implies $KL(\theta(\pi_i),\lambda'(\pi_i)) \geq \frac{1}{\tilde H_i ^2} KL(\theta(\pi),\lambda(\pi))$ and the result follows as $\tilde H_i \leq H,\forall i \in [H]$.
\end{proof}

\subsubsection{Issue with deriving a general bound}
\label{app:lower_bounds_issues}
We now try to give some intuition regarding why we could not derive a generic lower bound for deterministic transition MDPs. We have already outlined our general approach of restricting the set $\Pi$ and $\Lambda(\theta)$ to finite subsets of manageable size and then showing that the value of the LP on these restricted sets is not much smaller than the value of the original LP. One natural restriction of $\Pi$ is the set $\Pi^*$ from Theorem~\ref{thm:lower_bound_deterministic}. Suppose we restrict ourselves to the same set and consider only environments making policies in $\Pi^*$ optimal as the restriction for $\Lambda(\theta)$. We now give an example of an MDP for which such a restriction will lead to an $\Omega(SA)$ multiplicative discrepancy between the value of the original semi-infinite LP and the restricted LP.
\begin{figure}
    \centering
    \includegraphics{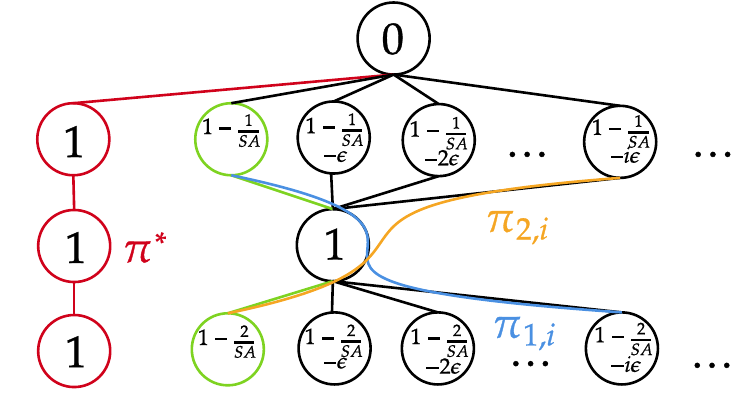}
    \caption{Issue with restricting LP to $\Pi^*$}
    \label{fig:lower_bound_counterexample}
\end{figure}
The MDP can be found in Figure~\ref{fig:lower_bound_counterexample}. The rewards for each action for a fixed state $s$ are equal and are shown in the vertices corresponding to the states. The number of states in the second and last layer of the MDP are equal to $(SA-3)/2$. The optimal policy takes the red path and has value $V^{\pi^*} = 3$. The set $\Pi^*$ consists of all policies $\pi_{j,i}$ which visit one of the states in green. The policies $\pi_{1,i}$, in blue, visit the green state in the second layer of the MDP and one of the states in the final layer, following the paths in blue. Similarly the policies $\pi_{2,i}$, in orange, visit one of the state in the second layer and the green state in the last layer, following the orange paths. The value function of $\pi_{j,i}$ is $V^{\pi_{j,i}} = 3 - \frac{3}{SA} - i\epsilon$, where $0\leq i \leq (SA-4)/2$. We claim that playing each $\pi_{j,i}$ $\eta(\pi_{j,i}) = \Omega(SA)$ times is a feasible solution to the LP restricted to $\Pi^*$. Fix $i$, the $\lambda_{\pi_{1,i}}$ must put weight at least $1/SA$ on the green state in layer 2. Coupling with the fact that for all $i'$ the rewards $\pi_{1,i'}$ are also changed under this environment we know that the constraint of the restricted LP with respect to $\lambda_{\pi_{1,i}}$ is lower bounded by
$\sum_{i'} \eta(\pi_{1,i'})/(SA)^2$. Since there are $\Omega(SA)$ policies $\{\pi_{1,i'}\}_{i'}$, this implies that $\eta(\pi_{1,i}) = \Omega(SA)$ is feasible. A similar argument holds for any $\pi_{2,i}$. Thus the value of the restricted LP is at most $O(SA)$, for any $\epsilon \ll SA$.

However, we claim that the value of the semi-infinite LP which actually characterizes the regret is at least $\Omega(S^2A^2)$. First, to see that the above assignment of $\eta$ is not feasible for the semi-infinite LP, consider any policy $\pi \not \in \Pi^*$, e.g. take the policy which visits the state in layer $2$ with reward $1-1/SA - \epsilon$ and the state in layer $4$ with reward $1-2/SA - \epsilon$. Each of these states have been visited $O(SA)$ times and $\eta(\pi) = 0$ hence the constraint for the environment $\lambda_{\pi}$ is upper bounded by $SA\left(\left(\frac{1}{SA} + \epsilon\right)^2 + \left(\left(\frac{2}{SA} + \epsilon\right)^2\right)\right) \approx 1/SA$. In general each of the states in black in the second layer and the fourth layer have been visited $1/SA$ times less than what is necessary to distinguish any $\pi \not \in \Pi^*$ as sub-optimal. If we define the $i$-th column of the MDP as the pair consisting of the states with rewards $1-1/SA - i\epsilon$ and $1-2/SA - i\epsilon$ then to distinguish the policy visiting both of these states as sub-optimal we need to visit at least one of these $\Omega(S^2A^2)$ times. This implies we need to visit each column of the MDP $\Omega(S^2A^2)$ times and thus any strategy must incur regret at least $\Omega\left(\sum_{i} S^2A^2 \frac{1}{SA}\right) = \Omega(S^2A^2)$, leading to the promised multiplicative gap of $\Omega(SA)$ between the values of the two LPs.

Why does such a gap arise and how can we hope to fix it this issue? Any feasible solution to the LP restricted to $\Pi^*$ essentially needs to visit the states in green $\Theta(S^2A^2)$ times. This is sufficient to distinguish the green states as sub-optimal to visit and hence any strategy visiting these states would be also deemed sub-optimal. This is achievable by playing each strategy in $\Pi^*$ in the order of $\Theta(SA)$ times as already discussed. Now, even though $\Pi^*$ covers all other states, from our argument above we see that we need to play each $\pi \in \Pi^*$ in the order of $\Theta(S^2A^2)$ times to be able to determine all sub-optimal states. To solve this issue, we either have to increase the size of $\Pi^*$ to include for example all policies visiting each column of the MDP or at the very least include changes of environments in the constraint set which make such policies optimal. This is clearly computationally feasible for the MDP in Figure~\ref{fig:lower_bound_counterexample}, however, it is not clear how to proceed for general MDPs, without having to include exponentially many constraints. This begs the question about the computational hardness of achieving both upper and lower regret bounds in a factor of $o(SA)$ from what is optimal.

\subsection{Lower bounds for optimistic algorithms in MDPs with deterministic transitions}
In this section we prove a lower bound on the regret of optimistic algorithms, demonstrating that optimistic algorithms can not hope to achieve the information-theoretic lower bounds even if the MDPs have deterministic transitions. While the result might seem similar to the one proposed by \citet{simchowitz2019non} (Theorem 2.3) we would like to emphasize that the construction of \citet{simchowitz2019non} does not apply to MDPs with deterministic transitions, and that the idea behind our construction is significantly different. 

\begin{figure}
\centering
\includegraphics[scale=0.5]{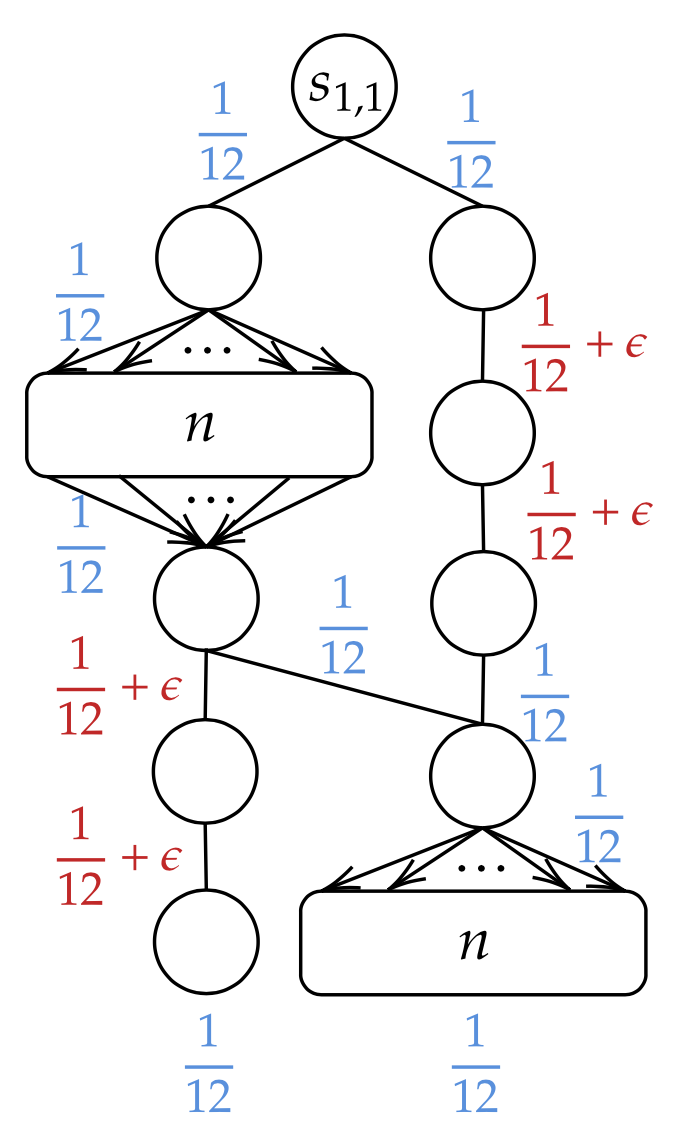}
\caption{Deterministic MDP instance for optimistic lower bound}
\label{fig:mdp_det_lower}
\end{figure}
Consider the MDP in Figure~\ref{fig:mdp_det_lower}. This MDP has $2n+9$ states and $4n+8$ actions. The rewards for each action are either $1/12$ or $1/12+\epsilon/2$ and can be found next to the transitions from the respective states. We are going to label the states according to their layer and their position in the layer so that the first state is $s_{1,1}$ the state which is to the left of $s_{1,1}$ in layer 2 is $s_{2,1}$ and to the right $s_{2,2}$. In general the $i$-th state in layer $h$ is denoted as $s_{h,i}$. The rewards in all states are deterministic, with a single exception of a Bernoulli reward from state $s_{4,1}$ to $s_{5,2}$ with mean $1/12$. From the construction it is clear that $V^*(s_{1,1}) = 1/2+\epsilon$. Further there are two sets of optimal policies with the above value function -- the $n$ optimal policies which visit state $s_{2,2}$ and the $n$ optimal policies which visit $s_{5,1}$. Notice that the information-theoretic lower bound for this MDP is in $O(\log(K)/\epsilon)$ as only the transition from state $s_{4,1}$ to $s_{5,2}$ does not belong to an optimal policy. In particular, there is no dependence on $n$. Next we try to show that the class of optimistic algorithms will incur regret at least $\Omega(n\log(\delta^{-1})/\epsilon)$.
\paragraph{Class of algorithms.}
We adopt the class of algorithms from Section G.2 in \citep{simchowitz2019non} with an additional assumption which we clarify momentarily. Recall that the class of algorithms assumes access to an optimistic value function $\bar V_k(s) \geq V^*(s)$ and optimistic Q-functions. 
In particular the algorithms construct optimistic Q and value functions as
\begin{align*}
    \bar V_k(s) &= \max_{a\in\Acal} \bar Q_k(s,a)\\
    Q_k(s,a) &= \hat r_k(s,a) + b_k^{rw}(s,a) + \hat p_k(s,a)^\top \bar V_k + b_k(s,a).
\end{align*}
We assume that there exists a $c\geq 1$ such that 
\begin{align*}
    \frac{c}{2}\sqrt{\frac{\log(M(1\lor n_k(s,a)))/\delta}{(1\lor n_k(s,a))}}\leq b_k^{rw}(s,a) \leq c\sqrt{\frac{\log(M(1\lor n_k(s,a)))/\delta}{(1\lor n_k(s,a))}}\,,
\end{align*}
where $M = \theta(n)$ and $b_k(s,a) \sim \sqrt{S}f_k(s,a)b_k^{rw}(s,a)$, where $f_k$ is a decreasing function in the number of visits to $(s,a)$ given by $n_k(s,a)$. 
For $n_k(s,a) = \Omega(n\log(n))$, we assume $b_k(s,a) \leq b_k^{rw}(s,a)$.
One can verify that this is true for the the Q and value functions of StrongEuler.
\paragraph{Lower bound.}
Let $\epsilon>0$ be sufficiently small to be specified later and let $N$ be such that 
\begin{align*}
    N = \lfloor\frac{c^2n\log(MN/(n\delta))}{16\epsilon^2}\rfloor\,.
\end{align*}
\begin{lemma}
\label{lem:det_low_bound1}
There exists $n_0,\epsilon_0$ such that for any pair of $n\geq n_0$ and $\epsilon\leq\epsilon_0$ and any $k\leq N$, with probability at least $1-\delta$, it holds that either $n_{k}(s_{5,1}) < N/4$, or $\bar Q_k(s_{4,1},1)<\bar Q_k(s_{4,1},2)$.
\end{lemma}
\begin{proof}
Assume $n_{k}(s_{5,1}) \geq N/4$, then we have
\begin{align*}
    \bar Q_k(s_{4,1},1) &= \frac{1}{4}+\epsilon +\sum_{i=4}^6 b_k^{rw}(s_{i,1},1)+b_k(s_{i,1})\\
    &\leq \frac{1}{4}+\epsilon + 6c\sqrt{\frac{\log(MN/(4\delta))}{N/4}}\leq \frac{1}{4}+\epsilon+\frac{48\epsilon}{\sqrt{n}}\,,
\end{align*}
where we assume $\epsilon$ is sufficiently small such that $b_k(s,a) \leq b_k^{rw}(s,a)$ for $n_k(s,a)\geq N/4$.

On the other hand, we have have with probability at least 1-$\delta$, that
$\forall k:\hat r_k(s_{4,1},2)+b_k^{rw}(s_{4,1},2)\geq 1/12$. Hence conditioned under that event, we have
\begin{align*}
    \bar Q_k(s_{4,1},2) &= \frac{1}{4} +b_k^{rw}(s_{4,1},2)+b_k(s_{4,1},2)+\max_{j\in\{2,\dots n+1}\sum_{i=5}^6 b_k^{rw}(s_{i,j},1)+b_k(s_{i,j},1)\\
    &\geq \frac{1}{4}+ c\sqrt{\frac{\log(MN/(n\delta))}{N/n}}
    \geq \frac{1}{4} + 4 \epsilon\,.
\end{align*}
The proof is completed for $n_0 = 48^2$.
\end{proof}
We can show the same for the upper part of the MDP.
\begin{lemma}
\label{lem:det_low_bound2}
There exists $n_0,\epsilon_0$ such that for any pair of $n\geq n_0$ and $\epsilon\leq\epsilon_0$ and any $k\leq N$, with probability at least $1-\delta$, it holds that either $n_{k}(s_{1,2}) < N/4$, or $\bar Q_k(s_{1,1},2)<\bar Q_k(s_{1,1},1)$.
\end{lemma}
\begin{proof}
First we split $\bar Q_k(s_{1,1},2)$ into the observed sum of mean rewards and bonuses from $s_{1,1}$ to $s_{5,2}$ and the value $\bar V_k(s_{5,2})$.
Then we upper bound $\bar Q_k(s_{1,1},1)$ by $\bar V_k(s_{5,2})$ and the maximum observed sum of mean rewards and bonuses along the paths passing by $s_{3,j}$ for $j\in[n]$.
Finally analogous to the proof of Lemma~\ref{lem:det_low_bound1}, it is straightforward show that the latter is always larger as long as the visitation count for $s_{2,2}$ exceeds $N/4$.
\end{proof}

\begin{theorem}
\label{thm:det_lower_bound}
There exists an MDP instance with deterministic transitions on which any optimistic algorithm with confidence parameter $\delta$ will incur expected regret of at least $\Omega(S\log(\delta^{-1})/\epsilon))$ while it is asymptotically possible to achieve  $\Omega(\log(K)/\epsilon)$ regret.
\end{theorem}
\begin{proof}
Taking the MDP from Figure~\ref{fig:mdp_det_lower}.
Applying Lemma~\ref{lem:det_low_bound1} and \ref{lem:det_low_bound2} shows that after $N$ episodes with probability at least $1-2\delta$, the visitation count of $s_{2,2}$ and $s_{5,1}$ each do not exceed $N/4$.
Hence there are at least $N/2$ episodes in which neither of them is visited, which means an $\epsilon$-suboptimal policy is taken. Hence the expected regret after $N$ episodes is at least
\begin{align*}
    (1-2\delta)\epsilon N/2 = \Omega\left(\frac{S\log(\delta^{-1})}{\epsilon}\right)\,.
\end{align*}
\end{proof}
Theorem~\ref{thm:det_lower_bound} has two implications for optimistic algorithms in MDPs with deterministic transitions.
\begin{itemize}
    \item It is impossible to be asymptotically optimal if the confidence parameter $\delta$ is tuned to the time horizon $K$.
    \item It is impossible to have an anytime bound matching the information-theoretic lower bound.
\end{itemize}

\section{Proofs and extended discussion for regret upper-bounds}
\label{app:upper_bounds}

\subsection{Further discussion on Opportunity~\ref{enum_prob_3}}
The example in \pref{fig:summary} does not illustrate \ref{enum_prob_3} to its fullest extent. We now expand this example and elaborate why it is important to address Opportunity~\ref{enum_prob_3}.
\begin{figure}[h]
    \centering
    \includegraphics[scale=0.4]{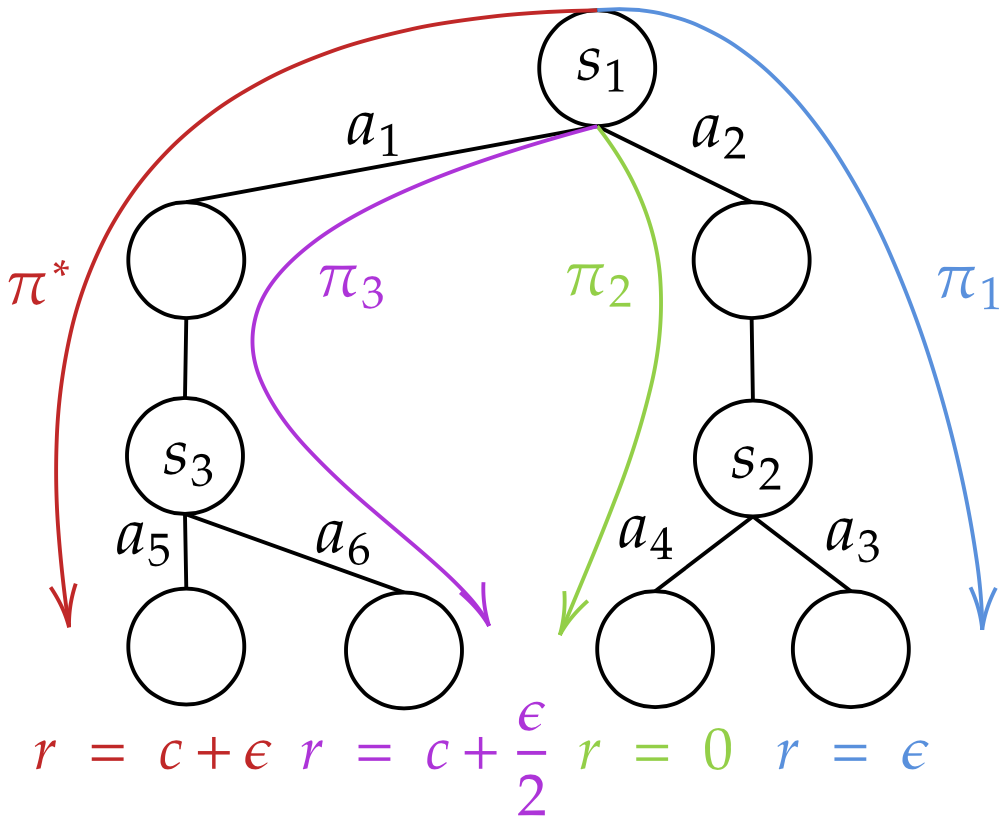}
    \caption{Example for Opportunity~\ref{enum_prob_3}}
    \label{fig:fail_gap2}
\end{figure}
Our example can be found in \pref{fig:fail_gap2}. The MDP is an extension of the one presented in Figure~\ref{fig:summary} with the new addition of actions $a_5$ and $a_6$ in state $s_3$ and the new state following action $a_6$. Again there is only a single action available at all other states than $s_1,s_2, s_3$. The reward of the state following action $a_6$ is set as $r=c+\epsilon/2$. This defines a new sub-optimal policy $\pi_3$ and the gap $\gap(s_3,a_6) = \frac{\epsilon}{2}$. Information theoretically it is impossible to distinguish $\pi_3$ as sub-optimal in less than $\Omega(\log(K)/\epsilon^2)$ rounds and so any uniformly good algorithm would have to pay at least $O(\log(K)/\epsilon)$ regret. However, what we observed previously still holds true, i.e., we should not have to play more than $\log(K)/c^2$ rounds to eliminate both $\pi_1$ and $\pi_2$ as sub-optimal policies. Prior work now suffers Opportunity~\ref{enum_prob_3} as it would pay $\log(K)/\epsilon$ regret for all zero gap state-action pairs belonging to either $\pi_1$ or $\pi_2$, essentially evaluating to $SA\log(K)/\epsilon$. On the other hand our bounds will only pay $\log(K)/\epsilon$ regret for zero gap state-action pairs belonging to $\pi_3$.

\subsection{Useful decomposition lemmas}
We start by providing the following lemma that establishes that the instantaneous regret can be decomposed into gaps defined w.r.t. any optimal (and not necessarily Bellman optimal) policy.
\begin{lemma}[General policy gap decomposition]
\label{lem:gap_decomp_pi}
Let $\gap^{\hat \pi}(s,a) = V^{\hat \pi}(s) - Q^{\hat \pi}(s,a)$ for any optimal policy $\hat \pi \in \Pi^*$. Then the difference in values of $\hat \pi$ and any policy $\pi \in \Pi$ is
\begin{align}
    V^{\hat \pi}(s) - V^{\pi}(s)
 = \EE_{\pi}\left[\sum_{h=\kappa(s)}^H \gap^{\hat \pi}(S_h, A_h) ~ \bigg| ~ S_{\kappa(s)} = s\right]
\end{align}
and, further, the instantaneous regret of $\pi$ is
\begin{align}
\label{eq:gap_decomp_pi}
    \return{*} - \return{\pi} = \sum_{s,a} w^{\pi}(s,a) \gap^{\hat \pi}(s,a).
\end{align}
\end{lemma}
\begin{proof}
We start by establishing a recursive bound for the value difference of $\pi$ and $\hat \pi$ for any $s$
\begin{align*}
    V^{\hat \pi}(s) - V^{\pi}(s)
    &= V^{\hat \pi}(s) 
    - Q^{\hat \pi}(s, \pi(s)) 
    +  Q^{\hat \pi}(s, \pi(s)) 
    - V^{\pi}(s)\\
    &= \gap^{\hat \pi}(s, \pi(s)) + Q^{\hat \pi}(s, \pi(s)) 
    - Q^{\pi}(s, \pi(s))\\
&= \gap^{\hat \pi}(s, \pi(s))
+ \sum_{s'} P_\theta(s' | s, \pi(s)) 
[V^{\hat \pi}(s') - V^{\pi}(s')].
\end{align*}
Unrolling this recursion for all layers gives
\begin{align*}
    V^{\hat \pi}(s) - V^{\pi}(s)
 = \EE_{\pi}\left[\sum_{h=\kappa(s)}^H \gap^{\hat \pi}(S_h, A_h) ~ \bigg| ~ S_{\kappa(s)} = s\right].
\end{align*}
To show the second identity, consider $s=s_1$ and note that $\return{\pi} = V^\pi(s_1)$ and $\return{*} = \return{\hat \pi} = V^{\hat \pi}(s_1)$ because $\hat \pi$ is an optimal policy.
\end{proof}
For the rest of the paper we are going to focus only on the Bellman optimal policy from each state and hence only consider $\gap^{\hat\pi}(s,a) = \gap(s,a)$. All of our analysis will also go through for arbitrary $\gap^{\hat\pi},\hat\pi\in \Pi^*$, however, this did not provide us with improved regret bounds.

We now show the following technical lemma which generalizes the decomposition of value function differences and will be useful in the surplus clipping analysis.
\begin{lemma}
\label{lem:rec_rel}
Let $\Psi:\Scal \rightarrow \RR$, $\Delta: \Scal \times \Acal \rightarrow \RR$ be functions satisfying $\Psi(s) = 0$ for any $s$ with $\kappa(s) = H+1$ and $\pi \colon \Scal \rightarrow \Acal$ a deterministic policy. Further, assume that the following relation holds
\begin{align*}
    \Psi(s)  = \Delta(s, \pi(s)) + \langle P(\cdot|s,\pi(s)), \Psi \rangle,
\end{align*}
and let $\Acal$ be any event that is $\Hcal_h$-measurable where $\Hcal_h = \sigma(S_1, A_1, R_1, \dots, S_h)$ is the sigma-field induced by the episode up to the state at time $h$.
Then, for any $h \in [H]$ and $h' \in \NN$ with $h \leq h' \leq H+1$, it holds that
\begin{align*}
    \EE_{\pi}[\indicator{\Acal}\Psi(S_h))] = \EE_{\pi}\left[\indicator{\Acal}
    \left(\sum_{t=h}^{h' - 1} \Delta(S_{t}, A_t) + \Psi(S_{h'+1})\right)\right]
    = \EE_{\pi}\left[\indicator{\Acal}
    \sum_{t=h}^{H} \Delta(S_{t}, A_t)\right]
    .
\end{align*}
\end{lemma}

\begin{proof}
First apply the assumption of $\Psi$ recursively to get
\begin{align*}
    \Psi(s) = \EE_\pi \left[ \sum_{t=\kappa(s)}^{h' - 1} \Delta(S_t, A_t) + \Psi(S_{h'}) ~\Bigg|~ S_{\kappa(s)} = s\right].
\end{align*}
Plugging this identity into $\EE_{\pi}[\indicator{\Acal}\Psi(S_h))]$ yields
\begin{align*}
    \EE_{\pi}[\indicator{\Acal}\Psi(S_h))]
    &= \EE_\pi \left[
    \indicator{\Acal}
    \EE_\pi \left[ \sum_{t=h}^{h' - 1} \Delta(S_t, A_t) + \Psi(S_{h'})~\Bigg|~ S_{h}\right]\right]\\
         &\overset{(i)}{=} \EE_\pi \left[
    \indicator{\Acal}
    \EE_\pi \left[ \sum_{t=h}^{h' - 1} \Delta(S_t, A_t) + \Psi(S_{h'}) ~\Bigg|~ \Hcal_h \right]\right]\\
    &\overset{(ii)}{=} \EE_\pi \left[
    \EE_\pi \left[ \indicator{\Acal}\left( \sum_{t=h}^{h' - 1} \Delta(S_t, A_t) + \Psi(S_{h'}) \right) ~\Bigg|~ \Hcal_h \right]\right]
    \\
    &\overset{(iii)}{=}\EE_\pi \left[\indicator{\Acal}\left( \sum_{t=h}^{h' - 1} \Delta(S_t, A_t) + \Psi(S_{h'}) \right) \right]
\end{align*}
where $\Hcal_h = \sigma(S_1, A_1, R_1, \dots, S_h)$ is the sigma-field induced by the episode up to the state at time $h$. Identity $(i)$ holds because of the Markov-property and $(ii)$ holds because $\Acal$ is $\Hcal_h$-measurable. The final identity $(iii)$ uses the tower-property of conditional expectations.
\end{proof}

%

\subsection{General surplus clipping for optimistic algorithms}

\paragraph{Clipped operators.}
One of the main arguments to derive instance dependent bounds is to write the instantaneous regret in terms of the surpluses which are clipped to the minimum positive gap. We now define the clipping threshold $\epsilon_{k} : \Scal \times \Acal \rightarrow \RR^+_0$ and associated clipped surpluses
\begin{align}
\label{eq:clipped_surp_def}
    \ddot E_{k}(s,a) = \clip\left[E_k(s,a) \mid \epsilon_{k}(s,a)\right] = \indicator{ E_{k}(s,a) \geq \epsilon_k(s,a)} E_{k}(s,a).
\end{align}
Next, define the clipped $Q$- and value-function as
\begin{equation}
\label{eq:clipped_value_def}
    \begin{aligned}
        \ddot Q_{k}(s,a) &= \ddot E_{k}(s,a) + r(s,a) + \langle P(\cdot|s,a), \ddot V_{k}, \rangle \quad\textrm{and}
        &
        \ddot V_{k}(s) &= \ddot Q_{k}(s,\pi_k(s)).
    \end{aligned}
\end{equation}

The random variable which is the state visited by $\pi_k$ at time $h$ throughout episode $k$ is denoted by $S_h$ and $A_h$ is the action at time $h$.

\paragraph{Events about encountered gaps} Define the event 
$\Ecal_h = \{\gap(S_h, A_h) > 0\}$ that at time $h$ an action with a positive gap played,
the $\Pcal_{1:h} = \bigcap_{h' =1 }^{h-1} \Ecal_{h'}^{c}$ that only actions with zero gap have been played until $h$
and the event 
$\Acal_h = \Ecal_h \cap \Pcal_{1:h}$ that the first positive gap was encountered at time $h$. Let $\Acal_{H+1}= \Pcal_{1:H}$ be the event that only zero gaps were encountered. 
Further, let 
\begin{align*}
    B = \min\{ h \in [H+1] \colon \gap(S_h, A_h) > 0 \}
\end{align*} 
be the first time a non-zero gap is encountered. Note that $B$ is a stopping time w.r.t. the filtration $\mathcal{F}_h = \sigma(S_1, A_1, \dots, S_h, A_h)$.

The proof of \citet{simchowitz2019non} consists of two main steps. First show that for their definition of clipped value functions one can bound $\ddot V_k(s_1) - V^{\pi_k}(s_1) \geq \frac{1}{2}(\bar V_k(s_1) - V^{\pi_k}(s_1))$. Next, using optimism together with the fact that $\pi_k$ has highest value function at episode $k$ it follows that $\bar V_k(s_1) - V^{\pi_k}(s_1) \geq V^*(s_1) - V^{\pi_k}(s_1)$. The second main step is to use a high-probability bound on the clipped surpluses to relate them to the probability to visit the respective state-action pair and the proof is finished via an integration lemma. We now show that the first step can be carried out in greater generality by defining a less restrictive clipping operator. This operator is independent of the details in the definition of gap at each state-action pair but rather only uses a certain property which allows us to decompose the episodic regret as a sum over gaps. We will also further show that one does not need to use an integration lemma for the second step but can rather reformulate the regret bound as an optimization problem. This will allow us to clip surpluses at state-action pairs with zero gaps beyond the $\gap_{\min}$ rate.


\paragraph{Clipping with an arbitrary threshold.} Recall the definition of the clipped surpluses and clipped value function in Equation~\ref{eq:clipped_surp_def} and Equation~\ref{eq:clipped_value_def}. We begin by showing a general relation between the clipped value function difference and the non-clipped surpluses for any clipping threshold $\epsilon_{k} : \Scal \rightarrow \RR$. This will help in establishing $\ddot V_{k}(s_1) - V^{\pi_k}(s_1) \geq \frac{1}{2}(\bar V_k(s_1) - V^{\pi_k}(s_1))$.
\begin{lemma}
\label{lem:Vdd_lb1}
Let $\epsilon_k : \Scal \times \Acal \rightarrow \RR^+_0$ be arbitrary. Then for any optimistic algorithm it holds that 
\begin{equation}
    \label{eq:clipped_ineq_tight}
    \ddot V_k(s_1) - V^{\pi_k}(s_1) \geq 
    \mathbb{E}_{\pi_k}\left[
    \sum_{h=B}^H\left( \gap(S_{h},A_h) - \epsilon_{k}(S_{h}, A_h)\right)
    \right].
\end{equation}
\end{lemma}

\begin{proof}
We use $W_k(s) = \ddot V_k(s) - V^{\pi_k}(s)$ in the following and first show that $W(s_1) \geq \EE_{\pi_k}[ W_k(S_B)]$.
As a precursor, we prove 
\begin{align}
    \EE_{\pi_k}\left[ \indicator{\Pcal_{1:h} } W_k(S_h)\right]
    \geq 
    \EE_{\pi_k}\left[  \indicator{ \Acal_{h+1} } W_k(S_{h+1})\right]
    + 
    \EE_{\pi_k}\left[ \indicator{ \Pcal_{1:h+1}  } W_k(S_{h+1})\right].
    \label{eq:first_rec}
\end{align}
To see this, plug the definitions into $W_k(s)$ which gives $W_k(s) = \ddot V_k(s) - V^{\pi_k}(s) = \ddot E_k(s, \pi_k(s)) + \langle P(\cdot | s, \pi_k(s)), W_{k} \rangle $ and use this in the LHS of \pref{eq:first_rec} as
\begin{align*}
\EE_{\pi_k}\left[ \indicator{\Pcal_{1:h}  } W_k(S_h)\right]
&= \EE_{\pi_k}\big[ \indicator{\Pcal_{1:h}  }
\underset{\geq 0}{\underbrace{ \ddot E_k(S_h, A_h)}\big]}
+ 
\EE_{\pi_k}\left[ \indicator{\Pcal_{1:h}} \EE[ W_k(S_{h+1}) \mid S_h]\right]
\\
& \overset{(i)}{\geq} 
\EE_{\pi_k}\left[ \indicator{\Pcal_{1:h}} \EE_{\pi_k}[ W_k(S_{h+1}) \mid \Hcal_{h}]\right]
\\
& \overset{(ii)}{=}
\EE_{\pi_k}\left[ \EE_{\pi_k}[ \indicator{\Pcal_{1:h} }  W_k(S_{h+1}) \mid \Hcal_{h}]\right] =
\EE_{\pi_k}\left[ \indicator{\Pcal_{1:h}}  W_k(S_{h+1}) \right]
\end{align*}
where $\Hcal_h = \sigma(S_1, A_1, R_1, \dots, S_h)$ is the sigma-field induced by the episode up to the state at time $h$. Step $(i)$ follows from $\clip[\cdot | c] \geq 0$ for any $c \geq 0$ and the Markov property and $(ii)$ holds because $\Pcal_{1:h}$ is $\Hcal_h$-measurable.
We now rewrite the RHS by splitting the expectation based on whether event $\Ecal_{h+1}$ occurred as
\begin{align*}
    \EE_{\pi_k}\left[ \indicator{\Pcal_{1:h} }  W_k(S_{h+1}) \right]
    = 
    \EE_{\pi_k}\left[ \indicator{\Pcal_{1:h+1} }  W_k(S_{h+1}) \right]
    + \EE_{\pi_k}\left[ \indicator{\Acal_{h+1}}  W_k(S_{h+1}) \right].
\end{align*}
We have now shown \pref{eq:first_rec}, which we will now use to lower-bound $W_k(s_1)$ as
\begin{align*}
W_k(s_1) 
&= \EE_{\pi_k}[\indicator{\Ecal_1} W_1(S_1)] + \EE_{\pi_k}[\indicator{\Ecal_1^c} W_1(S_1)]\\
&= \EE_{\pi_k}[\indicator{\Acal_1} W_1(S_1)] + \EE_{\pi_k}[\indicator{\Pcal_{1:1}} W_1(S_1)]\\
& \geq  \EE_{\pi_k}[\indicator{\Acal_1} W_1(S_1)] + \sum_{h=2}^H \EE_{\pi_k}\left[  \indicator{ \Acal_{h} } W_k(S_{h})\right]\\
&= \sum_{h=1}^H \EE_{\pi_k}\left[  \indicator{ \Acal_{h} } W_k(S_{h})\right] =
\EE_{\pi_k}[ W_k(S_B)].
\end{align*}
Applying \pref{lem:rec_rel} with $\Acal = \Acal_h$, $\Psi = W_k$ and $\Delta = \ddot E_k$ yields
\begin{align*}
W_k(s_1) 
&\geq \sum_{h=1}^H \EE_{\pi_k}\left[  \indicator{ \Acal_{h}} \sum_{h' = h}^H \ddot E_k(S_{h'}, A_{h'})\right]\\
&\geq \sum_{h=1}^H \EE_{\pi_k}\left[  \indicator{ \Acal_{h} } \sum_{h' = h}^H E_k(S_{h'}, A_{h'})\right] - \sum_{h=1}^H \EE_{\pi_k}\left[  \indicator{ \Acal_{h} } \sum_{h' = h}^H \epsilon_{k}(S_{h'}, A_{h'})\right],
\end{align*}
where we applied the definition clipped surpluses which gives $\ddot E_k(s,a) = \clip[ E_k(s,a) \mid \epsilon_{k}(s,a)] \geq E_k(s,a) - \epsilon_{k}(s,a)$. It only remains to show that 
\begin{align*}
\EE_{\pi_k}\left[  \indicator{ \Acal_{h} } \sum_{h' = h}^H E_k(S_{h'}, A_{h'})\right] \geq 
\EE_{\pi_k}\left[  \indicator{ \Acal_{h} } \sum_{h' = h}^H \gap(S_{h'}, A_{h'})\right].
\end{align*}
To do so, we apply \pref{lem:rec_rel} twice, first with $\Acal = \Acal_h$, $\Psi = \bar V_k - V^{\pi_k}$ and $\Delta = E_k$ and then again with $\Acal = \Acal_h$, $\Psi =  V^{*} - V^{\pi_k}$ and $\Delta = \gap$ which gives
\begin{align*}
\EE_{\pi_k}\left[  \indicator{ \Acal_{h} } \sum_{h' = h}^H E_k(S_{h'}, A_{h'})\right] 
&= 
\EE_{\pi_k}\left[  \indicator{ \Acal_{h} }  (\bar V_k(S_h) - V^{\pi_k}(S_{h}))\right] \\
&\geq 
\EE_{\pi_k}\left[  \indicator{ \Acal_{h} }  (V^*(S_h) - V^{\pi_k}(S_{h}))\right]\\
&=
\EE_{\pi_k}\left[  \indicator{ \Acal_{h} } \sum_{h' = h}^H \gap(S_{h'}, A_{h'})\right].
\end{align*}
Thus, we have shown that 
\begin{align*}
    & \ddot V_k(s_1) - V^{\pi_k}(s_1) = W_k(s_1) \\
&\geq \sum_{h=1}^H \EE_{\pi_k}\left[  \indicator{ \Acal_{h} } \sum_{h' = h}^H \gap(S_{h'}, A_{h'})\right] - \sum_{h=1}^H \EE_{\pi_k}\left[  \indicator{ \Acal_{h} } \sum_{h' = h}^H \epsilon_{k}(S_{h'}, A_{h'})\right]\\
&= \sum_{h=1}^H \EE_{\pi_k}\left[  \indicator{ \Acal_{h} } \sum_{h' = h}^H \left( \gap(S_{h'}, A_{h'}) - \epsilon_{k}(S_{h'}, A_{h'}) \right)\right]\\
&= \EE_{\pi_k}\left[  \sum_{h = B}^H \left( \gap(S_{h}, A_{h}) - \epsilon_{k}(S_{h}, A_{h}) \right)\right]
\end{align*}
where the last equality uses the definition of $B$, the first time step at which a non-zero gap was encountered.
\end{proof}

\begin{lemma}[Optimism of clipped value function]
\label{lem:opt_clipped_weak_stopping_time}
Let the clipping thresholds $\epsilon_k \colon \Scal \times \Acal \rightarrow \RR^+_0$ used in the definition of $\ddot V_k$ satisfy
\begin{align*}
    \EE_{\pi_k}\left[\sum_{h=B}^H \epsilon_k(S_h, A_h) \right]
    \leq \frac{1}{2} \EE_{\pi_k} \left[\sum_{h=1}^H \gap(S_h, A_h) \right]
\end{align*}
for some optimal policy $\hat \pi$. Then scaled optimism holds for the clipped value function, i.e.,
\begin{align*}
    \ddot V_k(s_1) - V^{\pi_k}(s_1) 
    \geq \frac{1}{2} (V^*(s_1) - V^{\pi_k}(s_1)).
\end{align*}
\end{lemma}
\begin{proof}
The proof works by establishing the following chain of inequalities:
\begin{align*}
    \frac{V^*(s_1) - V^{\pi_k}(s_1)}{2} 
    &\overset{(a)}{=} 
    \frac{1}{2} \EE_{\pi_k}\left[\sum_{h=1}^H\gap(S_{h},A_{h})\right]
        \overset{(b)}{=} 
    \frac{1}{2} \EE_{\pi_k}\left[\sum_{h=B}^H\gap(S_{h},A_{h}))\right]
    \\
    &\overset{(c)}{=}
    \EE_{\pi_k}\left[\sum_{h=B}^H\left( \gap(S_{h},A_{h})) - \frac{1}{2}\gap(S_{h},A_{h}))\right)\right]\\
    &\overset{(d)}{\leq}
    \EE_{\pi_k}\left[\sum_{h=B}^H\left( \gap(S_{h},A_{h})) - \epsilon_k(S_{h},A_{h}))\right)\right]\\
    &\overset{(e)}{\leq} \ddot V_k(s_1) - V^{\pi_k}(s_1).
\end{align*}
Here, $(a)$ uses Lemma~\ref{lem:gap_decomp_pi} and $(b)$ uses the definition of $B$. Step $(c)$ is just algebra and
step $(d)$ uses the assumption on the threshold function. 
The last step $(e)$ follows from Lemma~\ref{lem:Vdd_lb1}.
\end{proof}

\surplusclippingbound*
\begin{proof}
Applying \pref{lem:opt_clipped_weak_stopping_time} which ensures scaled optimism of the clipped value function gives
\begin{align*}
    V^*(s_1) - V^{\pi_k}(s_1) \leq 2(\ddot V_k(s_1) - V^{\pi_k}(s_1))
    = 2\sum_{s,a} w^{\pi_k}(s,a) \ddot E_k(s,a),
\end{align*}
where the equality follows from the definition of $\ddot V_k(s_1)$ and \pref{lem:rec_rel}. Subtracting $\frac 1 2 (V^*(s_1) - V^{\pi_k}(s_1))$ from both sides gives
\begin{align*}
   \frac{1}{2}( V^*(s_1) - V^{\pi_k}(s_1) )\leq  2\sum_{s,a} w^{\pi_k}(s,a) \left( \ddot E_k(s,a) - \frac{\gap(s,a)}{4} \right)
\end{align*}
because \pref{lem:gap_decomp_pi} ensures that $\frac{1}{2}(V^*(s_1) - V^{\pi_k}(s_1)) = \frac{1}{2} \sum_{s,a} w^{\pi_k}(s,a) \gap(s,a)$. Reordering terms yields
\begin{align*}
    V^*(s_1) - V^{\pi_k}(s_1) 
    & \leq 4 \sum_{s,a} w^{\pi_k}(s,a) \left( \ddot E_k(s,a) - \frac{\gap(s,a)}{4} \right)\\
    & = 4 \sum_{s,a} w^{\pi_k}(s,a) \left( \clip\left[ E_k(s,a) ~\bigg| ~ \epsilon_{k}(s, a) \right] - \frac{\gap(s,a)}{4} \right)\\
    & \leq 4 \sum_{s,a} w^{\pi_k}(s,a)  \clip\left[ E_k(s,a) ~\bigg| ~ \epsilon_{k}(s, a)  \vee \frac{\gap(s,a)}{4} \right],
\end{align*}
where the final inequality follows from the general properties of the clipping operator, 
which satisfies 
\begin{align*}
    \clip[a | b ] - c = 
    \begin{cases}
    a - c \leq a & \textrm{for } a \geq b \vee c\\
    0 - c \leq 0 & \textrm{for } a \leq b\\
    a - c \leq 0 & \textrm{for } a \leq c
    \end{cases}
    \leq \clip[a | b \vee c].
\end{align*}
\end{proof}

\subsection{Definition of valid clipping thresholds $\epsilon_{k}$}

\pref{prop:surplus_clipping_bound} establishes a sufficient condition on the clipping thresholds $\epsilon_k$ that ensures that the penalized surplus clipping bounds holds.
We now discuss several choices for this threshold that satisfy this condition.

\paragraph{Minimum positive gap $\gap_{\min}$:} 
We now make the quick observation that taking $\epsilon_{k} \equiv \frac{\gap_{\min}}{2H}$ will satisfy the condition of Proposition~\ref{prop:surplus_clipping_bound}, because on the event $\Bcal \equiv \Acal_{H+1}^c$ there exists at least one positive gap in the sum $\sum_{h=1}^H\gap(S_h,A_h)$, which, by definition, is at least $\gap_{\min}$. This shows that our results already can recover the bounds in prior work, with significantly less effort.

\paragraph{Average gaps:}Instead of the minimum gap which was used in existing analyses, we now show that we can also use the marginalized average gap which we will define now.
Recall that $B = \min\{ h \in [H+1] \colon \gap(S_h, A_h) > 0 \}$ is the first time a non-zero gap is encountered. Note that $B$ is a stopping time w.r.t. the filtration $\mathcal{F}_h = \sigma(S_1, A_1, \dots, S_h, A_h)$. Further let
\begin{align}
    \Bcal(s,a) \equiv \{B \leq \kappa(s), S_{\kappa(s)} = s, A_{\kappa(s)} = a\} 
\end{align}
be the event that $(s,a)$ was visited after a non-zero gap in the episode.
We now define this clipping threshold
\begin{align}
    \label{eqn:new_avgclip_clean}
    \epsilon_k(s,a) \equiv 
    \begin{cases}
    \frac{1}{2H}
    \EE_{\pi_k}\left[ \sum_{h=1}^H \gap(S_h, A_h) ~ \bigg| ~\Bcal(s,a) \right]
    & \textrm{if }\PP_{\pi_k}(\Bcal(s,a)) > 0\\
    \infty & \textrm{otherwise}
    \end{cases}
\end{align}
As the following lemma shows, this is a valid choice which satisfies the condition of \pref{prop:surplus_clipping_bound}.
\begin{lemma}
\label{lem:clipping_gaps_rel}
The expected sum of clipping thresholds in Equation~\eqref{eqn:new_avgclip_clean} over all state-action pairs encountered after a positive gap is at most half the expected total gaps per episode. That is,
\begin{align*}
    \EE_{\pi_k}\left[\sum_{h=B}^H \epsilon_k(S_h, A_h) \right]
    \leq \frac{1}{2} \EE_{\pi_k} \left[\sum_{h=1}^H \gap(S_h, A_h) \right].
\end{align*}
\end{lemma}
\begin{proof}
We rewrite the LHS of the inequality to show as
$\EE_{\pi_k}\left[\sum_{h=1}^H \indicator{B \leq h}\epsilon_k(S_h, A_h) \right]$ and from now on consider the random variable $f_h(B, S_h, A_h) = \indicator{B \leq h}\epsilon_k(S_h, A_h)$ where $f_h(b, s, a) = \indicator{b \leq h}\epsilon_k(s,a)$  is a deterministic function\footnote{It may still depend on the current policy $\pi_k$ which is determined by observations in episodes $1$ to $k-1$. But, crucially, $f_h$ does not depend on any realization in the $k$-th episode}.
We will show below that $\EE_{\pi_k}\left[ f_h(B, S_h, A_h)\right] \linebreak[1]\leq \linebreak[1]\frac{1}{2H}\EE_{\pi_k}\left[\sum_{h=B}^H \gap(S_{h}, A_{h}) \right]$. This is sufficient to prove the statement, because
\begin{align*}
    \EE_{\pi_k}\left[\sum_{h=B}^H \epsilon_k(S_h, A_h) \right]
        &= \sum_{h=1}^H \EE_{\pi_k}\left[ f_h(B, S_h, A_h)\right]\\
       & \leq \frac{1}{2H}\sum_{h=1}^H\EE_{\pi_k}\left[
    \sum_{h'=B}^H \gap(S_{h'}, A_{h'}) 
    \right]\\
    &= \frac{1}{2}\EE_{\pi_k}\left[
    \sum_{h=B}^H \gap(S_{h}, A_{h}) 
    \right]
    = \frac{1}{2}\EE_{\pi_k}\left[
    \sum_{h=1}^H \gap(S_{h}, A_{h}) 
    \right].
    \end{align*}
To bound the expected value of $f_h(B, S_h, A_h)$, we first write $f_h$ for all triples $b, s, a$ such that $\PP_{\pi_k}(B = b, A_h = a, S_h = s) > 0$ as
\begin{align*}
    f_h(b, s, a) &\overset{(i)}{=} 
    \indicator{b \leq h}
    \frac{1}{2H}
    \EE_{\pi_k}\left[ \sum_{h'=1}^H \gap(S_{h'}, A_{h'}) ~ \bigg| ~B \leq h, ~S_h = s, A_h = a \right]\\
    &\overset{(ii)}{=} 
        \indicator{b \leq h}
    \frac{1}{2H}
    \EE_{\pi_k}\left[ \sum_{h'=B}^h \gap(S_{h'}, A_{h'}) ~ \bigg| ~B \leq h, ~S_h = s, A_h = a \right]\\
        &\quad   + \indicator{b \leq h}
    \frac{1}{2H}
    \EE_{\pi_k}\left[ \sum_{h'=h+1}^H \gap(S_{h'}, A_{h'}) ~ \bigg| S_h = s, A_h = a \right],
\end{align*}
where $(i)$ expands the definition of $\epsilon_k$ and $(ii)$ decomposes the sum inside the conditional expectation and uses the Markov-property to simplify the conditioning for terms after $h$.
Before taking the expectation of $f_h(B,S_h, A_h)$, we first rewrite the conditional expectation in the first term above, which will be useful later.
\begin{align*}
    &\EE_{\pi_k}\left[ \sum_{h'=B}^h \gap(S_{h'}, A_{h'}) ~ \bigg| ~B \leq h, ~S_h = s, A_h = a \right]\\
    &\overset{(i)}{=} \frac{
    \EE_{\pi_k}\left[ \sum_{h'=B}^h \gap(S_{h'}, A_{h'}) \indicator{A_h = a, S_h = s} \indicator{B \leq h}\right]
    }{    \PP_{\pi_k}\left[ B \leq h, ~S_h = s, A_h = a \right]}\\
        &\overset{(ii)}{=} \frac{
    \EE_{\pi_k}\left[ \sum_{h'=B}^h \gap(S_{h'}, A_{h'}) \indicator{A_h = a, S_h = s} \right]
    }{    \PP_{\pi_k}\left[ B \leq h, ~S_h = s, A_h = a \right]}
    \\
        &= \frac{
    \EE_{\pi_k}\left[ \sum_{h'=B}^h \gap(S_{h'}, A_{h'}) ~ \bigg| ~ S_h = s, A_h = a \right]
    }{    \PP_{\pi_k}\left[ B \leq h ~ \mid  ~S_h = s, A_h = a \right]}.
\end{align*}
Here, step $(i)$ uses the property of conditional expectations with respect to an event with nonzero probability and $(ii)$ follows from the definition of $B$: When $B > h$, the sum of gaps until $h$ is zero.
Consider now the expectation of $f_h(B, S_h, A_h)$
\begin{align}
    &\EE_{\pi_k}\left[ f_h(B, S_h, A_h)\right]\nonumber\\
    &=
    \frac{1}{2H}\EE_{\pi_k}\left[\indicator{B \leq h}
    \frac{
    \EE_{\pi_k}\left[ \sum_{h'=B}^h \gap(S_{h'}, A_{h'}) ~ \bigg| ~ S_h, A_h \right]
    }{    \PP_{\pi_k}\left[ B \leq h ~ \mid  ~S_h, A_h \right]}
    \right]
    \label{eqn:term1_fexp}
    \\
        &\quad   + 
    \frac{1}{2H} \EE_{\pi_k}\left[\indicator{B \leq h}
    \EE_{\pi_k}\left[ \sum_{h'=h+1}^H \gap(S_{h'}, A_{h'}) ~ \bigg| S_h , A_h  \right]\right]
    \label{eqn:term2_fexp}
\end{align}
The term in \eqref{eqn:term2_fexp} can be bounded using the tower-property of expectations as
\begin{align*}
    &\frac{1}{2H} \EE_{\pi_k}\left[\indicator{B \leq h}
    \EE_{\pi_k}\left[ \sum_{h'=h+1}^H \gap(S_{h'}, A_{h'}) ~ \bigg| S_h , A_h  \right]\right]\\
    &\leq 
    \frac{1}{2H} \EE_{\pi_k}\left[
    \EE_{\pi_k}\left[ \sum_{h'=h+1}^H \gap(S_{h'}, A_{h'}) ~ \bigg| S_h , A_h  \right]\right]
    = \frac{1}{2H} \EE_{\pi_k}\left[\sum_{h'=h+1}^H \gap(S_{h'}, A_{h'}) \right].
\end{align*}
For the term in \eqref{eqn:term1_fexp}, we also use the tower-property to rewrite it as
\begin{align*}
    &\frac{1}{2H}\EE_{\pi_k}\left[\indicator{B \leq h}
    \frac{
    \EE_{\pi_k}\left[ \sum_{h'=B}^h \gap(S_{h'}, A_{h'}) ~ \bigg| ~ S_h, A_h \right]
    }{    \PP_{\pi_k}\left[ B \leq h ~ \mid  ~S_h, A_h \right]}
    \right]
    \\
    &=
    \frac{1}{2H}\EE_{\pi_k}\left[
    \EE_{\pi_k}\left[
    \indicator{B \leq h}
    \frac{
    \EE_{\pi_k}\left[ \sum_{h'=B}^h \gap(S_{h'}, A_{h'}) ~ \bigg| ~ S_h, A_h \right]
    }{    \PP_{\pi_k}\left[ B \leq h ~ \mid  ~S_h, A_h \right]}
     ~ \bigg| ~ S_h, A_h\right]
    \right]\\
        &=
    \frac{1}{2H}\EE_{\pi_k}\left[
    \EE_{\pi_k}\left[
    \indicator{B \leq h} ~ \bigg| ~ S_h, A_h\right]
    \frac{
    \EE_{\pi_k}\left[ \sum_{h'=B}^h \gap(S_{h'}, A_{h'}) ~ \bigg| ~ S_h, A_h \right]
    }{    \PP_{\pi_k}\left[ B \leq h ~ \mid  ~S_h, A_h \right]}
    \right]\\
            &=
    \frac{1}{2H}\EE_{\pi_k}\left[
    \EE_{\pi_k}\left[ \sum_{h'=B}^h \gap(S_{h'}, A_{h'}) ~ \bigg| ~ S_h, A_h \right]
    \right]\\
    &=\frac{1}{2H}\EE_{\pi_k}\left[
    \sum_{h'=B}^h \gap(S_{h'}, A_{h'}) 
    \right].
\end{align*}
Summing both terms yields the required upper-bound $\frac{1}{2H}\EE_{\pi_k}\left[\sum_{h=B}^H \gap(S_{h}, A_{h}) \right]$ on the expectation $\EE_{\pi_k}\left[ f_h(B, S_h, A_h)\right]$.
\end{proof}

\subsection{Policy-dependent regret bound for \textsc{StrongEuler}}
We now show how to derive a regret bound for \textsc{StrongEuler} algorithm in \citet{simchowitz2019non} that depends on the gaps of the played policies throughout the $K$ episodes.

To build on parts of the analysis in \citet{simchowitz2019non}, we first define some useful notation analogous to \citet{simchowitz2019non} but adapted to our setting:
\begin{align*}
    \bar n_k(s,a) &= \sum_{j=1}^k w^{\pi_k}(s,a),\\
    M &= (SAH)^3,\\
    \Vcal^{\pi}(s,a) &= \VV[R(s,a)] + \VV_{s'\sim P(\cdot|s,a)}[V^{\pi}(s')],\\
    \Vcal_k(s,a) &= \Vcal^{\pi_k}(s,a) \wedge \Vcal^*(s,a)
\end{align*}

We will use their following results:
\begin{proposition}[Proposition~F.1, F.9 and B.4 in \citet{simchowitz2019non}]
\label{prop:strongeuler_surplus_bound}
There is a good event $\Acal^{\mathrm{conc}}$ that holds with probability $1 - \delta / 2$. In this event,  
\textsc{StrongEuler} is strongly optimistic (as well as optimistic). Further, there is a universal constant $c \geq 1$ so that for all $k \geq 1$, $s \in \Scal$, $a \in \Acal$, the surpluses are bounded as
\begin{align*}
        0 \leq \frac{1}{c} E_{k}(s,a) \leq \Blead_{k}(s,a) + \sum_{h=\kappa(s)}^H \EE_{\pi_k}\left[\Bfut_{k}(S_h ,A_h) \mid (S_{\kappa(s)},A_{\kappa(s)}) = (s,a) \right],
\end{align*}
where $\Blead, \Bfut$ are defined as
\begin{align*}
    \Blead_{k}(s,a) &= H \wedge \sqrt{\frac{\Vcal_{k}(s,a)\log(M n_k(s,a)/\delta) }{ n_k(s,a)}},\\
    \Bfut_{k}(s,a) &= H^3 \wedge H^3\left(\sqrt{\frac{S\log(M n_k(s,a)/\delta) }{n_k(s,a)}} + \frac{S\log(M n_k(s,a)/\delta) }{n_k(s,a)}\right)^2.
\end{align*}
\end{proposition}

\begin{lemma}[Lemma~B.3 in \citet{simchowitz2019non}]
\label{lem:clip_dist}
Let $m \geq 2$, $a_1, \dots, a_m \geq 0$ and $\epsilon \geq 0$. Then $\clip\left[\sum_{i=1}^m a_i \big| \epsilon \right] \leq 2 \sum_{i=1}^m \clip\left[ a_i | \frac{\epsilon}{2m} \right]$.
\end{lemma}
 Equipped with these results and our improved surplus clipping proposition in \pref{prop:strongeuler_surplus_bound}, we can now derive the following bound on the regret of \textsc{StrongEuler}

 \begin{lemma}
 \label{lem:strongeuler_regB_bound}
In event $\Acal^{\mathrm{conc}}$, the regret of \textsc{StrongEuler} is bounded for all $k\geq 1$ as
 \begin{align*}
 \regret(K) \leq &
 8 \sum_{k=1}^K \sum_{s,a} w^{\pi_k}(s,a) \clip\left[c\Blead_{k}(s,a) ~ \bigg| ~ \frac{\breve\gap_k(s,a)}{4}\right]\\
& +
16 \sum_{k=1}^K \sum_{s,a} w^{\pi_k}(s,a) \clip\left[c\Bfut_{k}(s,a) ~ \bigg| ~  \frac{\breve\gap_k(s,a)}{8SA} \right],
 \end{align*}
 with a universal constant $c \geq 1$ and $\breve \gap_k(s,a) = \frac{\gap(s,a)}{4} \lor \epsilon_{k}(s,a)$.
 \end{lemma}

\begin{proof}
We now use our improved surplus clipping result from \pref{prop:surplus_clipping_bound} as a starting point to bound the instantaneous regret of \textsc{StrongEuler} in the $k$th episode as
\begin{align}
V^*(s_1) - V^{\pi_k}(s_1) 
\leq  
4\sum_{s,a} w^{\pi_k}(s,a) \clip\left[E_{k}(s,a)~ \bigg| \breve \gap_k(s,a) ~\right].
\label{eq:surplus_clip1}
\end{align}
Next, we write the bound on the surpluses from \pref{prop:strongeuler_surplus_bound} as
\begin{align*}
     E_{k}(s,a) \leq &~ c\Blead_{k}(s,a) \\
     &+ c \sum_{s', a'} \indicator{\kappa(s') \geq \kappa(s)}\PP^{\pi_k}\left[ S_{\kappa(s')} = s', A_{\kappa(s')} = a' \mid (S_{\kappa(s)},A_{\kappa(s)}) = (s,a) \right] \Bfut_{k}(s' ,a')
\end{align*}
and plugging it in \pref{eq:surplus_clip1} and applying \pref{lem:clip_dist} gives
\begin{align*}
    V^*(s_1) - V^{\pi_k}(s_1) 
\leq & ~  
8 \sum_{s,a} w^{\pi_k}(s,a) \clip\left[c\Blead_{k}(s,a) ~ \bigg| ~ \frac{\breve\gap_k(s,a)}{4}\right]\\
& +
16 \sum_{s,a} w^{\pi_k}(s,a) \clip\left[c\Bfut_{k}(s,a) ~ \bigg| ~  \frac{\breve\gap_k(s,a)}{8SA} \right].
\end{align*}
The statement to show follows now by summing over $k \in [K]$. The form of the second term in the previous display follows from the inequality 
\begin{align*}
    &\sum_{s,a}w^{\pi_k}(s,a) \indicator{\kappa(s') \geq \kappa(s)}\PP^{\pi_k}\left[ S_{\kappa(s')} = s', A_{\kappa(s')} = a' \mid (S_{\kappa(s)},A_{\kappa(s)})
    = (s,a) \right] \\
    &\leq \sum_{s,a}w^{\pi_k}(s,a) \PP^{\pi_k}\left[ S_{\kappa(s')} = s', A_{\kappa(s')} 
    = a' \mid (S_{\kappa(s)},A_{\kappa(s)}) = (s,a) \right] = w^{\pi_k}(s', a').
\end{align*}
\end{proof}
We note that if $\pi_k \equiv \hat\pi$ for any $\hat\pi \in \Pi^*$ then $V^*(s_1) - V^{\pi_k}(s_1) = 0$, and WLOG we can disregard such terms in the total regret.

The next step is to relate $\bar n_k(s,a)$ to $n_k(s,a)$ via the following lemma.
\begin{lemma}[Lemma~B.7 in \citet{simchowitz2019non}]
\label{lem:b7_simchowitz}
Define the event $\Acal^{\mathrm{samp}}$
\begin{align*}
    \Acal^{\mathrm{samp}} = \left\{\forall (s,a) \in \Scal\times\Acal, \forall k \geq \tau(s,a) \colon n_k(s,a) \geq \frac{\bar n_k(s,a)}{4}\right\},
\end{align*}
where $\tau(s,a) = \inf\{k : \bar n_k(s,a) \geq H_{\mathrm{samp}}\}$ and $H_{\mathrm{samp}} = c' \log(M/\delta)$ for a universal constant $c'$. Then event $\Acal^{\mathrm{samp}}$ holds with probability $1 - \delta/2$.
\end{lemma}
\begin{proof}
    This can be proved analogously to Lemma~B.7 in \citet{simchowitz2019non} and Lemma~6 in \citet{dann2019policy} with the difference that in our case, there can only be at most one observation of $(s,a)$ per episode for each $(s,a)$ due to our layered assumption. Thus, there is no need to sum over observations accumulated for each $h \in [H]$ and our  $H_{\mathrm{samp}} = O(\log(H))$ as opposed to $O(H \log(H))$.
\end{proof}

\begin{lemma}
\label{lem:fntofbarn}
Let $f_{s,a} \colon \NN \rightarrow \RR$ be non-increasing with $\sup_{u} f_{s,a}(u) \leq \hat f < \infty$ for all $s,a \in \Scal \times \Acal$. Then on event $\Acal^{\mathrm{samp}}$ in \pref{lem:b7_simchowitz}, we have
\begin{align*}
    \sum_{k=1}^K \sum_{s,a} w^{\pi_k}(s,a) f_{s,a}(n_k(s,a)) \leq S A \hat f H_{\mathrm{samp}} +
    \sum_{s,a} \sum_{k=\tau(s,a)}^K  w^{\pi_k}(s,a) f_{s,a}(\bar n_k(s,a) / 4).
\end{align*}
\end{lemma}
\begin{proof}
\begin{align*}
        &\sum_{k=1}^K \sum_{s,a} w^{\pi_k}(s,a) f_{s,a}(n_{k}(s,a)) \\
        =&~\sum_{s,a}\sum_{k=1}^{\tau(s,a)-1}  w^{\pi_k}(s,a)
        f_{s,a}(n_{k}(s,a))
        +\sum_{s,a}\sum_{k=\tau(s,a)}^{K}  w^{\pi_k}(s,a)
        f_{s,a}(n_{k}(s,a))\\
        \leq&~\sum_{s,a}\left(\sum_{k=1}^{\tau(s,a)-1}  w^{\pi_k}(s,a)\right) \hat f
        +\sum_{s,a}\sum_{k=\tau(s,a)}^{K}  w^{\pi_k}(s,a)
        f_{s,a}(\bar n_{k}(s,a) / 4)\\
        =& ~
        \sum_{s,a}n_{\tau(s,a)}(s,a) \hat f
        +\sum_{s,a}\sum_{k=\tau(s,a)}^{K}  w^{\pi_k}(s,a)
        f_{s,a}(\bar n_{k}(s,a) / 4)\\
        \leq&~         SA H_{\mathrm{samp}} \hat f
        +\sum_{s,a}\sum_{k=\tau(s,a)}^{K}  w^{\pi_k}(s,a)
        f_{s,a}(\bar n_{k}(s,a) / 4).
\end{align*}
\end{proof}

\begin{theorem}[Regret Bound for \textsc{StrongEuler}]
\label{thm:reg_bound_gen_se}
With probability at least $1 - \delta$, the regret of \textsc{StrongEuler} is bounded for all number of episodes $K \in \NN$ as
\begin{align*}
    \regret(K) \lesssim &~
    \sum_{s,a} \min_{t \in [K_{(s,a)}]}\Bigg\{\frac{\Vcal^*(s,a)\mathcal{LOG}(M/\delta,t,\breve\gap_t(s,a))}{\breve\gap_t(s,a)}\\
    &+ \sqrt{(K_{(s,a)}-t)\mathcal{LOG}(M/\delta,K_{(s,a)},\breve\gap_{K_{(s,a)}}(s,a))}\Bigg\}
    \\&
     + \sum_{s,a} SH^3 \log\frac{MK}{\delta} \min\left\{ \log \frac{MK}{\delta}, \log \frac{MH}{\breve \gap_{\min}(s,a)}\right\}
     \\& + SAH^3 (S \vee H)\log\frac{M}{\delta}.
\end{align*}
Here, $K_{(s,a)}$ is the last round during which a policy $\pi$ was played such that $w^{\pi}(s,a)>0$, $\breve\gap_{t}(s,a) = \gap(s,a) \lor \epsilon_{t}(s,a)$, $\breve \gap_{\min}(s,a) = \min_{k \in [K] \colon \breve\gap_k(s,a) > 0} \breve\gap_k(s,a)$ is the smallest gap encountered for each $(s,a)$, and $\mathcal{LOG}(M/\delta,t,\breve\gap_t(s,a)) = \log\left(\frac{M}{\delta}\right)\log\left(t\land 1 + \frac{16\Vcal^*(s,a)\log(M/\delta)}{\breve\gap_t(s,a)^2}\right)$.
\end{theorem}
\begin{proof}
We here consider the event $\Acal^{\mathrm{conc}} \cap \Acal^{\mathrm{samp}}$ which has probability at least $1 - \delta$ by \pref{prop:strongeuler_surplus_bound} and \pref{lem:b7_simchowitz}. We now start with the regret bound in \pref{lem:strongeuler_regB_bound} and bound the two terms individually in the following:
\paragraph{Bounding the $\Blead$ term}
We have
\begin{align}
& \sum_{k=1}^K \sum_{s,a} w^{\pi_k}(s,a) \clip\left[c\Blead_{k}(s,a) ~ \bigg| ~ \frac{\breve\gap_k(s,a)}{4}\right]\nonumber\\
&\overset{(i)}{\leq} 
SAH H_{\mathrm{samp}} + \sum_{s,a}\sum_{k=\tau(s,a)}^{K}  w^{\pi_k}(s,a)
\clip\left[c\sqrt{\frac{4\Vcal_{k}(s,a)\log(M \bar n_k(s,a)/4\delta) }{ \bar n_k(s,a)}} ~ \bigg| ~ \frac{\breve\gap_k(s,a)}{4}\right]
\nonumber\\
&\overset{(ii)}{\leq} \!
SAH H_{\mathrm{samp}} +\! \sum_{s,a}\sum_{k=\tau(s,a)}^{K_{(s,a)}} \!\!\! w^{\pi_k}(s,a)\!
\clip\left[2c\sqrt{\Vcal^*(s,a)\log\frac{M}{\delta}}\sqrt{\frac{\log(\bar n_k(s,a)) }{ \bar n_k(s,a)}} ~ \bigg| ~ \frac{\breve\gap_k(s,a)}{4}\right],
\label{eq:optapp1}
\end{align}
where step $(i)$ applies \pref{lem:fntofbarn} and $(ii)$ follows from the definition of $\Vcal_k(s,a)$, the definition of $K_{(s,a)}$ and
\begin{align*}
    &\log\left(\frac{M \bar n_k(s,a)}{4\delta}\right) = 
    \log\left(\frac{M}{4\delta}\right)
    +\log\left(\bar n_k(s,a)\right)\\
    &\leq \left(\log\left(\frac{M}{4\delta}\right) + 1\right)\log(\bar n_k(s,a)) 
    =  \log\left(\frac{Me}{4\delta}\right) \log(\bar n_k(s,a)) \leq \log(M/ \delta) \log(\bar n_k(s,a)).
\end{align*}
We now apply our optimization lemma (\pref{lem:sa_regret_bound}) with $x_k = w^{\pi_k}(s,a)$, $v_k = 2c\sqrt{\Vcal^*(s,a)\log(M /\delta)}$, and $\epsilon_{k} = \frac{\breve \gap_{k}(s,a)}{4 v_k}$ to bound each $(s,a)$-term in \pref{eq:optapp1} for any $t \in [K]$ as
\begin{align*}
    &4\frac{v_t}{\epsilon_t}\log\left(t\land 1 + \frac{1}{\epsilon_t^2}\right) + 4v_t\sqrt{\log\left(K\land 1+\frac{1}{\epsilon_K^2}(K-t)\right)}\\
    =&\frac{32c^2\Vcal^*(s,a)\log\left(\frac{M}{\delta}\right)\log\left(t\land 1 + \frac{16\Vcal^*(s,a)\log(M/\delta)}{\breve\gap_t(s,a)^2}\right)}{\breve\gap_t(s,a)}\\
    +&8c\sqrt{(K-t)\Vcal^*(s,a)\log\left(\frac{M}{\delta}\right)\log\left(K\land 1 + \frac{16\Vcal^*(s,a)\log(M/\delta)}{\breve\gap_K(s,a)^2}\right)}.
\end{align*}
Let $\mathcal{LOG}(M/\delta,t,\breve\gap_t(s,a)) = \log\left(\frac{M}{\delta}\right)\log\left(t\land 1 + \frac{16\Vcal^*(s,a)\log(M/\delta)}{\breve\gap_t(s,a)^2}\right)$. We have
\begin{align*}
& \sum_{k=\tau(s,a)}^{K}  w^{\pi_k}(s,a)
\clip\left[2c\sqrt{\Vcal^*(s,a)\log(M /4\delta)}\sqrt{\frac{\log(\bar n_k(s,a)) }{ \bar n_k(s,a)}} ~ \bigg| ~ \frac{\breve\gap_k(s,a)}{4}\right]\\
\leq &\frac{32c^2\Vcal^*(s,a)\mathcal{LOG}(M/\delta,t,\breve\gap_t(s,a))}{\breve\gap_t(s,a)} + 8c\sqrt{(K-t)\mathcal{LOG}(M/\delta,K,\breve\gap_K(s,a))}.
\end{align*}
%
Plugging this bound back in \pref{eq:optapp1} gives
\begin{align*}
    &\sum_{k=1}^K \sum_{s,a} w^{\pi_k}(s,a) \clip\left[c\Blead_{k}(s,a) ~ \bigg| ~ \frac{\breve\gap_k(s,a)}{4}\right]\\
    &\lesssim SAH \log\frac{M}{\delta}\\
    &+  \sum_{s,a} \min_{t \in [K_{(s,a)}]}\Bigg\{\frac{\Vcal^*(s,a)\mathcal{LOG}(M/\delta,t,\breve\gap_t(s,a))}{\breve\gap_t(s,a)} + \sqrt{(K_{(s,a)}-t)\mathcal{LOG}(M/\delta,K,\breve\gap_{K_{(s,a)}}(s,a))}
    \Bigg\}
\end{align*}
where $\lesssim$ only ignores absolute constant factors.
\paragraph{Bounding the $\Bfut$ term}
Consider the second term in \pref{lem:strongeuler_regB_bound} and event  $\Acal^{\mathrm{conc}} \cap \Acal^{\mathrm{samp}}$. Then by \pref{lem:fntofbarn}
\begin{align*}
    &\sum_{k=1}^K \sum_{s,a} w^{\pi_k}(s,a) \clip\left[c\Bfut_{k}(s,a) ~ \bigg| ~  \frac{\breve\gap_k(s,a)}{8SA} \right]\\
    &\leq SAH^3 H_{\mathrm{samp}} + \sum_{s,a} \sum_{k=\tau(s,a)}^{K}  w^{\pi_k}(s,a) f_{s,a}(\bar n_k(s,a))
\end{align*}
where $f_{s,a}$ is
\begin{align*}
    f_{s,a}(\bar n_k(s,a)) = \clip\left[2c H^3 \wedge 2cH^3 \left(\sqrt{\frac{S\log(M \bar n_k(s,a)/\delta) }{\bar n_k(s,a)}} + \frac{S\log(M \bar n_k(s,a)/\delta) }{\bar n_k(s,a)}\right)^2 ~ \bigg| ~ \frac{\breve\gap_k(s,a)}{4}\right].
\end{align*}
We now apply Lemma~C.1 by \citet{simchowitz2019non} which gives
\begin{align*}
    &\sum_{k=1}^K \sum_{s,a} w^{\pi_k}(s,a) \clip\left[c\Bfut_{k}(s,a) ~ \bigg| ~  \frac{\breve\gap_k(s,a)}{8SA} \right]\\
    &\leq SAH^3 H_{\mathrm{samp}} + \sum_{s,a} H f_{s,a}(H) + \sum_{s,a}\int_{H}^{\bar n_K(s,a)}  f_{s,a}(u) du\\
    &\leq SAH^4 c'\log(M/\delta) + \sum_{s,a}\int_{H}^{\bar n_K(s,a)}  f_{s,a}(u) du.
\end{align*}
The remaining integral term is bounded with Lemma~B.9 (b) by \citet{simchowitz2019non} with $C' = S, C= H^3$ and $\epsilon = \breve \gap_{\min}(s,a) = \min_{k \in [K_{(s,a)}] \colon \breve\gap_k(s,a) > 0} \breve\gap_k(s,a)$ as follows.
\begin{align*}
    &\sum_{k=1}^K \sum_{s,a} w^{\pi_k}(s,a) \clip\left[c\Bfut_{k}(s,a) ~ \bigg| ~  \frac{\breve\gap_k(s,a)}{8SA} \right]\\
    &\lesssim  SAH^4 \log\frac{M}{\delta} + \sum_{s,a}\left(SH^3 \log\frac{M}{\delta}
    + SH^3 \log\frac{MK}{\delta} \min\left\{ \log \frac{MK}{\delta}, \log \frac{MH}{\breve \gap_{\min}(s,a)}\right\}\right)\\
        &\lesssim  SAH^3 (S \vee H)\log\frac{M}{\delta} + \sum_{s,a} SH^3 \log\frac{MK}{\delta} \min\left\{ \log \frac{MK}{\delta}, \log \frac{MH}{\breve \gap_{\min}(s,a)}\right\}
\end{align*}

\end{proof}

\paragraph{Comparing with the bound in \citet{simchowitz2019non}.} We now proceed to compare our bound directly to the one stated in Corollary B.1~\citep{simchowitz2019non}. We will ignore the factors with only poly-logarithmic dependence on gaps as they are are common between both bounds. We now recall the regret bound presented in Corollary B.1, modulo said factors:
\begin{align*}
    \regret(K) \leq O\Bigg(\sum_{(s,a) \in \Zcal_{sub}} \frac{\alpha H\Vcal^*(s,a)}{\gap(s,a)}\mathcal{LOG}(M/\delta,K,\gap(s,a)) + |\Zcal_{opt}|\frac{H\Vcal^*}{\gap_{\min}}\mathcal{LOG}(M/\delta,K,\gap_{\min})\Bigg),
\end{align*}
where $\Vcal^* = \max_{(s,a)}\Vcal(s,a)$, $\Zcal_{opt}$ is the set on which $\gap(s,\pi^*(s)) = 0$, i.e., the set of state-action pairs assigned to $\pi^*$ according to the Bellman optimality condition, and $\Zcal_{sub}$ is the complement of $\Zcal_{opt}$. If we take $t = K$ in Theorem~\ref{thm:reg_bound_gen_se}, we have the following upper bound:
\begin{align*}
    \regret(K) \leq O\Bigg(\sum_{(s,a) \in \Zcal_{sub}} \frac{\Vcal^*(s,a)\mathcal{LOG}(M/\delta,K,\gap(s,a))}{\gap(s,a)} + \frac{H\Vcal^*|\Scal_{opt}|\mathcal{LOG}(M/\delta,K,\gap_{\min})}{\min_{k,s,a}\epsilon_k(s,a)}\Bigg),
\end{align*}
where $\Scal_{opt}$ is the set of all states for $s\in\Scal$ for which $\gap(s,\pi^*(s)) = 0$ and there exists at least one state $s'$ with $\kappa(s')<s$ for which $\gap(s',\pi^*(s))>0$. We note that this set is no larger than the set $\mathcal{Z}_{opt}$ and further that even the smallest $\epsilon_k(s,a)$ can still be much larger than $\gap_{\min}$, as it is the conditional average of the gaps. In particular, this leads to an arbitrary improvement in our example in Figure~\ref{fig:summary} and an improvement of $SA$ in the example in Figure~\ref{fig:fail_gap2}.

\subsection{Nearly tight bounds for deterministic transition MDPs}
We recall that for deterministic MDPs, $\epsilon_{k}(s,a) = \frac{V^*(s_1) - V^{\pi_k}(s_1)}{2H},\forall a$ and the definition of the set $\Pi_{s,a}$:
$$
\Pi_{s,a} \equiv \{\pi \in \Pi ~\colon s^\pi_{\kappa(s)} = s ,a^\pi_{\kappa(s)} = a, \exists~ h \leq \kappa(s), \gap(s^\pi_{h},a^\pi_{h})>0\}.
$$
We note that $\Vcal(s,a) \leq 1$ as this is just the variance of the reward at $(s,a)$.
Theorem~\ref{thm:reg_bound_gen_se} immediately yields the following regret bound by taking $t=K$.
\begin{corollary}[Explicit bound from \pref{eq:det_trans_reg}]
\label{cor:det_trans_formal}
Suppose the transition kernel of the MDP consists only of point-masses. Then with probability $1-\delta$, \texttt{StrongEuler}'s regret is bounded as
\begin{align*}
    \regret(K) &\leq O\Bigg(\sum_{(s,a) : \Pi_{s,a}\neq \emptyset} \frac{H\mathcal{LOG}\left(M/\delta,K,\returngap(s,a)\right)}{\return{*}-\return{*}_{s,a}}\\
    \\&+ \sum_{s,a} SH^3 \log\frac{MK}{\delta} \min\left\{ \log \frac{MK}{\delta}, \log \frac{MH}{\returngap(s,a)}\right\}
     \\& + SAH^3 (S \vee H)\log\frac{M}{\delta}\Bigg),
\end{align*}
where $\return{*}_{s,a} = \max_{\pi\in\Pi_{s,a}}\return{\pi}$.
\end{corollary}
We now compare the above bound with the one in \citep{simchowitz2019non} again. For simplicity we are going to take $K$ to be the smaller of the two quantities in the logarithm. To compare the bounds, we compare $\sum_{(s,a) : \Pi_{s,a}\neq \emptyset} \frac{H(\log(KM/\delta)))}{\return{*}- \return{*}_{(s,a)}}$ to $\sum_{(s,a) \in \mathcal{Z}_{sub}} \frac{\alpha H \log(KM/\delta)}{\gap(s,a)} + \frac{|\mathcal{Z}_{opt}|H}{\gap_{\min}}$. Recall that $\alpha \in [0,1]$ is defined as the smallest value such that for all $(s,a,s') \in \Scal\times\Acal\times\Scal$ it holds that
\begin{align*}
    P(s'|s,a) - P(s'|s,\pi^*(s)) \leq \alpha P(s'|s,a).
\end{align*}
For any deterministic transition MDP with more than one layer and one sub-optimal action it holds that $\alpha = 1$. We will compare $V^*(s_1) - V^{\pi^*_{(s,a)}}(s_1)$ to $\gap(s,a) = Q^*(s,\pi^*(s)) - Q^*(s,a)$. This comparison is easy as by Lemma~\ref{lem:gap_decomp_pi} we can write 
\begin{align*}
V^*(s_1) - V^{\pi^*_{(s,a)}}(s_1) = \sum_{(s',a') \in \pi^*_{(s,a)}} w_{\pi^*_{(s,a)}(s',a')} \gap(s',a') = \sum_{(s',a') \in \pi^*_{(s,a)}}\gap(s',a') \geq \gap(s,a).
\end{align*}
Hence, our bound in the worst case matches the one in \citet{simchowitz2019non} and can actually be significantly better. We would further like to remark that we have essentially solved all of the issues presented in the example MDP in Figure~\ref{fig:summary}. In particular we do not pay any gap-dependent factors for states which are only visited by $\pi^*$, we do not pay a $\gap_{\min}$ factor for any state and we never pay any factors for distinguishing between two suboptimal policies. Finally, we compare this bound to the lower bound derived Theorem~\ref{thm:lower_bound_deterministic} only with respect to number of episodes and gaps. Let $\Scal^*$ be the set of all states in the support of an optimal policy
\begin{align*}
    \sum_{(s,a) \in \Scal\setminus\Scal^*\times\Acal} \frac{\log(K)}{H(\return{*} - \return{\pi^*_{(s,a)}}(s_1))} \leq \regret(K) \leq \sum_{(s,a) : \Pi_{s,a}\neq \emptyset}\frac{H\log(K)}{\return{*} - \return{*}_{s,a}}.
\end{align*}
The difference between the two bounds, outside of an extra $H^2$ factor, is in the sets $\Scal^*$ and the set $\{s,a : \Pi_{s,a}=\emptyset\}$. We note that $\{s,a : \Pi_{s,a}=\emptyset\} \subseteq \Scal^*$. Unfortunately there are examples in which $\{s,a : \Pi_{s,a}=\emptyset\}$ is $O(1)$ and $\Scal^* = \Omega(S)$ leading to a discrepancy between the upper and lower bounds of the order $\Omega(S)$. As we show in \pref{thm:det_lower_bound} this discrepancy can not really be avoided by optimistic algorithms.

\subsection{Tighter bounds for unique optimal policy.}
\label{app:unique_opt_pol}
If we further assume that the optimal policy is unique on its support, then we can show \textsc{StrongEuler} will only incur regret on sub-optimal state-action pairs. This matches the information theoretic lower bound up to horizon factors. We begin by showing a different type of upper bound on the expected gaps by the surpluses.
Define the set $\beta_k = range(B)$ where $B$ is the r.v. which is the stopping time with respect to $\pi_k$. For any $\pi^*$, define the set 
\begin{align*}
    \Ocal_k(\pi^*) = \bigcup_{s_b \in \beta_k} \{(s,a)\in \Scal\times\Acal : \PP_{\pi^*}((S_h,A_h) = (s,a)|S_{\kappa(s_b)}= s_b) \geq \PP_{\pi_k}((S_h,A_h) = (s,a)|S_{\kappa(s_b)}= s_b)\}.
\end{align*}
This set has the following intuitive definition -- whenever $\Acal_B$ occurs we restrict our attention to the MDP with initial state $S_B$. On this restricted MDP, $\Ocal_k$ is the set of state-action pairs which have greater probability to be visited by the optimal $\pi^*$ than by $\pi_k$.
\begin{lemma}
\label{lem:gap_surp_bound}
Assume strong optimism and greedy $\bar V_k,$ i.e., $\bar V_k(s) \geq \max_a \bar Q_k(s,a)$ for all $s \in \Scal$. Then there exists an optimal $\pi^*$ for which 
\begin{align*}
    \EE_{\pi_k}\left[\sum_{h=B}^H\gap(S_h,A_h)\right] \leq \EE_{\pi_k}\left[\sum_{h=B}^H \chi(S_h,A_h \not \in \Ocal_k(\pi^*))E_k(S_h,A_h)\right].
\end{align*}
\end{lemma}
\begin{proof}
One can write the optimistic value function for any $s$ and $\pi$ as follows
\begin{align*}
    \bar V^{\pi}(s) &= \EE_{\pi}\left[\sum_{h=\kappa(s)}^H E_k(S_h,A_h) + r(S_h,A_h)\big\vert S_{\kappa(s)} = s\right]\\
    &= E_k(s,\pi(s)) + r(s,\pi(s)) + \langle P(\cdot | s,\pi(s)), \bar V^{\pi} \rangle.
\end{align*}
By backwards induction on $H$ we show that for any $s$, $\kappa(s) \leq H$ $\bar V^{\pi} \leq \bar V_k$. The base case holds from the fact that on all $s: \kappa(s)=H$, $\bar V_k(s)$ is just the largest optimistic reward over all actions at $s$. For the induction step it holds that
\begin{align*}
    \bar V^{\pi}(s) &= E_k(s,\pi(s)) + r(s,\pi(s)) + \langle P(\cdot|s,\pi(s)),\bar V^{\pi} \rangle\\
    &\leq E_k(s,\pi(s)) + r(s,\pi(s)) + \langle P(\cdot|s,\pi(s)),\bar V_k \rangle\\
    & = \bar Q_k(s,\pi(s)) \leq \bar V_k(s),
\end{align*}
where the first inequality holds from the induction hypothesis and the second inequality holds by definition of the value function.
We now have
\begin{align*}
    \EE_{\pi_k}\left[\sum_{h=B}^H \gap(S_h,A_h)\right] &= \EE_{\pi_k}\left[V^*(S_B) - V_k(S_B)\right]\\
    &\leq \EE_{\pi_k}\left[\bar V_k(S_B) - V_k(S_B)\right] - \EE_{\pi_k}\left[\bar V^*(S_B) - V^*(S_B)\right].
\end{align*}
Let us focus on the term $\EE_{\pi_k}\left[\bar V^*(S_B) - V^*(S_B)\right]$
\begin{align*}
    \EE_{\pi_k}\left[\bar V^*(S_B) - V^*(S_B)\right] &= \EE_{\pi_k}\left[\EE_{\pi_k}\left[\bar V^*(S_B) - V^*(S_B)|S_B\right]\right]\\
    &=\EE_{\pi_k}\left[ \sum_s \frac{\bar V^*(s) - V^*(s)}{\PP_{\pi_k}(S_B=s)}\chi(S_B=s)\right]\\
    &=\EE_{\pi_k}\left[ \sum_s \frac{ \EE_{\pi^*}\left[ \sum_{h=\kappa(s)}^H E_k(S_h,A_h)| S_{\kappa(s)}=s\right] }{\PP_{\pi_k}(S_B=s)}\chi(S_B=s)\right].
\end{align*}
We can similarly expand the term $\EE_{\pi_k}\left[\bar V_k(S_B) - V_k(S_B)\right]$. By the definition of $\Ocal_k(\pi^*)$ it holds that for any $h\geq\kappa(s)$
\begin{align*}
    \EE_{\pi_k}\left[E_k(S_h,A_h)|S_{\kappa(s)} = s\right] &- \EE_{\pi^*}\left[E_k(S_h,A_h)|S_{\kappa(s)} = s\right]\\
    &\leq \EE_{\pi_k}\left[\chi(S_h,A_h \not \in \Ocal_k(\pi^*))E_k(S_h,A_h)|S_{\kappa(s)} = s\right].
\end{align*}
This implies 
\begin{align*}
    \EE_{\pi_k}\left[\bar V^*(S_B) - V^*(S_B)\right] &\leq \EE_{\pi_k}\left[ \sum_s \frac{ \EE_{\pi_k}\left[ \sum_{h=\kappa(s)}^H \chi(S_h,A_h \not \in \Ocal_k(\pi^*))E_k(S_h,A_h)| S_{\kappa(s)}=s\right] }{\PP_{\pi_k}(S_B=s)}\chi(S_B=s)\right]\\
    &= \EE_{\pi_k}\left[\sum_{h=B}^H \chi(S_h,A_h \not \in \Ocal_k(\pi^*))E_k(S_h,A_h)\right].
\end{align*}
\end{proof}

We next show a version of Lemma~\ref{lem:Vdd_lb1} which takes into account the set $\Ocal_k(\pi^*)$.

\begin{lemma}
\label{lem:clipped_val_lower}
With the same assumptions as in Lemma~\ref{lem:gap_surp_bound}, there exists an optimal $\pi^*$ for which
\begin{align*}
    \ddot V_k(s_1) - V_k(s_1) \geq \EE_{\pi_k}\left[\sum_{h=B}^H \gap(S_h,A_h) - \sum_{h=B}^H\chi(S_h,A_h \not \in \Ocal_k(\pi^*))\epsilon_k(S_h,A_h)\right],
\end{align*}
where $\epsilon_k$ is arbitrary.
\end{lemma}

\begin{proof}
Since $\ddot E_k$ is non-negative on all state-action pairs we have
\begin{align*}
    \ddot V_k(s_1) - V^{\pi_k}(s_1) &= \EE_{\pi_k}\left[\sum_{h=1}^{H} \ddot E_k(S_h,A_h)\right]
    \geq \EE_{\pi_k}\left[\sum_{h=B}^{H} \ddot E_k(S_h,A_h)\right]\\
    & \geq \EE_{\pi_k}\left[\sum_{h=B}^{H}
    \indicator{(S_h,A_h) \not\in \Ocal_k}
    \ddot E_k(S_h,A_h)\right]\\
    &\geq \EE_{\pi_k}\left[\sum_{h=B}^{H}\chi((S_h,A_h) \not\in \Ocal_k) E_k(S_h,A_h)\right] - \EE_{\pi_k}\left[\sum_{h=B}^{H}\chi((S_h,A_h) \not\in \Ocal_k) \epsilon_k(S_h,A_h)\right]\\
    &\geq \EE_{\pi_k}\left[\sum_{h=B}^H \gap(S_h,A_H)\right] - \EE_{\pi_k}\left[\sum_{h=B}^{H}\chi((S_h,A_h) \not\in \Ocal_k) \epsilon_k(S_h,A_h)\right],
\end{align*}
where the second to last inequality follows from the definition of $\ddot E_k$ and the last inequality follows from Lemma~\ref{lem:gap_surp_bound}.
\end{proof}
Next, we define $\bar\epsilon_k$ in the following way. Let
\begin{align}
    \bar\epsilon_k(s,a) &\equiv
    \begin{cases}
    \epsilon_k(s,a) &\textrm{if } (s,a) \not\in \Ocal_k(\pi^*)\\
    \infty & \textrm{otherwise},
    \end{cases}
\end{align}
where $\epsilon_k$ is the clipping function defined in \pref{eqn:new_avgclip_clean}.
Lemma~\ref{lem:clipped_val_lower} now implies that
\begin{align*}
    \ddot V_k(s_1) - V_k(s_1) \geq \EE_{\pi_k}\left[\sum_{h=B}^H \gap(S_h,A_h) - \sum_{h=B}^H \bar \epsilon_k(S_h,A_h)\right].
\end{align*}
This is sufficient to argue Lemma~\ref{lem:strongeuler_regB_bound} with $\breve\gap_k(s,a) = \frac{\gap(s,a)}{4}\lor \bar\epsilon_k(s,a)$ and hence arrive at a version of Corollary~\ref{cor:det_trans_formal} which uses $\bar\epsilon_k$ as the clipping thresholds. Let us now argue that $\bar\epsilon_k(s,a) = \infty$ for all $(s,a)\in\pi^*$ whenever $\pi^*$ is the unique optimal policy for the deterministic MDP. To do so consider $(s,a)\in\pi^*$ and $\pi_k \neq \pi^*$. Since the MDP is deterministic, $\beta_k$ is a singleton and is the the first state $s_b$ at which $\pi_k$ differs from $\pi^*$. We now observe that if $\kappa(s) < \kappa(s_b)$, this implies $\epsilon_k(s,a) = \infty$ as $\Bcal(s,a)$ does not occur. Further, the conditional probabilities $\PP_{\pi^*}((S_h,A_h) = (s,a)|S_{\kappa(s_b)}= s_b)$ and $\PP_{\pi_k}((S_h,A_h) = (s,a)|S_{\kappa(s_b)}= s_b)$ are both equal to $1$ if $\kappa(s) > \kappa(s_b)$ and so $(s,a) \in \Ocal_k(\pi^*)$ which implies $\bar\epsilon_k(s,a) = \infty$. Thus we can clip all gaps at $(s,a) \in \pi^*$ to infinity and they will never appear in the regret bound. With the notation from Corollary~\ref{cor:det_trans_formal} we have the following tighter bound.
\begin{corollary}
\label{cor:det_upper_tight}
Suppose the transition kernel of the MDP consists only of point-masses and there exists a unique optimal $\pi^*$. Then with probability $1-\delta$, \texttt{StrongEuler}'s regret is bounded as
\begin{align*}
    \regret(K) &\leq O\Bigg(\sum_{(s,a) \not\in \pi^*} \frac{\mathcal{LOG}\left(M/\delta,K,\returngap(s,a)\right)}{\returngap(s,a)}\\
    \\&+ \sum_{(s,a)\not\in\pi^*} SH^3 \log\frac{MK}{\delta} \min\left\{ \log \frac{MK}{\delta}, \log \frac{MH}{\returngap(s,a)}\right\}
     \\& + SAH^3 (S \vee H)\log\frac{M}{\delta}\Bigg).
\end{align*}
\end{corollary}
Comparing terms which depend polynomially on $1/\returngap$ to the information theoretic lower bound in \pref{thm:lower_bound_deterministic} we observe only a multiplicative difference of $H^2$.

\subsection{Alternative to integration lemmas}
The following lemma is an alternative to the integration lemmas when bounding the sum of the clipped surpluses and in some cases allows us to save additional factors of $H$.
\begin{restatable}[]{lemma}{optimlemma}
Consider the following optimization problem
\begin{equation}
\begin{aligned}
\label{eq:opt_prob_regret_upper}
    \maximize{x_1,\ldots,x_K}{\sum_{k=1}^K \frac{v_k x_k\sqrt{\log(\sum_{j=1}^k x_j)}}{\sqrt{\sum_{j=1}^k x_j}}}
    {
    1 \leq x_1,\quad 
    0\leq x_k\leq 1, \quad
    \frac{\sqrt{\log(\sum_{j=1}^k x_j)}}{\sqrt{\sum_{j=1}^k x_j}} \geq \epsilon_k\qquad\forall~ k\in[K]},
\end{aligned}
\end{equation}
with $(v_i)_{i \in [K]} \in \RR_{+}^K$ and $(\epsilon_i)_{i \in [K]} \in \RR_+^{K}$. Then the optimal value of Problem~\ref{eq:opt_prob_regret_upper} is bounded for any $t \in [K]$ as
\begin{align}
    4\frac{\bar v_{t}}{\epsilon_t}\log\left(t \land 1+\frac{1}{\epsilon_{t}^2}\right) + 4v^*_t\sqrt{\log\left(K\land 1+ \frac{1}{\epsilon_{K}^2}\right)(K-t)},
    \label{eq:opt_prob_bound}
\end{align}
where $\bar v_t = \max_{k \in [t]} v_k$ and $v^*_t = \max_{K\geq k\geq t} v_k$.
\label{lem:sa_regret_bound}
\end{restatable}

\begin{proof}Denote by $X_k = \sum_{t=1}^k x_t$ the cumulative sum of $x_t$.
The proof consists of splitting the objective of \prettyref{eq:opt_prob_regret_upper} into
two terms:
\begin{align}
    \sum_{k=1}^t \frac{v_k x_k\sqrt{\log( X_k)}}{\sqrt{X_k}} + \sum_{k=t+1}^K \frac{v_k x_k\sqrt{\log(X_k)}}{\sqrt{X_k}}
    \label{eq:obj_split}
\end{align}
and bounding each by the corresponding one in  \prettyref{eq:opt_prob_bound} respectively.

Before doing so, we derive the following bound on the sum of $\frac{x_k}{\sqrt{X_k}}$ terms:
\begin{align}
\label{eq:sqrt_X_bound}
    \sum_{k=m+1}^M\frac{x_k}{\sqrt{X_k}}=\sum_{k=m+1}^M\frac{X_k-X_{k-1}}{\sqrt{X_k}}\leq \int_{X_m}^{X_M}\frac{1}{\sqrt{x}}\,dx = 2(\sqrt{X_M}-\sqrt{X_m})\,,
\end{align}
where the inequality is due to $X_k$ being non-decreasing.

Consider now each term in the objective in \prettyref{eq:obj_split} separately.

\paragraph{Summands up to $t$:}

Since $X_k$ is non-decreasing, we can bound
\begin{align*}
    \sum_{k=1}^t \frac{v_k x_k \sqrt{\log(X_k)}}{\sqrt{X_k}}
    \leq 
    \bar v_t \sqrt{\log(X_t)}
    \sum_{k=1}^t \frac{x_k}{\sqrt{X_k}}
    & \overset{(i)}{\leq} 2\bar v_t \sqrt{\log(X_t)}\sqrt{ X_t}\\
    & \overset{(ii)}{\leq}
    2\frac{\bar v_t}{\epsilon_t} \log(X_t),
\end{align*} 
where $(i)$ follows from \prettyref{eq:sqrt_X_bound} using the convention $X_0=0$ and $(ii)$ from the optimization constraint $\sqrt{\log(X_t)} \geq \epsilon_t \sqrt{X_t}$. It remains to bound $\log(X_t)$ by $2\log\left(t \land 1+\frac{1}{\epsilon_{t}^2}\right)$.  Since all increments $x_j$ are at most $1$, the bound $\log(X_t) \leq \log(t)$ holds.

We claim the following:
\begin{claim}
\label{claim:log_ineq}
    For any $x$ s.t. $\log(x) \leq \log(\log(x)/a)$ it holds that $\log(x) \leq 2\log(1+1/a)$.
\end{claim}
\begin{proof}
    First, we note that if $0<x\leq e$, then $\log(\log(x)) < 0$ and thus the assumption of the claim implies $\log(x) \leq \log(1/a)$. Next, assume that $x > e$. Then we have $\frac{\log(\log(x))}{\log(x)} \leq 1/e$, which together with the assumption of the claim implies $\log(x) \leq 1/e\log(x) + \log(1/a)$ or equivalently $\log(x) \leq \frac{e}{e-1}\log(1/a)$. Noting that $e/(e-1) \leq 2$ completes the proof.
\end{proof}

 The constraints of the problem enforce $\sqrt{X_k} \leq \frac{\sqrt{\log(X_k)}}{\epsilon_k}$,
 which implies after squaring and taking the $\log$: 
 $\log(X_k) \leq \log(\log(X_k)/\epsilon_k^2)$. Thus, using Claim~\ref{claim:log_ineq} yields:
 \begin{align}
     \log(X_k) \leq 2\log(k\land 1+1/\epsilon_k^2).
     \label{eq:logX_bound}
 \end{align}

\paragraph{Summands larger than $t$:}
Let $v^*_t = \max_{k \colon t < k \leq K} v_k$. For this term, we have
\begin{align*}
    \sum_{k=t+1}^K \frac{v_k x_k \sqrt{\log(X_k)}}{\sqrt{X_k}} 
   &\overset{\prettyref{eq:logX_bound}}{\leq} 2v^*_t\sqrt{\log(K\land 1+1/\epsilon^2_K)}\sum_{k=t+1}^K \frac{x_k}{\sqrt{X_k}}\\
   &\overset{\prettyref{eq:sqrt_X_bound}}{\leq} 4v^*_t\sqrt{\log(K\land 1+1/\epsilon^2_K)}(\sqrt{X_K}-\sqrt{X_t})\\
   &\leq 4v^*_t\sqrt{\log(K\land 1+1/\epsilon^2_K)}(\sqrt{X_K-X_t})\\
   &\leq 4v^*_t\sqrt{\log(K\land 1+1/\epsilon^2_K)}(\sqrt{K-t}),
\end{align*}
where we first bounded $\log(X_k) \leq \log(X_K)$, because $X_k$ is non-decreasing, and used the upper bound on $\log(X_K)$.
Then we applied \prettyref{eq:sqrt_X_bound} and finally used $0\leq x_k\leq 1$.
\end{proof}

\end{document}